\documentclass{article}

\PassOptionsToPackage{numbers, square}{natbib}  

\usepackage[main, final]{neurips_2026}

\usepackage[utf8]{inputenc} 
\usepackage[T1]{fontenc}    

\usepackage[
  colorlinks,
  linkcolor=blue,
  citecolor=blue,
  urlcolor=blue
]{hyperref}
\usepackage{url}
\usepackage{booktabs}       
\usepackage{amsfonts}       
\usepackage{nicefrac}       
\usepackage{microtype}      
\usepackage{xcolor}         

\usepackage{graphicx}       
\usepackage{multirow}
\usepackage{makecell}
\usepackage{amsmath}
\usepackage{amssymb}
\usepackage{mathtools}
\usepackage{amsthm}
\usepackage{algorithm}
\usepackage{algpseudocode}
\usepackage[capitalise]{cleveref}    
\makeatletter
\renewcommand{\theHALG@line}{\thealgorithm.\arabic{ALG@line}}
\makeatother
\usepackage{wrapfig}            
\usepackage{float}              

\newtheorem{theorem}{Theorem}[section]

\newtheorem{lemma}[theorem]{Lemma}
\newtheorem{assumption}[theorem]{Assumption}

\DeclareMathOperator{\logit}{logit}

\definecolor{estblue}{RGB}{70,120,180}
\definecolor{corgreen}{RGB}{90,150,70}
\definecolor{intorange}{RGB}{180,120,80}
\definecolor{midgray}{RGB}{110,110,110}
\definecolor{clusred}{RGB}{190,80,90}

\title{Uncertainty-Aware Probabilistic Constrained Clustering from Entangled Pairwise Supervision}

\author{
 \hspace*{-6mm} Shaojie Zhang ~~~~~~~~ Ke Chen\\
  Department of Computer Science,
  The University of Manchester, Manchester M13 9PL, U.K. \\
  \texttt{\{shaojie.zhang,ke.chen\}@manchester.ac.uk} \\
}

\begin{document}

\maketitle

\begin{abstract}
Pairwise constrained clustering typically relies on hard must-link/cannot-link labels, whereas realistic pairwise supervision may be real-valued and entangle intrinsic ambiguity, expert judgment, and stochastic corruption.
Existing deep constrained clustering (DCC) methods mainly target hard, expert-agnostic constraints, treating soft labels mostly numerically rather than semantically.
We formalize this setting as uncertainty-aware probabilistic constrained clustering (UPCC), defining a canonical aleatoric relation through a heterogeneous observation process and analyzing its conditional identifiability.
We introduce ProbPair, an angular pairwise objective for probabilistic relations, and build ECI-PP, an estimator--corrector--integrator framework that refines imperfect supervision via belief estimation, correction, and reliability-aware integration.
Across challenging probabilistic supervision settings, experiments on diverse benchmarks show that ECI-PP outperforms state-of-the-art DCC methods and remains robust with a shared default configuration.
Code is available at our \href{https://github.com/ecipp/ECI-PP}{repository}.
\end{abstract}


\section{Introduction}
\label{sec:intro}

Clustering is a fundamental tool for uncovering structure in data, yet its unsupervised nature often yields partitions that may not align with domain knowledge \cite{SDCD,DCGMM,SpherePair}; pairwise constrained clustering (CC) addresses this by using weak instance-level pairwise supervision \cite{CCsurvey2007,CCsurvey2023}. 
Traditional CC methods \cite{CC1st,copkmeans,GMMCC,MPCKMeans,PCKMeans,KernelCC}, however, often struggle with high-dimensional and complex data, motivating deep constrained clustering (DCC). 
Existing DCC methods can be broadly organized into two paradigms: end-to-end DCC \cite{PCC,SDEC,CIDEC1,VolMaxDCC}, which reformulates clustering as a pseudo-classification problem with anchors, and deep constraint embedding \cite{CPAC,AutoEmbedder}, which instead learns clustering-friendly representations from pairwise supervision and, in more recent advances, further moves pairwise learning from Euclidean to angular space to better reconcile positive and negative relations \cite{SpherePair}.

Despite these advances, the supervision setting remains largely idealized.
Existing work typically assumes hard binary pairwise relations, whereas supervision in practice often lies on a continuum.
Moreover, such real-valued pairwise labels may simultaneously reflect intrinsic \textit{aleatoric} ambiguity, expert-conditioned and pair-dependent \textit{epistemic} judgment \cite{hullermeier2021aleatoric}, and additional \textit{stochastic} corruption from annotation or recording pipelines.
For example, when comparing two clinical cases, a recorded observation may reflect ambiguity in the cases themselves, systematic clinician judgment bias, and documentation or data-entry errors.
These considerations motivate what we term \textit{uncertainty-aware probabilistic constrained clustering} (UPCC): an unexplored, holistic, and practically realistic setting in which probabilistic pairwise relations jointly reflect intrinsic ambiguity, expert-dependent judgment, and stochastic corruption, more closely matching real-world supervision than binary or merely real-valued constraints.

We specify an observation model that distinguishes the canonical aleatoric relation from expert-conditioned judgment and stochastic corruption, and analyze when this relation is identifiable on observable pairs.
To learn under UPCC, we first introduce \textit{ProbPair}, a deep constraint embedding approach linking probabilistic pairwise relations to angular representation learning.
Built on it, we further develop an \textit{Estimator}--\textit{Corrector}--\textit{Integrator} \textit{ProbPair} (ECI-PP) framework, in which Estimators provide model beliefs, Correctors perform structured correction and suppress residual discordance, and the Integrator learns from the resulting surrogate supervision toward the underlying aleatoric relation.

Our main contributions are:  
(i) We formalize UPCC through a heterogeneous observation model and analyze identifiability of the canonical aleatoric relation on observable pairs.
(ii) We introduce \textit{ProbPair}, a deep constraint embedding approach linking probabilistic pairwise relations to angular embeddings.  
(iii) We develop \textit{ECI-PP}, a practical framework for aleatoric relation learning under entangled supervision.
(iv) We empirically show across diverse benchmarks that ECI-PP outperforms state-of-the-art DCC methods and is robust under fallible, heterogeneous, and corrupted supervision.

\section{Related work}
\label{sec:related_work}

\paragraph{Beyond binary pairwise constraints.} 
Existing CC methods either extend binary constraints only numerically, without semantically modeling intermediate supervision, or assign intermediate values to the confidence of a pre-defined relation type rather than a unified probability.
Advanced deep methods typically employ logistic losses \cite{MCL,SpherePair} compatible with real-valued constraints in $[0,1]$. However, \textit{end-to-end DCC} methods \cite{PCC,CIDEC1,S3C2,scDCC,ConstraintMatch,VolMaxDCC} remain tied to hard anchor semantics, with attendant anchor misalignment \cite{MV-Effi} and error propagation \cite{MV-Tens,MV-Rob}, while \textit{deep constraint embedding} methods \cite{SpherePair,AutoEmbedder,CPAC} avoid anchors but still require a binary regime for their theoretical geometric guarantees. Likewise, the \textit{deep generative} approach \cite{DCGMM} encodes constraints as a signed prior over discrete assignments, precluding a unified treatment of real-valued relations.
Earlier non-deep approaches considered intermediate supervision more explicitly: (i) \textit{probabilistic} treatments \cite{law2005PCC,law2004PCC} encoded the activation probability of a positive same-cluster tie, but offered neither negative relations nor a unified notion of pairwise affinity; and (ii) \textit{fuzzy-constraint} formulations attach an intermediate belief \cite{FuzzyCC1}, degree \cite{FuzzyCC2}, or penalty \cite{FuzzyCC3,FuzzyCC4,FuzzyCC5,FuzzyCC6} only after a discrete relation type has been specified.
Taken together, these lines of work do not address uncertain relational judgment expressed by probabilistic labels under a unified semantics.
By contrast, our \textit{ProbPair} formulation, built on deep constraint embedding, directly models such unified probabilistic supervision, accommodating both vague annotations in practice and inherently ambiguous aleatoric relations.

\paragraph{Imperfect pairwise supervision.}
While imperfect annotations in CC have long been recognized \cite{Pelleg2007,covoes2013study,Zhu2015,liu2017partition,MCVC,CIDEC2,SDCD,DCGMM,VolMaxDCC}, more realistic observation settings remain underexplored.
Existing studies treat simplified noise through constraint-set robustness \cite{Pelleg2007,Zhu2015}, coarse prior uncertainty over assignments \cite{DCGMM}, or expert-level reliability in semi-crowdsourced settings \cite{MCVC,SDCD}.
VolMaxDCC \cite{VolMaxDCC} further introduces confusion-based structured annotation modeling beyond unstructured label degradation or flipping. 
However, it still lacks heterogeneous multi-expert modeling and remains restricted to discrete binary-noise settings, leaving expert-specific, pair-dependent epistemic distortion unmodeled.
By contrast, we consider genuinely entangled observations, arising from intrinsic \emph{aleatoric} relations, expert-conditioned \emph{epistemic} judgments, and \emph{stochastic} corruption.
Such a setting is reminiscent of \textit{learning from crowds}, 
where expert behavior is commonly modeled under a fixed class-label scheme through expert-specific confusion or output layers \cite{DawidSkene1979,Rodrigues2018}, or through expertise coupled with instance difficulty \cite{Whitehill2009,Welinder2010}.
Yet these strategies rely on a restricted class-label channel and are therefore not directly transferable to deep constraint embedding, 
where a highly flexible pairwise learner can readily absorb structured distortion and random corruption into its learned geometry.
We therefore develop \textit{ECI-PP}, a multi-role iterative framework built on \textit{ProbPair} that uses model-side beliefs to correct structured distortion and filter residual inconsistency, yielding refined surrogate supervision that progressively guides learning toward the underlying aleatoric relation.

\section{UPCC formulation and canonical relation}
\label{sec:problem_formulation}


We study uncertainty-aware probabilistic constrained clustering (UPCC), where expert-indexed real-valued pairwise labels arise from a heterogeneous observation process.
Let $\mathcal{X}=\{x_j\}_{j=1}^{|\mathcal{X}|}$ be a dataset to be partitioned into $C$ clusters, and let $\mathcal{C}=\{(a_i,b_i,e_i,y_i)\}_{i=1}^{|\mathcal{C}|}$ denote the observed soft constraints, where $(a_i,b_i)$ indexes the pair $(x_{a_i},x_{b_i})$, $e_i\in\{1,\dots,E\}$ records the expert identity, and $y_i\in[0,1]$ is the observed pairwise relation.
The formulation targets a canonical pairwise probabilistic relation, while a final hard partition is treated as a downstream summary induced by that relation.

\subsection{Observation model}
\label{sec:observation}

We view each observed pairwise relation as arising from three factors: (i) the \textit{canonical aleatoric relation} in the data (e.g., intrinsic clinical similarities between a common cold and the flu), (ii) an expert-conditioned \textit{epistemic judgment} built on that relation (e.g., a novice doctor systematically overrating their similarity due to limited knowledge), and (iii) a subsequent \textit{stochastic corruption} stage accounting for accidental recording noise (e.g., random clerical errors in medical records).

\paragraph{Aleatoric relation.}
Let $\mathcal{S}$ be the set of all $C$-partitions of $\mathcal{X}$, and let $S\in\mathcal{S}$ be a latent random partition. For any index pair $(a,b)$, define the aleatoric co-membership relation
\[
R^{\star}_{ab}
:=
\Pr\!\big(\text{$x_a$ and $x_b$ belong to the same cluster under } S \mid \mathcal{X}\big)\in[0,1].
\]
This quantity defines the canonical aleatoric relation and remains well defined even when the latent partition is intrinsically ambiguous.

\paragraph{Epistemic judgment.}
Let $\sigma(\cdot)$ denote the sigmoid function, and define its clamped inverse
$\logit_{\varepsilon}(u):=\logit\!\big(\min\{\max\{u,\varepsilon\},1-\varepsilon\}\big)$ with $\varepsilon \ll 1$.
We also introduce a deterministic pair descriptor
$\phi_{ab}=\phi(x_a,x_b)$, which serves as the covariate through which expert-dependent effects are parameterized.
For expert $e$, let $m_e(\phi_{ab})\in\mathbb{R}$ denote a structured mean effect, and let
\[
u_{e,ab}\mid e,\phi_{ab}\sim\mathcal{N}\!\big(0,\tau_e^2(\phi_{ab})\big),
\qquad
\tau_e(\phi_{ab})>0\ \text{(conditional standard deviation)},
\]
be a centered residual perturbation.
The latent expert judgment is generated as
\begin{equation}
\label{eq:latent_judgment_v4}
y^{\mathrm{jud}}_{e,ab}
:=
\sigma\!\big(\logit_{\varepsilon}(R^{\star}_{ab})+m_e(\phi_{ab})+u_{e,ab}\big)\in(0,1).
\end{equation}
Under this construction, the judgment channel has the logistic-normal density
\begin{equation}
\label{eq:judgment_density_v4}
p_{\mathrm{jud}}\big(y\mid R^{\star}_{ab},e,\phi_{ab}\big)
=
\frac{1}{y(1-y)}\,
\varphi\!\Big(\logit(y);\,\logit_{\varepsilon}(R^{\star}_{ab})+m_e(\phi_{ab}),\,\tau_e^2(\phi_{ab})\Big)
\end{equation}
for $y\in(0,1)$, where $\varphi(\cdot;\mu,s^2)$ denotes the Gaussian density; see Appendix~\ref{app:derivation_jud_density_v4} for the derivation.
The same judgment channel also gives rise to the corresponding mean distortion in probability space,
\begin{equation}
\label{eq:mean_distortion_v4}
\bar\epsilon_{e,ab}(R^{\star}_{ab},\phi_{ab})
:=
\mathbb{E}\!\left[y^{\mathrm{jud}}_{e,ab}-R^{\star}_{ab}\mid R^{\star}_{ab},e,\phi_{ab}\right].
\end{equation}

\paragraph{Stochastic corruption.}
To capture unstructured recording noise beyond the latent judgment, we introduce a corruption indicator $c_{e,ab}\sim \mathrm{Bernoulli}(\pi_e)$, where $\pi_e\in[0,1)$ is the corruption probability for expert $e$, and define the final observed relation by
\[
y_{e,ab}\mid c_{e,ab}=0\sim p_{\mathrm{jud}}(\cdot\mid R_{ab}^{\star},e,\phi_{ab}),
\qquad
y_{e,ab}\mid c_{e,ab}=1\sim \mathrm{Uniform}(0,1).
\]
Equivalently, letting $\mathbf{1}_{[0,1]}(\cdot)$ denote the uniform density on $[0,1]$, the conditional observation law is
\begin{equation}
\label{eq:observation_law_v4}
p\big(y\mid R_{ab}^{\star},e,\phi_{ab}\big)
=
(1-\pi_e)\,p_{\mathrm{jud}}\big(y\mid R_{ab}^{\star},e,\phi_{ab}\big)
+
\pi_e\,\mathbf{1}_{[0,1]}(y).
\end{equation}

\subsection{Canonical relation and identifiability}
\label{sec:identifiability}
We now formalize the canonical relation within a structural setup for identifiability.
To obtain a finite-dimensional canonical estimand, we parameterize
\[
R^{\star}_{ab}=r_{\theta}(x_a,x_b),
\qquad
m_e(\phi_{ab})=m_{\eta_e}(\phi_{ab}),
\qquad
\tau_e(\phi_{ab})=\tau_{\zeta_e}(\phi_{ab}),
\]
and collect the parameters into
$\Xi=(\theta,\{\eta_e\}_{e=1}^{E},\{\zeta_e\}_{e=1}^{E},\{\pi_e\}_{e=1}^{E})$.
This parameterization makes the canonical relation explicit under the specific observation family.
Since $\phi_{ab}$ is introduced only as a deterministic pair descriptor without additional representational constraints, the identification of the canonical relation term $r_{\theta}(x_a,x_b)$ is governed by the structural conditions introduced below.

We first fix the trivial additive ambiguity in the location decomposition:
\begin{assumption}[Centering on observable support]
\label{ass:centering_v3}
    For each expert $e$ with nonzero observation probability, $\mathbb{E}_{(a,b)\mid e}[m_{\eta_e}(\phi_{ab})]=0$, where the expectation is taken over the observable-pair distribution conditional on expert $e$.
\end{assumption}

This centering removes any remaining expert-specific constant shift on the observable support:
\begin{lemma}[Gauge fixing]
\label{lem:gauge_fix_v3}
    Suppose both $\{m_{\eta_e}\}_{e=1}^{E}$ and $\{m_{\eta'_e}\}_{e=1}^{E}$ satisfy Assumption~\ref{ass:centering_v3}. If for each expert $e$ with nonzero observation probability there exists a constant $k_e\in\mathbb{R}$ such that
    $m_{\eta_e}(\phi_{ab})=m_{\eta'_e}(\phi_{ab})-k_e$
    for all observable triples $(a,b,e)$, then $k_e=0$ for every such expert $e$.
\end{lemma}

The proof is given in Appendix~\ref{app:proof_gauge_fix_v4}. 
The logistic-normal clean channel in \cref{eq:judgment_density_v4}, combined with the observation law in \cref{eq:observation_law_v4}, yields the following injectivity property:

\begin{lemma}[Injectivity of the observation family]
\label{lem:mix_injective_v3}
    Let $\mu:=\logit_{\varepsilon}(R^\star_{ab})+m_e(\phi_{ab})$, $s:=\tau_e(\phi_{ab})$, and $\pi:=\pi_e$. For the observation law rewritten as
    \begin{equation}
    \label{eq:mix_family_v4}
    p(y\mid \mu,s,\pi)
    :=
    (1-\pi)\,\frac{1}{y(1-y)}\,\varphi\!\big(\logit(y);\mu,s^2\big)
    +\pi\,\mathbf{1}_{[0,1]}(y),
    \qquad y\in(0,1),
    \end{equation}
    the mapping $(\mu,s,\pi)\mapsto p(\cdot\mid \mu,s,\pi)$ is injective on $\mathbb{R}\times(0,\infty)\times[0,1)$.
\end{lemma}

The proof is given in Appendix~\ref{app:proof_mix_injective_v4}.
The canonical formulation is completed by the following structural separability condition on the location term $\logit_{\varepsilon}\!\big(r_{\theta}(x_a,x_b)\big)+m_{\eta_e}(\phi_{ab})$:

\begin{assumption}[Structural separability]
\label{ass:separability_v3}
Under Assumption~\ref{ass:centering_v3}, if for all observable triples $(a,b,e)$,
\[
\logit_{\varepsilon}\!\big(r_{\theta}(x_a,x_b)\big)+m_{\eta_e}(\phi_{ab})
=
\logit_{\varepsilon}\!\big(r_{\theta'}(x_a,x_b)\big)+m_{\eta'_e}(\phi_{ab}),
\]
then for each expert $e$ there exists a constant $k_e\in\mathbb{R}$ such that for all such observable triples,
\[
\logit_{\varepsilon}\!\big(r_{\theta}(x_a,x_b)\big)
=
\logit_{\varepsilon}\!\big(r_{\theta'}(x_a,x_b)\big)+k_e,
\qquad
m_{\eta_e}(\phi_{ab})=m_{\eta'_e}(\phi_{ab})-k_e.
\]
\end{assumption}

Assumption~\ref{ass:separability_v3} provides the key structural requirement for the identifiability argument; boundary cases clarifying this condition are discussed in Appendix~\ref{app:structural_failure}.
Writing the support of observable pairs as
$\mathcal{P}_{\mathrm{obs}}:=\{(a,b): \exists e \text{ such that } (a,b,e) \text{ is observable}\}$,
we then have the identifiability result:

\begin{theorem}[Conditional canonical identifiability of the aleatoric relation]
\label{thm:canonical_ident_v3}
Let $\Xi$ and $\Xi'$ satisfy the observation model and Assumptions~\ref{ass:centering_v3} and~\ref{ass:separability_v3}. Assume further that $r_{\theta}(x_a,x_b),\,r_{\theta'}(x_a,x_b)\in[\varepsilon,1-\varepsilon]$ for all $(a,b)\in\mathcal{P}_{\mathrm{obs}}$. If the induced conditional laws of the observed response coincide under $\Xi$ and $\Xi'$ for every observable triple $(a,b,e)$, then $r_{\theta}(x_a,x_b)=r_{\theta'}(x_a,x_b)$ for all $(a,b)\in\mathcal{P}_{\mathrm{obs}}$. Hence the canonical aleatoric relation $R^{\star}_{ab}=r_{\theta}(x_a,x_b)$ is identifiable on observable pairs.
\end{theorem}

The proof is in Appendix~\ref{app:proof_canonical_ident_v4}. Under the specific observation family, \Cref{thm:canonical_ident_v3} gives conditional identifiability of the canonical aleatoric relation on observable pairs at the law level. In the practical single-observation regime, however, estimation still relies on the shared parametric structure imposed across pairs and experts rather than repeated observations of the same pair--expert combination.

\paragraph{Implication for learning.} The analysis above specifies the expert-independent relation that learning should target and clarifies when it is distinguishable, at the observation-law level, from structured expert effects and stochastic corruption. It does not imply that these latent components can be separately recovered from a single observation of each pair--expert triple. Instead, it motivates a practical learning strategy that exploits shared structure across pairs and experts: construct a model-side proxy for the canonical relation, use expert-conditioned regularities to correct predictable observational distortion, and treat the remaining unexplained discrepancy conservatively as a reliability signal. \Cref{sec:method} operationalizes these principles through practical surrogate learning rather than explicit estimation of the full latent observation model.

\section{Method}
\label{sec:method}

\Cref{sec:problem_formulation} provides both the learning target and the structural guidance for learning from entangled supervision: the canonical relation is the expert-independent quantity of interest, whereas expert-conditioned distortion and residual corruption should not be indiscriminately absorbed into the learned pairwise geometry. We translate this guidance into two stages. We first introduce \textit{ProbPair}, a base representation learner that gives unified probabilistic pairwise relations a direct angular learning objective. We then develop \textit{ECI-PP}, which uses model-side relation estimates, expert-conditioned correction, reliability-aware screening, and iterative integration to refine the observed supervision toward the canonical relation without explicitly fitting every latent component of the observation law.

\subsection{ProbPair: probability-aware pairwise representation learning}
\label{sec:probpair}

We establish a representation learning formulation under probabilistic pairwise supervision. Given supervised pairs $\{(a_i,b_i,y_i)\}_{i=1}^{|\mathcal{C}|}$, where $y_i\in[0,1]$ denotes the probabilistic relation for pair $(x_{a_i},x_{b_i})$, we aim to learn clustering-oriented representations. To this end, we adopt a deep constraint embedding formulation with encoder $f_{\psi}:\mathcal{X}\to\mathcal{Z}\subset\mathbb{R}^{D}$ and latent code $z_j=f_{\psi}(x_j)$. We work in angular space, where bounded cosine similarity provides a natural geometric quantity for probabilistic calibration. To connect probabilistic pairwise relations with the latent representation space, we introduce the following probabilistic readout:
\begin{equation}
\label{eq:probpair_readout}
\hat{y}_{a_i b_i}
:=
\sigma\big((\cos(z_{a_i},z_{b_i})-m)/T\big)\in(0,1),
\end{equation}
where $m\in\mathbb{R}$ and $T>0$ are learnable parameters controlling the operating point and sharpness of the mapping. The resulting pairwise objective, referred to as the \textit{ProbPair} loss, is given by
\begin{equation}
\label{eq:probpair_loss}
\mathcal{L}_{\mathrm{PP}}
=
-\frac{1}{|\mathcal{C}|}
\sum_{i=1}^{|\mathcal{C}|}
\Big[
y_i\log \hat{y}_{a_i b_i}
+
(1-y_i)\log(1-\hat{y}_{a_i b_i})
\Big].
\end{equation}
Unlike hard-constraint objectives, \cref{eq:probpair_loss} directly accommodates non-binary relation strengths. As a result, the induced geometry need not collapse to a purely binary regime, but can preserve both confident and uncertain pairwise structure.

We further include a reconstruction regularizer to avoid degenerate representations and preserve the global structure of the data. Since the representation is learned in angular space, reconstruction is performed from normalized latent embeddings, as in \cite{SpherePair}. Specifically, with decoder $g_{\psi'}:\mathcal{Z}\to\mathcal{X}$ and $\mathrm{Norm}(\cdot)$ denoting $\ell_2$ normalization, we decode from $\hat{x}_j=g_{\psi'}(\mathrm{Norm}(z_j))$ and use reconstruction loss $\mathcal{L}_{\mathrm{rec}}=\frac{1}{|\mathcal{X}|}\sum_{j=1}^{|\mathcal{X}|}\|x_j-\hat{x}_j\|_2^2$. 
With $\lambda_{\mathrm{rec}}>0$
balancing pairwise relation fitting and reconstruction regularization, 
the overall objective is
\begin{equation}
\label{eq:probpair_total}
\mathcal{L}
=
\mathcal{L}_{\mathrm{PP}}
+
\lambda_{\mathrm{rec}}\mathcal{L}_{\mathrm{rec}}.
\end{equation}
Minimizing \cref{eq:probpair_total} yields optimal normalized embeddings $\mathcal{Z}_{\mathrm{norm}}\!=\!\{\mathrm{Norm}(z_j)\}_{j=1}^{|\mathcal{X}|}$ that are shaped by probabilistic pairwise relations while being regularized by reconstruction. These angular embeddings can then be used by a downstream clustering algorithm to obtain the final partition, and can also produce probabilistic relation estimates for instance pairs through \cref{eq:probpair_readout}.

\subsection{Estimator--Corrector--Integrator ProbPair}
\label{sec:ECI}
Building on ProbPair, ECI-PP operationalizes the decomposition in \cref{sec:problem_formulation} through practical proxies: the \textit{Estimator} supplies an out-of-fold model-side proxy for the canonical relation; the \textit{Corrector} captures predictable expert-conditioned mean discrepancy; and the \textit{Integrator} treats the remaining inconsistency as a reliability signal when constructing refined supervision. The resulting supervision is then fed back to refresh the model-side reference, yielding a finite-stage iterative refinement procedure. These modules are motivated by the roles of the components in the observation model, but are deliberately practical surrogates rather than explicit estimators of every latent term in \crefrange{eq:latent_judgment_v4}{eq:observation_law_v4}. \Cref{fig:eci} summarizes this correspondence and the resulting ECI-PP pipeline.

\begin{figure}[t]
    \begin{center}
    \centerline{\includegraphics[width=\textwidth]{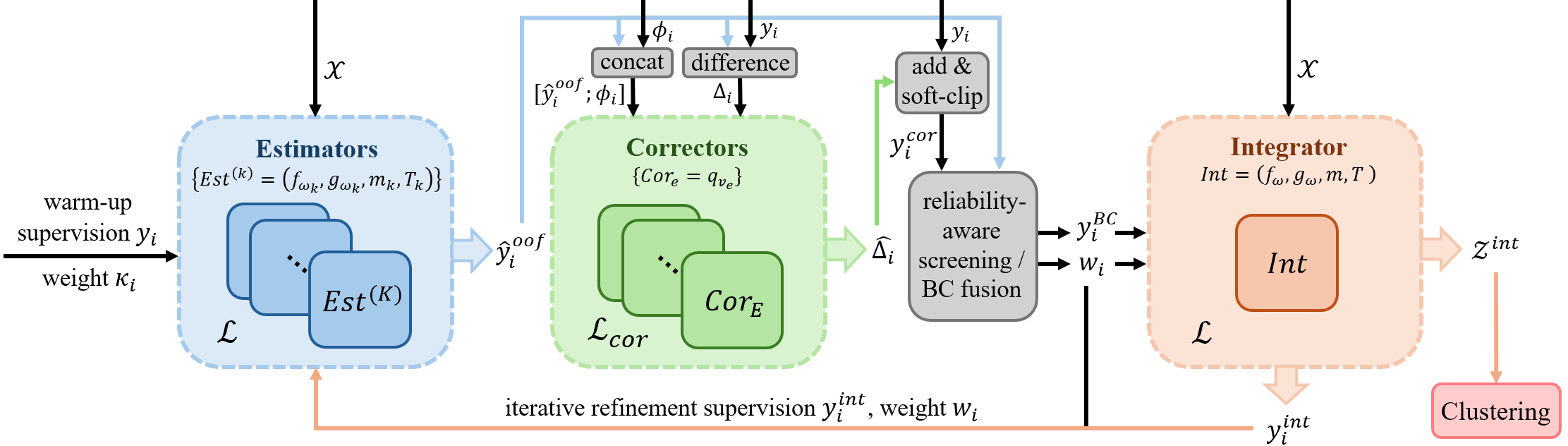}}
    \caption{Overview of the ECI-PP pipeline. The \textcolor{estblue}{\textbf{Estimators}} are trained with the ProbPair objective $\mathcal{L}$ (\cref{eq:probpair_total}) and produce out-of-fold beliefs $\hat{y}_i^{\mathrm{oof}}$ via the readout in \cref{eq:probpair_readout}. 
    The \textcolor{corgreen}{\textbf{Correctors}} take $[\hat{y}_i^{\mathrm{oof}};\phi_i]$ as input, with $\phi_i$ the pair-$i$ descriptor, and learn probability-space corrections $\hat{\Delta}_i$ via \cref{eq:eci_corrector_loss}, yielding corrected relations $y_i^{\mathrm{cor}}$.
    \textcolor{midgray}{\textbf{Reliability-aware screening}} produces reliability weights $w_i$, and \textcolor{midgray}{\textbf{BC fusion}} forms the refined supervision $y_i^{\mathrm{BC}}$ via \cref{eq:eci_int_screening,eq:eci_int_bc}. The \textcolor{intorange}{\textbf{Integrator}} learns embeddings from the refined supervision; its Integrator-side relations $y_i^{\mathrm{int}}$ are fed back to the \textcolor{estblue}{\textbf{Estimators}} for iterative refinement, while the final embeddings $\mathcal{Z}^{\mathrm{int}}$ are \textcolor{clusred}{\textbf{clustered}} into the final partition.}
    \label{fig:eci}
    \end{center}
    \vskip -0.2in
\end{figure}

\paragraph{Estimator.}
To enable the subsequent learning of structured observational effects from entangled supervision, ECI-PP first constructs a model-side reference that is less coupled to the raw observations. In flexible pairwise representation learning, beliefs fitted on the same constraints used for training can absorb noisy observations in-sample and are therefore unreliable for this role. We accordingly partition $\mathcal{C}$ into $K$ disjoint folds $\{\mathcal{C}^{(k)}\}_{k=1}^{K}$ and instantiate $K$ non-shared Estimators $\{\mathrm{Est}^{(k)}\}_{k=1}^{K}$, where $\mathrm{Est}^{(k)}\!=\!(f_{\omega_k},g_{\omega'_k},m_k,T_k)$ denotes the encoder, decoder, and fold-specific readout parameters. Each $\mathrm{Est}^{(k)}$ is trained on $\mathcal{C}\setminus\mathcal{C}^{(k)}$ by minimizing $\mathcal{L}$ in \cref{eq:probpair_total}, and then produces holdout beliefs on $\mathcal{C}^{(k)}$ through \cref{eq:probpair_readout}. Aggregating these predictions over all folds yields a full set of out-of-fold beliefs $\{\hat{y}_i^{\mathrm{oof}}\}_{i=1}^{|\mathcal{C}|}$, which serves as the model-side reference. 
To initialize the estimator, we weight each observed pair in ProbPair loss by an information-based decisiveness score, defined as
\begin{equation}
\label{eq:eci_infoweight}
\kappa_i
=
\frac{D_{\mathrm{KL}}\!\big(\mathrm{Bern}(y_i)\,\Vert\,\mathrm{Bern}(\bar{y})\big)}
{(1-y_i)\,D_{\mathrm{KL}}\!\big(\mathrm{Bern}(0)\,\Vert\,\mathrm{Bern}(\bar{y})\big)
+y_i\,D_{\mathrm{KL}}\!\big(\mathrm{Bern}(1)\,\Vert\,\mathrm{Bern}(\bar{y})\big)}
\in [0,1],
\end{equation}
where $\bar{y}=\frac{1}{|\mathcal{C}|}\sum_{i=1}^{|\mathcal{C}|}y_i$, $D_{\mathrm{KL}}(\cdot\Vert\cdot)$ denotes the Kullback--Leibler divergence, and $\mathrm{Bern}(\cdot)$ denotes the Bernoulli law. This weighting emphasizes more decisive judgments, which are typically expressed with greater expert confidence, and thereby stabilizes the initial out-of-fold beliefs.

\paragraph{Corrector.}
Following the theory-to-surrogate mapping above, the Corrector targets the \emph{predictable mean effect} induced by the expert-conditioned judgment channel, rather than attempting to recover an individual residual realization or the full conditional distribution in \crefrange{eq:latent_judgment_v4}{eq:mean_distortion_v4}. For a constraint $i=(e_i,a_i,b_i)$, we define the empirical discrepancy $\Delta_i:=\hat{y}_i^{\mathrm{oof}}-y_i\in(-1,1)$ and introduce expert-specific Correctors $\{\mathrm{Cor}_e\}_{e=1}^{E}$, where each $\mathrm{Cor}_e$ is parameterized by a small network $q_{\nu_e}$ on $\mathcal{I}_e:=\{i\!:\!e_i\!=\!e\}$ and conditioned on the out-of-fold belief and pair descriptor $[\hat{y}_i^{\mathrm{oof}};\phi_i]$. Here we instantiate $\phi_i=[x_{a_i}\!\odot\!x_{b_i};\,|x_{a_i}\!-\!x_{b_i}|]$, with $\odot$ denoting element-wise product, and write the Corrector output as $\hat{\Delta}_i:=q_{\nu_{e_i}}([\hat{y}_i^{\mathrm{oof}};\phi_i])\in(-1,1)$. Although the structured expert effect is defined in logit space in \cref{sec:observation}, we perform correction in probability space for practical stability: small probability differences near $0$ or $1$ can correspond to disproportionately large logit residuals. 
Thus, $q_{\nu_{e_i}}$ approximates a probability-space correction for the predictable mean distortion induced by the expert-conditioned channel, rather than explicitly estimating its logit-space residual or heteroscedastic variance $\tau_{e_i}^{2}(\phi_i)$.
Importantly, $\Delta_i$ need not contain only structured expert effects; it may also include stochastic corruption and Estimator proxy error. We therefore use a small regularized Corrector to favor recurring expert-conditioned structure over arbitrary pair-specific fluctuations, trained as
\begin{equation}
\label{eq:eci_corrector_loss}
\mathcal{L}_{\mathrm{cor}}^{(e)}
=
\frac{1}{|\mathcal I_e|}
\sum_{i\in\mathcal I_e}
\Big[
\rho (\hat{\Delta}_i-\Delta_i)
+
\lambda_{\mathrm{cor}} |\hat{\Delta}_i |^2
\Big].
\end{equation}
Here, $\rho(\cdot)$ is the Huber loss and $\lambda_{\mathrm{cor}}>0$ weights the $\ell_2$ regularizer, discouraging spurious corrections.

\paragraph{Integrator.}
The Integrator, parameterized as $\mathrm{Int}=(f_{\omega},g_{\omega'},m,T)$, is the learner that integrates the refined supervision into a clustering-oriented embedding. To build this refined supervision from the Corrector output, we first form a corrected relation $y_i^{\mathrm{cor}}$, quantify the remaining inconsistency by $\mathrm{gap}_i$, and convert it into a screening weight $w_i$, with $\gamma>0$ controlling the sharpness of down-weighting:
\begin{equation}
\label{eq:eci_int_screening}
y_i^{\mathrm{cor}}:=\operatorname{softclip}_{\xi}(y_i+\hat{\Delta}_i),\qquad
\mathrm{gap}_i:=\left|\operatorname{softclip}_{\xi}(\hat{y}_i^{\mathrm{oof}})-y_i^{\mathrm{cor}}\right|,\qquad
w_i:=(1-\mathrm{gap}_i)^{\gamma}.
\end{equation}
Here, $\operatorname{softclip}_{\xi}(u):=\xi^{-1}\operatorname{softplus}(\xi u)-\xi^{-1}\operatorname{softplus}(\xi(u-1))$ with $\xi>0$ softly maps values into $(0,1)$ while preserving gradient flow near the boundaries; applying the same map to $\hat{y}_i^{\mathrm{oof}}$ keeps the gap computation in the same bounded space.
The residual $\mathrm{gap}_i$ quantifies the unexplained discrepancy that remains after structured correction. 
While the observation model associates this discrepancy with stochastic corruption and residual epistemic variation, $\mathrm{gap}_i$ may also reflect proxy error in $\hat{y}_i^{\mathrm{oof}}$ and imperfect Corrector fitting in practice. We therefore treat it uniformly as a reliability signal for heuristic screening through $w_i$. Accordingly, $w_i$ should not be interpreted as a posterior corruption probability.
We then set $n_i:=n_0w_i$ with $n_0>0$, and construct the Bayesian-confidence supervision
\begin{equation}
\label{eq:eci_int_bc}
y_i^{\mathrm{BC}}:=(y_i+n_i\,y_i^{\mathrm{cor}})/(1+n_i),
\end{equation}
which keeps $y_i$ as a fixed prior while letting $y_i^{\mathrm{cor}}$ contribute according to its screened reliability. 
The Integrator then optimizes $\mathcal{L}$ (\cref{eq:probpair_total}) using $y_i^{\mathrm{BC}}$ as the pairwise target, with $w_i$ multiplying each per-constraint term in $\mathcal{L}_{\mathrm{PP}}$ (\cref{eq:probpair_loss}) before averaging, yielding Integrator embeddings.

\paragraph{Iterative refinement and deployment.}
ECI-PP operates by passing supervision through the Estimator--Corrector--Integrator pipeline to produce, via \cref{eq:probpair_readout}, Integrator-side relations $y_i^{\mathrm{int}}$, and then feeding $(y_i^{\mathrm{int}},w_i)$ back to the Estimators for iterative refinement. Warm-up initializes this loop by starting the Estimators from $(y_i,\kappa_i)$; subsequent rounds replace this with $(y_i^{\mathrm{int}},w_i)$. 
At each round, the refreshed Estimator beliefs are passed again through the Corrector--Integrator pipeline to update $y_i^{\mathrm{cor}}$, $\mathrm{gap}_i$, $w_i$, and $y_i^{\mathrm{BC}}$,
thereby yielding a practical finite-stage procedure that iteratively refines the surrogate supervision toward the canonical relation.
After refinement, the final normalized Integrator embeddings $\mathcal{Z}^{\mathrm{int}}:=\{\mathrm{Norm}(f_{\omega}(x_j))\}_{j=1}^{|\mathcal{X}|}$ encode the learned pairwise relation structure; applying an unsupervised clustering algorithm to $\mathcal{Z}^{\mathrm{int}}$ then yields the final partition.
Appendix~\ref{app:alg_complexity} provides algorithmic details and complexity analysis.

\section{Experiments}
\label{sec:experiment}

\subsection{Experimental settings}
\label{sec:exp_settings}

Our experiments evaluate ECI-PP under UPCC along four axes: (i) comparison with DCC and probabilistic pairwise baselines, (ii) robustness and complementary supervision tests, (iii) held-out diagnostics of internal dynamics, and (iv) ablation and sensitivity analysis.

\textbf{Datasets.} We adopt six image benchmarks (CIFAR100-20, CIFAR10 \cite{CIFAR}, FMNIST \cite{FMNIST}, ImageNet10 \cite{Imagenet10}, MNIST \cite{MNIST}, STL10 \cite{STL10}) and two text benchmarks (Reuters subset \cite{DEC}, RCV1-10 \cite{SpherePair}), covering varying class counts and class imbalance.
We additionally use CIFAR-10H \citep{CIFAR10H} for evaluation under human-derived probabilistic supervision.
Dataset details are in Appendix~\ref{app:datasets}.


\textbf{Compared methods.} We compare our progressive models (ProbPair, Weighted ProbPair, and the full ECI-PP framework) against four external baselines: VanillaDCC \cite{MCL} as a basic logistic pairwise baseline, VolMaxDCC \cite{VolMaxDCC} with explicit noise modeling, and CIDEC \cite{CIDEC1} and SpherePair \cite{SpherePair} as state-of-the-art end-to-end DCC and deep constraint embedding methods, respectively.
ProbPair uses only $\mathcal{L}$ in \cref{eq:probpair_total}, while Weighted ProbPair additionally uses the information-based weight in \cref{eq:eci_infoweight}.

\textbf{Constraint generation.} To instantiate UPCC, we generate constraints with trained expert classifiers rather than deriving binary labels from ground-truth classes.
Experts are trained on controlled labeled subsets; uniformly sampled constraint pairs assigned to expert $e$ are given a judgment $y^{\mathrm{jud}}_{e,ab}=\langle \boldsymbol{p}^{e}_a,\boldsymbol{p}^{e}_b\rangle\in[0,1]$ from the predictive distributions, and the recorded $y_i$ is then produced through the corruption channel in \cref{sec:observation}.
For single-expert supervision, settings {\tt lv0.1}/{\tt lv0.01}/{\tt lv0.001} vary the labeled-data fraction used to train the expert.
Multi-expert settings {\tt multi2}/{\tt multi3}/{\tt multi10} use multiple experts with complementary familiar/unfamiliar categories, where unfamiliar categories act as expert blind spots.
All training constraints are generated from the corresponding training split only.
See Appendix~\ref{app:constraint_generation} for the full procedure.

\textbf{Protocol.} For a comparative study, we report clustering metrics (ACC/NMI/ARI) over five trials; embedding-based methods use K-means on learned representations, while end-to-end DCC baselines use native outputs.
The main comparison uses the default {\tt lv0.01} single-expert and {\tt multi3} multi-expert regimes, both with corruption probability $0.3$ and 9k constraints in total ($3$k per expert under {\tt multi3}).
Robustness tests vary single-expert quality, corruption probability, and multi-expert configuration with matched budgets; complementary tests probe generator misspecification through alternative relation constructions and corruption mechanisms, as well as human-derived supervision.
We also conduct held-out diagnostics, ablation, and sensitivity analysis.

\textbf{Implementation.} For fair comparison, we follow established fully connected architectures and the unsupervised pretraining protocol in \cite{SpherePair,VolMaxDCC}: encoder-based methods use a $500$--$500$--$2000$ backbone with embedding dimension $10$ except $20$ for CIFAR100-20, while VanillaDCC and VolMaxDCC use their native $512$--$512$ classifier architectures.
For baselines, we use reported best settings where available; VolMaxDCC follows its original validation-based search with $1{,}000$ training samples held out.
Except in sensitivity analysis, ECI-PP uses \textbf{one global setting} without dataset-specific tuning: $5$ estimator folds, Corrector $64$--$16$, $\lambda_{\mathrm{cor}}=0.5$, $\gamma=10$, $n_0=10$, $\xi=20$, and $\lambda_{\mathrm{rec}}=0.02$.

Full experimental details are provided in Appendix~\ref{app:experimental_settings} to support reproducibility.

\subsection{Experimental results}
\label{sec:exp_results}

\begin{table}[t]
    \caption{Comparative performance (\%) (ACC, NMI, ARI) across datasets for external baselines and our progressive models (ProbPair, Weighted ProbPair, and ECI-PP) under {\tt lv0.01} and {\tt multi3} settings, with corruption probability $0.3$ and $9$k constraints. \textcolor{blue}{Blue} and black represent \textcolor{blue}{training} and test results, respectively. Best results are in \textbf{bold}, and second-best are \underline{underlined}.}
    \label{table:exp1_9k_main}
    \vskip -0.1in
    \begin{center}
    \begin{scriptsize}
    \renewcommand{\arraystretch}{0.0}
    \setlength{\tabcolsep}{4.1pt}
    \setlength{\aboverulesep}{0.pt}
    \setlength{\belowrulesep}{0.pt}
    \setlength{\extrarowheight}{0pt}
    \begin{tabular}{lllccccccc}
        \toprule
        \addlinespace[2pt]
        & & & \multicolumn{4}{c}{\textbf{External baselines}} & \multicolumn{3}{c}{\textbf{Our progressive models}} \\
        \addlinespace[2pt]
        \cmidrule(lr){4-7}\cmidrule(lr){8-10}
        & & & \makecell{Vanilla-\\DCC} & \makecell{VolMax-\\DCC} & CIDEC & SpherePair & ProbPair & \makecell{Weighted\\ProbPair} & \makecell{\textbf{ECI-PP}\\(Full)} \\
        \midrule
        \multirow{24}{*}{\raisebox{-23.0ex}{\makecell{Single-\\expert\\{\tt lv0.01}}}} & \multirow{3}{*}{\raisebox{-1ex}{CIFAR100-20}}
        & ACC & \textcolor{blue}{16.8}, 17.0 & \textcolor{blue}{46.1}, 45.9 & \textcolor{blue}{14.3}, 13.6 & \textcolor{blue}{50.2}, 50.5 & \textcolor{blue}{49.4}, 49.5 & \textcolor{blue}{\underline{50.3}}, \underline{50.8} & \textcolor{blue}{\textbf{52.7}}, \textbf{52.9} \\
        & & NMI & \textcolor{blue}{13.5}, 14.6 & \textcolor{blue}{\underline{45.6}}, \underline{46.1} & \textcolor{blue}{11.0}, 11.2 & \textcolor{blue}{43.0}, 44.2 & \textcolor{blue}{43.2}, 44.2 & \textcolor{blue}{44.5}, 45.7 & \textcolor{blue}{\textbf{49.2}}, \textbf{49.9} \\
        & & ARI & \textcolor{blue}{6.0}, 6.3 & \textcolor{blue}{30.4}, 30.5 & \textcolor{blue}{2.8}, 2.3 & \textcolor{blue}{\underline{32.1}}, \underline{32.6} & \textcolor{blue}{30.6}, 30.9 & \textcolor{blue}{31.5}, 32.1 & \textcolor{blue}{\textbf{36.2}}, \textbf{36.5} \\
        \cmidrule(lr){2-10}
        & \multirow{3}{*}{\raisebox{-1ex}{CIFAR10}}
        & ACC & \textcolor{blue}{38.3}, 38.7 & \textcolor{blue}{82.5}, 82.5 & \textcolor{blue}{44.3}, 44.0 & \textcolor{blue}{84.9}, 85.7 & \textcolor{blue}{85.9}, 86.0 & \textcolor{blue}{\underline{86.6}}, \underline{86.7} & \textcolor{blue}{\textbf{87.8}}, \textbf{87.8} \\
        & & NMI & \textcolor{blue}{29.9}, 31.2 & \textcolor{blue}{71.9}, 72.2 & \textcolor{blue}{36.1}, 37.6 & \textcolor{blue}{72.4}, 74.1 & \textcolor{blue}{75.0}, 75.4 & \textcolor{blue}{\underline{76.2}}, \underline{76.8} & \textcolor{blue}{\textbf{79.0}}, \textbf{79.1} \\
        & & ARI & \textcolor{blue}{23.2}, 24.0 & \textcolor{blue}{68.3}, 68.3 & \textcolor{blue}{21.1}, 21.3 & \textcolor{blue}{71.2}, 72.5 & \textcolor{blue}{73.0}, 73.2 & \textcolor{blue}{\underline{74.3}}, \underline{74.5} & \textcolor{blue}{\textbf{76.7}}, \textbf{76.8} \\
        \cmidrule(lr){2-10}
        & \multirow{3}{*}{\raisebox{-1ex}{FMNIST}}
        & ACC & \textcolor{blue}{39.8}, 40.2 & \textcolor{blue}{66.5}, 65.5 & \textcolor{blue}{36.1}, 35.7 & \textcolor{blue}{72.0}, 71.2 & \textcolor{blue}{\underline{73.6}}, \underline{72.5} & \textcolor{blue}{\textbf{74.6}}, \textbf{73.6} & \textcolor{blue}{72.4}, 71.1 \\
        & & NMI & \textcolor{blue}{30.1}, 31.3 & \textcolor{blue}{57.2}, 56.8 & \textcolor{blue}{29.7}, 30.5 & \textcolor{blue}{60.4}, 60.9 & \textcolor{blue}{62.9}, 62.9 & \textcolor{blue}{\underline{63.9}}, \underline{64.0} & \textcolor{blue}{\textbf{65.7}}, \textbf{65.0} \\
        & & ARI & \textcolor{blue}{21.5}, 22.2 & \textcolor{blue}{46.4}, 45.2 & \textcolor{blue}{15.7}, 15.6 & \textcolor{blue}{53.0}, 52.3 & \textcolor{blue}{54.8}, 53.7 & \textcolor{blue}{\underline{55.3}}, \textbf{54.3} & \textcolor{blue}{\textbf{55.5}}, \underline{53.9} \\
        \cmidrule(lr){2-10}
        & \multirow{3}{*}{\raisebox{-1ex}{ImageNet10}}
        & ACC & \textcolor{blue}{39.0}, 41.5 & \textcolor{blue}{87.4}, 88.1 & \textcolor{blue}{63.5}, 68.3 & \textcolor{blue}{85.1}, 89.1 & \textcolor{blue}{\underline{89.2}}, \underline{90.5} & \textcolor{blue}{89.1}, 90.3 & \textcolor{blue}{\textbf{92.8}}, \textbf{93.1} \\
        & & NMI & \textcolor{blue}{27.6}, 33.7 & \textcolor{blue}{77.8}, 79.2 & \textcolor{blue}{47.2}, 57.5 & \textcolor{blue}{71.0}, 79.2 & \textcolor{blue}{80.1}, \underline{83.2} & \textcolor{blue}{\underline{80.3}}, 82.7 & \textcolor{blue}{\textbf{88.1}}, \textbf{88.5} \\
        & & ARI & \textcolor{blue}{21.1}, 25.2 & \textcolor{blue}{75.2}, 76.1 & \textcolor{blue}{35.5}, 42.2 & \textcolor{blue}{70.7}, 78.2 & \textcolor{blue}{\underline{78.5}}, \underline{80.7} & \textcolor{blue}{78.4}, 80.2 & \textcolor{blue}{\textbf{85.8}}, \textbf{86.2} \\
        \cmidrule(lr){2-10}
        & \multirow{3}{*}{\raisebox{-1ex}{MNIST}}
        & ACC & \textcolor{blue}{38.9}, 40.3 & \textcolor{blue}{75.6}, 76.9 & \textcolor{blue}{47.0}, 46.7 & \textcolor{blue}{87.8}, 89.1 & \textcolor{blue}{89.0}, 90.0 & \textcolor{blue}{\underline{89.9}}, \underline{90.8} & \textcolor{blue}{\textbf{90.7}}, \textbf{91.0} \\
        & & NMI & \textcolor{blue}{29.5}, 32.2 & \textcolor{blue}{59.3}, 61.6 & \textcolor{blue}{41.4}, 43.7 & \textcolor{blue}{75.4}, 78.0 & \textcolor{blue}{76.9}, 79.0 & \textcolor{blue}{\underline{78.3}}, \underline{80.2} & \textcolor{blue}{\textbf{80.7}}, \textbf{81.6} \\
        & & ARI & \textcolor{blue}{23.1}, 24.8 & \textcolor{blue}{56.5}, 58.7 & \textcolor{blue}{24.7}, 24.2 & \textcolor{blue}{75.4}, 77.8 & \textcolor{blue}{77.4}, 79.4 & \textcolor{blue}{\underline{79.1}}, \underline{80.9} & \textcolor{blue}{\textbf{80.8}}, \textbf{81.5} \\
        \cmidrule(lr){2-10}
        & \multirow{3}{*}{\raisebox{-1ex}{REUTERS}}
        & ACC & \textcolor{blue}{69.7}, 77.0 & \textcolor{blue}{71.7}, 76.8 & \textcolor{blue}{72.9}, 79.3 & \textcolor{blue}{77.9}, 81.8 & \textcolor{blue}{78.2}, 80.8 & \textcolor{blue}{\underline{80.4}}, \underline{82.7} & \textcolor{blue}{\textbf{85.1}}, \textbf{86.6} \\
        & & NMI & \textcolor{blue}{30.5}, 43.3 & \textcolor{blue}{34.3}, 43.4 & \textcolor{blue}{35.7}, 48.1 & \textcolor{blue}{48.7}, 56.3 & \textcolor{blue}{53.2}, 58.3 & \textcolor{blue}{\underline{54.5}}, \underline{59.2} & \textcolor{blue}{\textbf{61.4}}, \textbf{64.4} \\
        & & ARI & \textcolor{blue}{36.8}, 50.4 & \textcolor{blue}{40.9}, 50.4 & \textcolor{blue}{41.8}, 54.9 & \textcolor{blue}{52.9}, 61.1 & \textcolor{blue}{55.5}, 61.0 & \textcolor{blue}{\underline{58.6}}, \underline{63.4} & \textcolor{blue}{\textbf{66.6}}, \textbf{70.0} \\
        \cmidrule(lr){2-10}
        & \multirow{3}{*}{\raisebox{-1ex}{STL10}}
        & ACC & \textcolor{blue}{35.5}, 37.8 & \textcolor{blue}{79.0}, 80.1 & \textcolor{blue}{57.8}, 63.3 & \textcolor{blue}{76.1}, 79.7 & \textcolor{blue}{\underline{81.6}}, \underline{83.4} & \textcolor{blue}{\underline{81.6}}, 83.1 & \textcolor{blue}{\textbf{85.4}}, \textbf{85.8} \\
        & & NMI & \textcolor{blue}{22.9}, 27.6 & \textcolor{blue}{66.7}, 69.3 & \textcolor{blue}{39.3}, 48.5 & \textcolor{blue}{59.5}, 66.5 & \textcolor{blue}{\underline{68.2}}, \underline{72.1} & \textcolor{blue}{67.9}, 71.6 & \textcolor{blue}{\textbf{75.8}}, \textbf{76.6} \\
        & & ARI & \textcolor{blue}{17.1}, 19.8 & \textcolor{blue}{62.8}, 64.7 & \textcolor{blue}{29.8}, 37.3 & \textcolor{blue}{57.2}, 63.3 & \textcolor{blue}{\underline{65.8}}, \underline{68.9} & \textcolor{blue}{\underline{65.8}}, 68.4 & \textcolor{blue}{\textbf{72.8}}, \textbf{73.4} \\
        \cmidrule(lr){2-10}
        & \multirow{3}{*}{\raisebox{-1ex}{RCV1-10}}
        & ACC & \textcolor{blue}{46.2}, 46.4 & \textcolor{blue}{61.3}, 61.8 & \textcolor{blue}{42.5}, 44.3 & \textcolor{blue}{\textbf{71.2}}, \textbf{71.4} & \textcolor{blue}{52.7}, 53.0 & \textcolor{blue}{55.2}, 55.1 & \textcolor{blue}{\underline{69.0}}, \underline{68.9} \\
        & & NMI & \textcolor{blue}{15.5}, 15.9 & \textcolor{blue}{35.7}, 36.7 & \textcolor{blue}{29.5}, 30.0 & \textcolor{blue}{\underline{60.5}}, \underline{61.3} & \textcolor{blue}{53.8}, 54.1 & \textcolor{blue}{57.4}, 57.5 & \textcolor{blue}{\textbf{65.5}}, \textbf{65.8} \\
        & & ARI & \textcolor{blue}{19.8}, 20.1 & \textcolor{blue}{44.7}, 45.4 & \textcolor{blue}{26.2}, 28.3 & \textcolor{blue}{\underline{61.0}}, \underline{61.4} & \textcolor{blue}{43.6}, 43.7 & \textcolor{blue}{47.1}, 47.0 & \textcolor{blue}{\textbf{62.0}}, \textbf{62.0} \\
        \midrule
        \multirow{24}{*}{\raisebox{-23.0ex}{\makecell{Multi-\\expert\\{\tt multi3}}}} & \multirow{3}{*}{\raisebox{-1.0ex}{CIFAR100-20}} & ACC & \textcolor{blue}{15.4}, 15.8 & \textcolor{blue}{42.8}, 42.5 & \textcolor{blue}{12.8}, 12.1 & \textcolor{blue}{44.3}, 44.4 & \textcolor{blue}{43.3}, 43.4 & \textcolor{blue}{\underline{46.1}}, \underline{46.5} & \textcolor{blue}{\textbf{48.6}}, \textbf{48.8} \\
        & & NMI & \textcolor{blue}{12.8}, 13.6 & \textcolor{blue}{\underline{43.4}}, \underline{43.8} & \textcolor{blue}{10.0}, 10.5 & \textcolor{blue}{40.5}, 41.5 & \textcolor{blue}{40.0}, 40.9 & \textcolor{blue}{41.8}, 43.0 & \textcolor{blue}{\textbf{48.1}}, \textbf{48.5} \\
        & & ARI & \textcolor{blue}{5.8}, 5.8 & \textcolor{blue}{27.1}, 27.1 & \textcolor{blue}{2.1}, 2.0 & \textcolor{blue}{27.1}, 27.6 & \textcolor{blue}{26.0}, 26.2 & \textcolor{blue}{\underline{28.0}}, \underline{28.5} & \textcolor{blue}{\textbf{32.9}}, \textbf{32.9} \\
        \cmidrule(lr){2-10}
        & \multirow{3}{*}{\raisebox{-1.0ex}{CIFAR10}} & ACC & \textcolor{blue}{32.4}, 32.8 & \textcolor{blue}{76.5}, 76.6 & \textcolor{blue}{33.2}, 31.6 & \textcolor{blue}{73.8}, 74.5 & \textcolor{blue}{81.3}, \underline{81.7} & \textcolor{blue}{\underline{81.4}}, \underline{81.7} & \textcolor{blue}{\textbf{81.6}}, \textbf{82.1} \\
        & & NMI & \textcolor{blue}{29.1}, 30.9 & \textcolor{blue}{\underline{73.1}}, \underline{73.2} & \textcolor{blue}{28.7}, 29.1 & \textcolor{blue}{64.7}, 66.5 & \textcolor{blue}{71.5}, 72.3 & \textcolor{blue}{72.0}, 72.6 & \textcolor{blue}{\textbf{77.8}}, \textbf{77.8} \\
        & & ARI & \textcolor{blue}{21.2}, 22.4 & \textcolor{blue}{65.8}, 65.9 & \textcolor{blue}{12.3}, 12.1 & \textcolor{blue}{59.0}, 60.5 & \textcolor{blue}{67.1}, 67.8 & \textcolor{blue}{\underline{67.9}}, \underline{68.4} & \textcolor{blue}{\textbf{73.1}}, \textbf{73.0} \\
        \cmidrule(lr){2-10}
        & \multirow{3}{*}{\raisebox{-1.0ex}{FMNIST}} & ACC & \textcolor{blue}{35.2}, 35.6 & \textcolor{blue}{59.4}, 59.1 & \textcolor{blue}{28.6}, 28.8 & \textcolor{blue}{61.0}, 61.3 & \textcolor{blue}{\underline{66.6}}, \underline{66.7} & \textcolor{blue}{\textbf{68.0}}, \textbf{67.6} & \textcolor{blue}{65.3}, 64.8 \\
        & & NMI & \textcolor{blue}{31.4}, 32.8 & \textcolor{blue}{59.2}, 59.0 & \textcolor{blue}{28.3}, 30.6 & \textcolor{blue}{56.6}, 57.8 & \textcolor{blue}{60.7}, 61.2 & \textcolor{blue}{\underline{61.7}}, \underline{62.0} & \textcolor{blue}{\textbf{66.0}}, \textbf{65.6} \\
        & & ARI & \textcolor{blue}{20.4}, 21.3 & \textcolor{blue}{45.0}, 44.3 & \textcolor{blue}{14.5}, 15.4 & \textcolor{blue}{46.8}, 47.4 & \textcolor{blue}{49.9}, 50.0 & \textcolor{blue}{\underline{51.9}}, \underline{51.8} & \textcolor{blue}{\textbf{53.4}}, \textbf{52.8} \\
        \cmidrule(lr){2-10}
        & \multirow{3}{*}{\raisebox{-1.0ex}{ImageNet10}} & ACC & \textcolor{blue}{31.8}, 34.2 & \textcolor{blue}{89.9}, 90.5 & \textcolor{blue}{39.4}, 38.8 & \textcolor{blue}{80.7}, 85.0 & \textcolor{blue}{87.4}, 89.2 & \textcolor{blue}{\underline{90.1}}, \underline{91.3} & \textcolor{blue}{\textbf{92.4}}, \textbf{92.6} \\
        & & NMI & \textcolor{blue}{22.0}, 26.8 & \textcolor{blue}{\textbf{87.8}}, \textbf{88.7} & \textcolor{blue}{31.0}, 34.8 & \textcolor{blue}{65.6}, 73.8 & \textcolor{blue}{76.1}, 80.1 & \textcolor{blue}{80.1}, 83.0 & \textcolor{blue}{\underline{87.7}}, \underline{88.1} \\
        & & ARI & \textcolor{blue}{15.8}, 18.7 & \textcolor{blue}{\underline{83.9}}, \underline{84.9} & \textcolor{blue}{14.3}, 15.6 & \textcolor{blue}{63.4}, 71.0 & \textcolor{blue}{75.0}, 78.3 & \textcolor{blue}{79.8}, 82.0 & \textcolor{blue}{\textbf{85.4}}, \textbf{85.7} \\
        \cmidrule(lr){2-10}
        & \multirow{3}{*}{\raisebox{-1.0ex}{MNIST}} & ACC & \textcolor{blue}{32.2}, 33.3 & \textcolor{blue}{67.1}, 67.8 & \textcolor{blue}{31.8}, 31.0 & \textcolor{blue}{78.6}, 80.0 & \textcolor{blue}{80.5}, 81.8 & \textcolor{blue}{\textbf{85.1}}, \textbf{86.1} & \textcolor{blue}{\underline{84.7}}, \underline{85.1} \\
        & & NMI & \textcolor{blue}{26.5}, 29.3 & \textcolor{blue}{57.5}, 58.9 & \textcolor{blue}{28.8}, 30.2 & \textcolor{blue}{70.8}, 74.0 & \textcolor{blue}{72.1}, 74.7 & \textcolor{blue}{\underline{75.0}}, \underline{77.2} & \textcolor{blue}{\textbf{79.9}}, \textbf{80.7} \\
        & & ARI & \textcolor{blue}{18.4}, 20.1 & \textcolor{blue}{50.9}, 52.1 & \textcolor{blue}{13.2}, 13.4 & \textcolor{blue}{66.3}, 69.1 & \textcolor{blue}{68.7}, 71.2 & \textcolor{blue}{\underline{73.4}}, \underline{75.5} & \textcolor{blue}{\textbf{77.1}}, \textbf{77.5} \\
        \cmidrule(lr){2-10}
        & \multirow{3}{*}{\raisebox{-1.0ex}{REUTERS}} & ACC & \textcolor{blue}{55.3}, 59.2 & \textcolor{blue}{67.5}, 68.0 & \textcolor{blue}{57.3}, 59.1 & \textcolor{blue}{62.1}, 69.9 & \textcolor{blue}{74.5}, 79.0 & \textcolor{blue}{\underline{81.0}}, \underline{83.9} & \textcolor{blue}{\textbf{88.3}}, \textbf{89.6} \\
        & & NMI & \textcolor{blue}{14.7}, 22.8 & \textcolor{blue}{45.4}, 46.5 & \textcolor{blue}{18.3}, 25.7 & \textcolor{blue}{25.6}, 37.2 & \textcolor{blue}{44.6}, 52.8 & \textcolor{blue}{\underline{51.5}}, \underline{57.9} & \textcolor{blue}{\textbf{65.4}}, \textbf{68.7} \\
        & & ARI & \textcolor{blue}{19.2}, 26.2 & \textcolor{blue}{47.0}, 47.9 & \textcolor{blue}{20.4}, 23.1 & \textcolor{blue}{28.0}, 40.7 & \textcolor{blue}{50.2}, 58.6 & \textcolor{blue}{\underline{57.9}}, \underline{64.2} & \textcolor{blue}{\textbf{72.9}}, \textbf{76.0} \\
        \cmidrule(lr){2-10}
        & \multirow{3}{*}{\raisebox{-1.0ex}{STL10}} & ACC & \textcolor{blue}{29.2}, 30.7 & \textcolor{blue}{71.0}, 71.1 & \textcolor{blue}{31.4}, 31.6 & \textcolor{blue}{67.8}, 71.4 & \textcolor{blue}{81.1}, 83.0 & \textcolor{blue}{\underline{83.1}}, \underline{84.3} & \textcolor{blue}{\textbf{87.9}}, \textbf{87.7} \\
        & & NMI & \textcolor{blue}{21.5}, 26.0 & \textcolor{blue}{\underline{70.2}}, 69.9 & \textcolor{blue}{24.9}, 31.7 & \textcolor{blue}{53.2}, 59.7 & \textcolor{blue}{66.1}, 70.2 & \textcolor{blue}{69.4}, \underline{72.0} & \textcolor{blue}{\textbf{78.3}}, \textbf{78.5} \\
        & & ARI & \textcolor{blue}{15.2}, 17.7 & \textcolor{blue}{59.5}, 59.1 & \textcolor{blue}{13.6}, 16.9 & \textcolor{blue}{47.8}, 53.3 & \textcolor{blue}{63.7}, 67.1 & \textcolor{blue}{\underline{67.4}}, \underline{69.4} & \textcolor{blue}{\textbf{75.9}}, \textbf{75.8} \\
        \cmidrule(lr){2-10}
        & \multirow{3}{*}{\raisebox{-1.0ex}{RCV1-10}} & ACC & \textcolor{blue}{44.3}, 44.6 & \textcolor{blue}{48.3}, 49.0 & \textcolor{blue}{32.3}, 32.0 & \textcolor{blue}{\textbf{55.7}}, \textbf{55.9} & \textcolor{blue}{43.4}, 43.5 & \textcolor{blue}{47.8}, 47.9 & \textcolor{blue}{\underline{53.4}}, \underline{53.3} \\
        & & NMI & \textcolor{blue}{14.3}, 14.7 & \textcolor{blue}{31.4}, 32.1 & \textcolor{blue}{14.1}, 13.5 & \textcolor{blue}{\underline{51.9}}, \underline{52.7} & \textcolor{blue}{46.6}, 47.4 & \textcolor{blue}{51.6}, 52.0 & \textcolor{blue}{\textbf{58.9}}, \textbf{59.0} \\
        & & ARI & \textcolor{blue}{17.8}, 18.0 & \textcolor{blue}{30.2}, 30.9 & \textcolor{blue}{10.8}, 10.3 & \textcolor{blue}{\underline{43.6}}, \underline{44.2} & \textcolor{blue}{34.4}, 34.7 & \textcolor{blue}{38.2}, 38.4 & \textcolor{blue}{\textbf{46.2}}, \textbf{46.1} \\
        \bottomrule
    \end{tabular}
    \end{scriptsize}
    \end{center}
    \vskip -0.22in
\end{table}

\paragraph{Main comparison.}
\cref{table:exp1_9k_main} compares our ECI-PP against baselines under the {\tt lv0.01} single-expert and {\tt multi3} multi-expert settings, both with corruption probability $0.3$ and 9k constraints; 
additional constraint-budget results appear in Appendix~\ref{app:full_budget_results}.
Overall, ECI-PP is strongest, ranking first in $41$ out of $48$ test entries ($2$ expert regimes$\times$$8$ datasets$\times$$3$ metrics) and within the top two in $46/48$,
while one of the ProbPair-family methods is best in $45/48$, supporting the ProbPair formulation under UPCC.
Under single-expert supervision, ECI-PP outperforms the strong non-ProbPair baseline, SpherePair, in most entries ($4$--$10\%$ absolute NMI margins); 
further comparisons without pretraining make this pattern clearer (Appendix~\ref{app:no_pretraining_results}).
Under multi-expert supervision, the NMI margin over SpherePair widens to $6$--$32\%$.
Appendix~\ref{app:expert_aware_ensemble_results} further compares baselines with an expert-aware ensemble wrapper.
As an ablation signal, Weighted ProbPair shows a clearer advantage over ProbPair in the multi-expert setting, suggesting that information-based weighting is more useful when low-decisiveness constraints can indicate expert unfamiliarity.
A remaining caveat is the severely imbalanced RCV1-10: SpherePair retains the best ACC, reflecting the advantage of its angular geometry for imbalanced pair structures, while ECI-PP still gives the best NMI/ARI, indicating stronger grouping consistency.

\paragraph{Robustness across supervision conditions.}
We further stress-test all methods by varying one supervision factor in UPCC at a time.
Around the $9$k default supervision in \cref{table:exp1_9k_main}, \cref{fig:exp2_nmi} reports test NMI as we vary (A) expert quality, (B) corruption probability, and (C) multi-expert configuration; ACC/ARI results are deferred to Appendix~\ref{app:robustness_full_results}.
Across the three sweeps, the trends largely follow the expected supervision-quality ordering: stronger experts, lower corruption, and multi-expert settings with reduced expert blind spots (as detailed in Appendix~\ref{app:constraint_generation}) generally yield better performance, while the gap between ECI-PP and strong ProbPair-family baselines often narrows in easier regimes.
ECI-PP is broadly stable, especially as corruption increases, where most baselines degrade sharply.

Complementary tests further examine whether this robustness extends beyond the default simulated supervision process.
On CIFAR-10H, ECI-PP remains among the strongest methods across all three clustering metrics under human-derived probabilistic supervision (\cref{tab:cifar10h}).
Across all eight main datasets, ECI-PP also retains the strongest overall performance under generator misspecification with alternative relation constructions and corruption mechanisms (Appendix~\ref{app:misspecification_results}), indicating that its advantage is not tied to a particular supervision generator.

\begin{figure}[t]
    \begin{center}
    \centerline{\includegraphics[width=\textwidth]{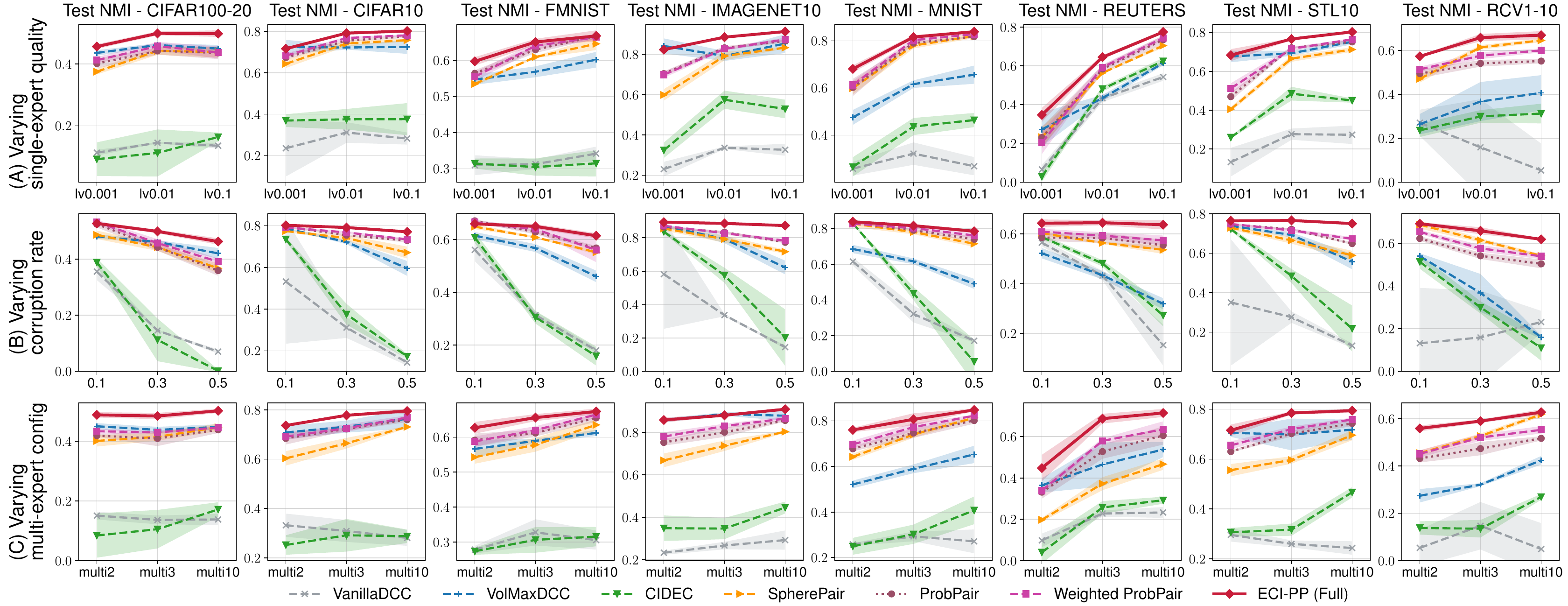}}
    \vskip -0.05in 
    \caption{
    Test NMI performance (mean$\pm$std over 5 runs) of external baselines and our progressive models (ProbPair, Weighted ProbPair, and ECI-PP) across datasets under varying (A) expert quality, (B) corruption probability, and (C) multi-expert configuration, with $9$k constraints.
    }
    \label{fig:exp2_nmi}
    \end{center}
    \vskip -0.2in
\end{figure}

\begin{table}[t]
    \caption{CIFAR-10H performance under human-derived supervision (\%; mean$\pm$std over 5 runs). $^{\dagger}$ marks pretrain-capable methods without pretraining; best/second-best are bold/underlined.}
    \vskip -0.05in
    \label{tab:cifar10h}
    \begin{center}
    \begin{scriptsize}
    \renewcommand{\arraystretch}{1.0}
    \setlength{\tabcolsep}{0pt}
    \setlength{\aboverulesep}{1pt}
    \setlength{\belowrulesep}{1pt}
    \setlength{\extrarowheight}{0pt}
    \begin{tabular}{@{}l*{7}{>{\centering\arraybackslash}m{53.5pt}}@{}}
        \toprule
        \addlinespace[1pt]
        & \multicolumn{4}{c}{\textbf{External baselines}} & \multicolumn{3}{c}{\textbf{Our progressive models}} \\
        \cmidrule(lr){2-5}\cmidrule(lr){6-8}
        & {VanillaDCC} & {VolMaxDCC} & {CIDEC$^{\dagger}$} & {SpherePair$^{\dagger}$} & {ProbPair$^{\dagger}$} & {\resizebox{50.5pt}{!}{Weighted ProbPair$^{\dagger}$}} & {\textbf{ECI-PP} (Full)$^{\dagger}$} \\
        \midrule
        ACC & 81.2$\pm$4.0 & 81.9$\pm$7.0 & 73.7$\pm$6.5 & 85.5$\pm$0.8 & \underline{86.3$\pm$0.5} & \textbf{87.1$\pm$0.6} & \textbf{87.1$\pm$0.7} \\
        NMI & 76.5$\pm$2.0 & 75.6$\pm$4.1 & 72.3$\pm$3.4 & 76.3$\pm$0.8 & \textbf{78.4$\pm$0.4} & \textbf{78.4$\pm$0.5} & \underline{78.1$\pm$0.5} \\
        ARI & 71.0$\pm$3.1 & 69.7$\pm$7.3 & 63.6$\pm$5.2 & 73.2$\pm$1.0 & \underline{74.8$\pm$0.4} & \textbf{75.3$\pm$0.8} & \textbf{75.3$\pm$0.9} \\
        \bottomrule
    \end{tabular}
    \end{scriptsize}
    \end{center}
    \vskip -0.2in
\end{table}


\paragraph{Held-out diagnostics.}
We examine whether ECI-PP's internal signals evolve as intended on sample-disjoint held-out constraints.
\cref{fig:diag_mnist} shows the MNIST case under {\tt multi3} with corruption rate $0.3$, using benchmark labels only to form an external hard co-membership reference for diagnostics, rather than the canonical aleatoric relation $R^\star$.
Panels (A--C) report Brier scores of the \textit{estimated}, \textit{corrected}, and \textit{integrated} relations, and panel (D) reports the \textit{residual discrepancy} between corrected relations and estimator beliefs on uncorrupted constraints; lower is better.
Panels (E,F) compare reliability-aware screening with the injected corruption indicator using AUC/AP, where higher values indicate better corruption separation.
(G,H) repeat the same screening diagnostic with oracle-generated clean supervision before corruption, removing the expert epistemic uncertainty present in (E,F).
The curves improve rapidly and then stabilize:
relation Brier scores decrease, the post-correction residual discrepancy drops, and the screening signal aligns with corruption above chance, while the oracle-setting curves in (G,H) more closely track corruption as expected.
Overall, these held-out trends are consistent with the intended behavior of the ECI framework, from internal relation estimates to reliability-aware screening; 
full cross-dataset diagnostics are in Appendix~\ref{app:heldout_diagnostics_full}.

\begin{figure}[t]
    \begin{center}
    \centerline{\includegraphics[width=0.72\textwidth]{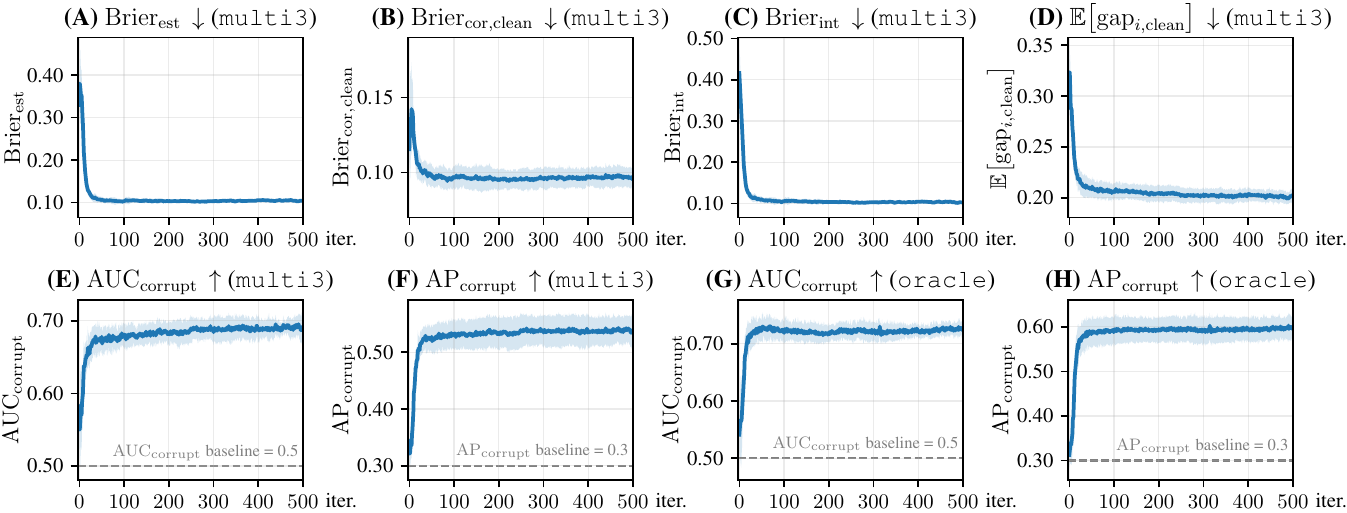}}
    \vskip -0.05in
    \caption{
    Held-out diagnostic evolution over $500$ training iterations on MNIST under {\tt multi3} with corruption rate $0.3$ (mean$\pm$std over 5 runs).
    (A--C) Brier scores for estimated, corrected, and integrated relations; (D) residual discrepancy between corrected relations and estimator beliefs on uncorrupted constraints; (E,F) AUC/AP measuring the alignment between reliability-aware screening and injected corruption; (G,H) the same screening diagnostic with oracle-generated pairwise supervision before corruption.
    Lower is better for (A--D), while higher is better for (E--H).
    }
    \label{fig:diag_mnist}
    \end{center}
    \vskip -0.15in
\end{figure}

\begin{table}[htbp]
    \caption{Ablation under the pretrained \texttt{multi3} setting ($p_{\mathrm{corr}}=0.3$, 9k constraints). Results are mean$\pm$std over five runs; best and second-best results are bold and underlined, respectively.}
    \vskip -0.05in
    \label{tab:ablation}
    \begin{center}
    \begin{scriptsize}
    \renewcommand{\arraystretch}{1.0}
    \setlength{\tabcolsep}{5.0pt}
    \setlength{\aboverulesep}{0.32pt}
    \setlength{\belowrulesep}{0.32pt}
    \setlength{\extrarowheight}{0pt}
    \begin{tabular}{llcccccc}
        \toprule
        Dataset & Metric & \makecell{Weighted\\ProbPair} & \makecell{w/o Cross-\\fitting} & \makecell{w/o\\Corrector} & \makecell{w/o\\Screening} & \makecell{1-pass\\ECI} & \makecell{\textbf{ECI-PP}\\(Full)} \\
        \midrule
        \multirow{3}{*}{\raisebox{-1ex}{CIFAR100-20}} & ACC & 46.4$\pm$1.9 & 45.8$\pm$2.5 & 48.2$\pm$1.5 & 45.8$\pm$1.9 & \underline{48.8$\pm$1.2} & \textbf{49.1$\pm$0.9} \\
          & NMI & 43.2$\pm$0.6 & 42.1$\pm$1.1 & 46.3$\pm$0.6 & 45.7$\pm$0.7 & \underline{47.7$\pm$0.6} & \textbf{48.1$\pm$0.6} \\
          & ARI & 28.4$\pm$1.0 & 28.4$\pm$1.7 & 31.9$\pm$1.2 & 30.6$\pm$1.3 & \underline{33.4$\pm$0.9} & \textbf{33.5$\pm$0.6} \\
        \midrule
        \multirow{3}{*}{\raisebox{-1ex}{REUTERS}} & ACC & 83.7$\pm$4.0 & 81.6$\pm$2.2 & 85.7$\pm$0.8 & \underline{88.1$\pm$1.6} & 87.9$\pm$0.4 & \textbf{89.5$\pm$1.1} \\
          & NMI & 59.2$\pm$3.8 & 53.5$\pm$3.7 & 61.0$\pm$1.9 & \underline{68.0$\pm$2.0} & 65.5$\pm$0.6 & \textbf{69.7$\pm$2.3} \\
          & ARI & 65.0$\pm$5.7 & 60.4$\pm$3.9 & 67.7$\pm$1.8 & \underline{74.5$\pm$3.1} & 73.2$\pm$0.9 & \textbf{76.5$\pm$2.5} \\
        \bottomrule
    \end{tabular}
    \end{scriptsize}
    \end{center}
    \vskip -0.2in
\end{table}

\paragraph{Component ablation and sensitivity.}
We assess the contribution of the main ECI-PP components under the component-level protocol detailed in Appendix~\ref{app:protocol}.

Using CIFAR100-20 and Reuters as complementary image and text benchmarks with different class-balance profiles, \cref{tab:ablation} shows that full ECI-PP achieves the strongest performance across all three metrics on both datasets.
Removing cross-fitting produces the largest or near-largest degradation overall, particularly on Reuters, highlighting the importance of out-of-fold beliefs as a reliable reference for subsequent refinement.
Removing either the Corrector or reliability screening also consistently weakens performance, supporting their complementary roles in structured correction and reliability-aware integration.
The 1-pass ECI variant (a single ECI cycle without iterative feedback) remains competitive but consistently below full ECI-PP, while Weighted ProbPair is clearly weaker, indicating additional benefits from iterative feedback and the complete ECI pipeline.

We further examine sensitivity around the shared default configuration in Appendix~\ref{app:sensitivity_ablation_results}.
The results show broad stability: larger $K$ yields only mild gains relative to its cross-fitting cost, information-based warm-up is mildly helpful, small Correctors with moderate regularization are sufficient, and the screening and boundary parameters $(\gamma,n_0,\xi)$ are generally insensitive.
Reconstruction is more dataset-dependent, while ECI-PP remains robust without dataset-specific hyperparameter tuning.

\section{Conclusion}
\label{sec:conclusion}

This paper addressed a realistic probabilistic constrained clustering setting, where pairwise supervision entangles aleatoric, epistemic, and stochastic factors.
We formalized the constraint-observation process and its conditional identifiability, introduced ProbPair for angular learning from probabilistic relations, and built ECI-PP as a practical framework for such supervision.
While existing methods largely rely on idealized constraints, ECI-PP outperforms state-of-the-art alternatives and remains robust under fallible, heterogeneous, and corrupted supervision without dataset-specific tuning.

These results come with several qualifications.
The identifiability analysis is population-level and conditional; fully factor-specific finite-sample estimation would require stronger assumptions or additional observations beyond the current surrogate residual structure.
Accordingly, the Corrector and reliability signal are practical surrogates rather than estimators of individual latent factors; a more theory-aligned extension would model structured correction, residual epistemic variation, and corruption-aware screening separately.
Finally, richer Corrector pair descriptors and principled stopping criteria may further improve robustness and training adaptivity.


\begin{ack}

We are grateful to the anonymous reviewers for their comments, which improved the presentation of this paper.

\paragraph{Funding.}
S.J. Zhang was supported by the University of Manchester--China Scholarship Council Joint Scholarship.
We also acknowledge the assistance provided by Research IT and the use of the Computational Shared Facility at The University of Manchester.

\paragraph{Competing interests.}
The authors declare no competing interests.

\end{ack}


\bibliographystyle{unsrt}
\bibliography{reference}


\appendix

\section{Additional theoretical details}
\label{app:proofs_problem_formulation_v3}

\subsection{Derivation of the logistic-normal judgment density}
\label{app:derivation_jud_density_v4}

Starting from \Cref{eq:latent_judgment_v4}, define
\[
\mu_{e,ab}:=\logit_{\varepsilon}(R^{\star}_{ab})+m_e(\phi_{ab}).
\]
Conditionally on $(R^{\star}_{ab},e,\phi_{ab})$, we then have
\[
\logit\big(y^{\mathrm{jud}}_{e,ab}\big)=\mu_{e,ab}+u_{e,ab},
\qquad
u_{e,ab}\sim \mathcal N\!\big(0,\tau_e^2(\phi_{ab})\big),
\]
so that
\[
\logit\big(y^{\mathrm{jud}}_{e,ab}\big)
\sim
\mathcal N\!\big(\mu_{e,ab},\tau_e^2(\phi_{ab})\big).
\]
For a generic $y\in(0,1)$, write $v=\logit(y)$. By the change-of-variables formula,
\[
p_{\mathrm{jud}}(y\mid R^{\star}_{ab},e,\phi_{ab})
=
\varphi\!\Big(v;\,\mu_{e,ab},\,\tau_e^2(\phi_{ab})\Big)
\left|\frac{dv}{dy}\right|,
\qquad v=\logit(y),\quad y\in(0,1).
\]
Using $\frac{dv}{dy}=\frac{d\logit(y)}{dy}=\frac{1}{y(1-y)}$, we obtain
\[
p_{\mathrm{jud}}(y\mid R^{\star}_{ab},e,\phi_{ab})
=
\frac{1}{y(1-y)}\,
\varphi\!\Big(\logit(y);\,\logit_{\varepsilon}(R^{\star}_{ab})+m_e(\phi_{ab}),\,\tau_e^2(\phi_{ab})\Big),
\qquad y\in(0,1),
\]
which is exactly \Cref{eq:judgment_density_v4}.

\subsection{Proof of Lemma~\ref{lem:gauge_fix_v3}}
\label{app:proof_gauge_fix_v4}
\begin{proof}
Fix an expert $e$ with nonzero observation probability. By assumption,
\[
m_{\eta_e}(\phi_{ab})=m_{\eta'_e}(\phi_{ab})-k_e
\]
for all observable triples $(a,b,e)$. Taking expectation over the observable-pair distribution conditional on $e$ yields
\[
\mathbb{E}_{(a,b)\mid e}[m_{\eta_e}(\phi_{ab})]
=
\mathbb{E}_{(a,b)\mid e}[m_{\eta'_e}(\phi_{ab})]-k_e.
\]
Both expectations vanish by Assumption~\ref{ass:centering_v3}, hence $k_e=0$.\footnote{The centering is imposed conditionally on each expert because the remaining additive ambiguity after fixing the total location is expert-wise: global centering would only fix the average shift across experts, not each expert-specific shift. Thus, under the present assumptions, global centering would not be sufficient for identifiability unless one further assumes that different experts observe sufficiently overlapping pair regions.} Since $e$ was arbitrary over experts with nonzero observation probability, the conclusion holds for every such expert.
\end{proof}

\subsection{Proof of Lemma~\ref{lem:mix_injective_v3}}
\label{app:proof_mix_injective_v4}
\begin{proof}
By \Cref{eq:judgment_density_v4,eq:observation_law_v4}, the observation law takes the form \Cref{eq:mix_family_v4}.

Assume that
\[
p(\cdot\mid \mu,s,\pi)=p(\cdot\mid \mu',s',\pi')
\]
on $(0,1)$, where $p(y\mid \mu,s,\pi)$ is the density defined in \Cref{eq:mix_family_v4}. Since $\mathbf 1_{[0,1]}(y)=1$ for $y\in(0,1)$, equality of the two densities means that
\[
(1-\pi)\,\frac{1}{y(1-y)}\,\varphi\!\big(\logit(y);\mu,s^2\big)+\pi
=
(1-\pi')\,\frac{1}{y(1-y)}\,\varphi\!\big(\logit(y);\mu',{s'}^2\big)+\pi'
\]
for all $y\in(0,1)$.

Write $v=\logit(y)$, so that $y=\sigma(v)$ and $y(1-y)=\sigma(v)(1-\sigma(v))$. Multiplying both sides by $y(1-y)$ gives
\[
(1-\pi)\,\varphi(v;\mu,s^2)+\pi\,\sigma(v)(1-\sigma(v))
=
(1-\pi')\,\varphi(v;\mu',{s'}^2)+\pi'\,\sigma(v)(1-\sigma(v))
\]
for all $v\in\mathbb{R}$.

As $v\to+\infty$, the Gaussian terms decay faster than $e^{-v}$, whereas
\[
\sigma(v)(1-\sigma(v))=\frac{e^{-v}}{(1+e^{-v})^2}\sim e^{-v}.
\]
Therefore,
\[
\lim_{v\to+\infty} e^v\Big[(1-\pi)\,\varphi(v;\mu,s^2)+\pi\,\sigma(v)(1-\sigma(v))\Big]=\pi,
\]
and similarly the right-hand side tends to $\pi'$. Hence $\pi=\pi'$.

Subtracting the common contamination term then yields
\[
\varphi(v;\mu,s^2)=\varphi(v;\mu',{s'}^2)
\qquad \text{for all } v\in\mathbb{R}.
\]
Injectivity of the Gaussian family implies $\mu=\mu'$ and $s=s'$. Therefore $(\mu,s,\pi)=(\mu',s',\pi')$, proving the claim.
\end{proof}


\subsection{Boundary cases for structural separability}
\label{app:structural_failure}

Assumption~\ref{ass:separability_v3} concerns the decomposition of the identifiable location term
$\logit_{\varepsilon}(r_{\theta}(x_a,x_b))+m_{\eta_e}(\phi_{ab})$
into a canonical relation component and an expert-specific epistemic component.
Two boundary cases clarify this requirement.
(i) If the pair descriptor $\phi_{ab}$ is uninformative, for example constant over observable pairs for a given expert, then $m_{\eta_e}(\phi_{ab})$ reduces to an expert-specific constant and cannot capture pair-dependent epistemic effects; the intended epistemic component is then not represented by the chosen model class.
(ii) Conversely, if $\phi_{ab}$ is overly rich and the function class for $m_{\eta_e}(\phi_{ab})$ overlaps with that of $\logit_{\varepsilon}(r_{\theta}(x_a,x_b))$, then shared pair-dependent variation can be moved between the two terms while preserving their sum.
For example, a sufficiently small zero-mean shift lying in this overlap can be added to the canonical location term and subtracted from the expert-specific term without violating centering or changing the observation law.
The condition therefore rules out such interchangeability, ensuring that canonical aleatoric variation and structured expert-specific variation play distinct explanatory roles on the observable support.

\subsection{Proof of \Cref{thm:canonical_ident_v3}}
\label{app:proof_canonical_ident_v4}
\begin{proof}
Assume that the induced conditional laws coincide under $\Xi$ and $\Xi'$ for every observable triple $(a,b,e)$. For each such triple, define
\[
\mu_{\Xi}(a,b,e):=\logit_{\varepsilon}\!\big(r_{\theta}(x_a,x_b)\big)+m_{\eta_e}(\phi_{ab}),
\qquad
s_{\Xi}(a,b,e):=\tau_{\zeta_e}(\phi_{ab}),
\]
and define $\mu_{\Xi'}(a,b,e)$ and $s_{\Xi'}(a,b,e)$ analogously. By \Cref{eq:observation_law_v4,eq:mix_family_v4}, the conditional law of the observed response for $(a,b,e)$ is exactly
\[
p\big(y\mid \mu_{\Xi}(a,b,e),s_{\Xi}(a,b,e),\pi_e\big),
\]
and similarly under $\Xi'$. Since the two conditional laws coincide, Lemma~\ref{lem:mix_injective_v3} implies that for every observable triple $(a,b,e)$,
\[
\mu_{\Xi}(a,b,e)=\mu_{\Xi'}(a,b,e),
\qquad
s_{\Xi}(a,b,e)=s_{\Xi'}(a,b,e),
\qquad
\pi_e=\pi'_e.
\]
In particular,
\[
\logit_{\varepsilon}\!\big(r_{\theta}(x_a,x_b)\big)+m_{\eta_e}(\phi_{ab})
=
\logit_{\varepsilon}\!\big(r_{\theta'}(x_a,x_b)\big)+m_{\eta'_e}(\phi_{ab})
\]
for all observable triples $(a,b,e)$.

Applying Assumption~\ref{ass:separability_v3}, for each expert $e$ there exists a constant $k_e\in\mathbb{R}$ such that for all observable triples $(a,b,e)$,
\[
\logit_{\varepsilon}\!\big(r_{\theta}(x_a,x_b)\big)
=
\logit_{\varepsilon}\!\big(r_{\theta'}(x_a,x_b)\big)+k_e,
\qquad
m_{\eta_e}(\phi_{ab})=m_{\eta'_e}(\phi_{ab})-k_e.
\]
By Lemma~\ref{lem:gauge_fix_v3}, each $k_e$ on the observable expert support must be zero. Therefore,
\[
\logit_{\varepsilon}\!\big(r_{\theta}(x_a,x_b)\big)
=
\logit_{\varepsilon}\!\big(r_{\theta'}(x_a,x_b)\big)
\qquad \text{for all } (a,b)\in\mathcal P_{\mathrm{obs}}.
\]
Since $r_{\theta}(x_a,x_b),r_{\theta'}(x_a,x_b)\in[\varepsilon,1-\varepsilon]$ on $\mathcal P_{\mathrm{obs}}$, the clamped logit satisfies $\logit_{\varepsilon}(u)=\logit(u)$ on this interval and is therefore injective there. It follows that
\[
r_{\theta}(x_a,x_b)=r_{\theta'}(x_a,x_b)
\qquad \text{for all } (a,b)\in\mathcal P_{\mathrm{obs}},
\]
which proves the theorem.
\end{proof}

\section{Algorithm and complexity analysis}
\label{app:alg_complexity}

\subsection{Algorithm}
\label{app:alg}

We provide the detailed ECI-PP procedure introduced in \cref{sec:method}, summarized in Algorithm~\ref{alg:eci}.

\begin{algorithm}[htbp]
\caption{ECI-PP for learning under UPCC}
\label{alg:eci}
\begin{algorithmic}[1]
\Require training data $\mathcal{X}=\{x_j\}_{j=1}^{|\mathcal{X}|}$; constraints $\mathcal{C}=\{(a_i,b_i,e_i,y_i)\}_{i=1}^{|\mathcal{C}|}$; number of experts $E$; number of estimator folds $K$; total training epochs $N_{\mathrm{ep}}$; warm-up epochs $N_{\mathrm{wu}}$; hyperparameters $\lambda_{\mathrm{rec}}$ (reconstruction weight), $\lambda_{\mathrm{cor}}$ (Corrector regularization), $\xi$ (soft-clipping sharpness), $\gamma$ (screening sharpness), $n_0$ (Bayesian-confidence scale); clustering routine $\mathrm{Clustering}(\cdot)$

\State Partition $\mathcal{C}$ into $K$ disjoint folds $\{\mathcal{C}^{(k)}\}_{k=1}^{K}$ and define $\mathcal{I}_e:=\{i:e_i=e\}$ for each expert $e$
\State Initialize $\{\mathrm{Est}^{(k)}=(f_{\omega_k},g_{\omega'_k},m_k,T_k)\}_{k=1}^{K}$, $\{\mathrm{Cor}_e=q_{\nu_e}\}_{e=1}^{E}$, and $\mathrm{Int}=(f_{\omega},g_{\omega'},m,T)$

\Statex
\State \textbf{Stage I: Warm-up}
\State Compute decisiveness weights $\{\kappa_i\}_{i=1}^{|\mathcal{C}|}$ by \cref{eq:eci_infoweight}
\For{$n=1,\dots,N_{\mathrm{wu}}$}
    \For{$k=1,\dots,K$}
        \State Update $\mathrm{Est}^{(k)}$ on $\mathcal{C}\!\setminus\!\mathcal{C}^{(k)}$ by gradient descent on \cref{eq:probpair_total} with $y_i$ and weights $\kappa_i$
    \EndFor
\EndFor
\State Aggregate holdout predictions from $\{\mathrm{Est}^{(k)}\}_{k=1}^{K}$ to obtain $\{\hat{y}_i^{\mathrm{oof}}\}_{i=1}^{|\mathcal{C}|}$ via \cref{eq:probpair_readout}
\For{$e=1,\dots,E$}
    \State Update $\mathrm{Cor}_e$ on $\mathcal{I}_e$ by gradient descent on \cref{eq:eci_corrector_loss} until convergence
\EndFor
\State Compute $\{y_i^{\mathrm{cor}},\mathrm{gap}_i,w_i\}_{i=1}^{|\mathcal{C}|}$ by \cref{eq:eci_int_screening}
\State Set $n_i \gets n_0 w_i$ and compute Bayesian-confidence supervision $\{y_i^{\mathrm{BC}}\}_{i=1}^{|\mathcal{C}|}$ by \cref{eq:eci_int_bc}
\For{$n=1,\dots,N_{\mathrm{wu}}$}
    \State Update $\mathrm{Int}$ by gradient descent on \cref{eq:probpair_total} with $y_i^{\mathrm{BC}}$ and weights $w_i$
\EndFor

\Statex
\State \textbf{Stage II: Iterative refinement}
\For{$t=N_{\mathrm{wu}}+1,\dots,N_{\mathrm{ep}}$}
    \State Read out Integrator-side relations $\{y_i^{\mathrm{int}}\}_{i=1}^{|\mathcal{C}|}$ from the Integrator embeddings via \cref{eq:probpair_readout}
    \For{$k=1,\dots,K$}
        \State Update $\mathrm{Est}^{(k)}$ on $\mathcal{C}\setminus\mathcal{C}^{(k)}$ by gradient descent on \cref{eq:probpair_total} with $y_i^{\mathrm{int}}$ and weights $w_i$
    \EndFor
    \State Aggregate updated holdout predictions to refresh $\{\hat{y}_i^{\mathrm{oof}}\}_{i=1}^{|\mathcal{C}|}$
    \For{$e=1,\dots,E$}
        \State Update $\mathrm{Cor}_e$ on $\mathcal{I}_e$ by gradient descent on \cref{eq:eci_corrector_loss} until convergence
    \EndFor
    \State Refresh $\{y_i^{\mathrm{cor}},\mathrm{gap}_i,w_i\}_{i=1}^{|\mathcal{C}|}$ by \cref{eq:eci_int_screening}
    \State Set $n_i \gets n_0 w_i$ and refresh Bayesian-confidence supervision $\{y_i^{\mathrm{BC}}\}_{i=1}^{|\mathcal{C}|}$ by \cref{eq:eci_int_bc}
    \State Update $\mathrm{Int}$ by gradient descent on \cref{eq:probpair_total} with $y_i^{\mathrm{BC}}$ and weights $w_i$
\EndFor

\Statex
\State \textbf{Stage III: Clustering and unseen-instance assignment}
\State Obtain the final normalized Integrator embeddings $\mathcal{Z}^{\mathrm{int}}\gets\{\mathrm{Norm}(f_{\omega}(x_j))\}_{j=1}^{|\mathcal{X}|}$
\State Apply $\mathrm{Clustering}(\mathcal{Z}^{\mathrm{int}})$ to obtain the final partition $\widehat{S}$
\State Compute cluster centroids $\{\mu_c\}$ from $\mathcal{Z}^{\mathrm{int}}$ and $\widehat{S}$
\For{each unseen instance $\tilde{x}$}
    \State Compute $\tilde{z}^{\mathrm{int}}\gets \mathrm{Norm}(f_{\omega}(\tilde{x}))$
    \State Assign $\tilde{x}$ to the nearest centroid in $\{\mu_c\}$
\EndFor

\end{algorithmic}
\end{algorithm}

\subsection{Computational complexity analysis}
\label{app:complexity}

For an Estimator/Integrator backbone $\big(f,g,m,T\big)$, let $c_f$ and $c_g$ denote the per-instance training costs of the encoder $f$ and decoder $g$, respectively, and let $c_{m,T}=O(1)$ denote the cost of updating the scalar readout parameters $(m,T)$.
We write $c_{\mathrm{pair}}=O(c_f+c_{m,T})$ for the amortized per-constraint cost of the ProbPair-style pairwise term, and $c_{\mathrm{rec}}=O(c_f+c_g)$ for the per-instance reconstruction cost.
Thus, one backbone epoch over constraints $\mathcal{C}$ and instances $\mathcal{X}$ costs
$O\!\left(|\mathcal{C}|c_{\mathrm{pair}}+|\mathcal{X}|c_{\mathrm{rec}}\right)$.
In particular, the pairwise term scales with the number of constraints, while the reconstruction term scales linearly with $|\mathcal{X}|$ as the usual autoencoder regularization cost.
Notably, each Corrector $\mathrm{Cor}_e=q_{\nu_e}$ is a small expert-specific MLP, rather than an encoder--decoder backbone, and its capacity is intentionally kept more limited than that of the Estimator/Integrator networks so that it learns only shared structured expert-specific patterns rather than arbitrary fitting.
We write $c_{\mathrm{cor}}$ for its amortized per-constraint update cost.

Under this notation, with $N_{\mathrm{wu}}$ warm-up epochs and $N_{\mathrm{ep}}$ total epochs, the warm-up stage costs
$
O\!\left(
N_{\mathrm{wu}}(K+1)
\big(|\mathcal{C}|c_{\mathrm{pair}}+|\mathcal{X}|c_{\mathrm{rec}}\big)
+
J_{\mathrm{cor}}|\mathcal{C}|c_{\mathrm{cor}}
\right),
$
where $J_{\mathrm{cor}}$ is the average number of optimization passes used when fitting the Correctors.
Each of the $T_{\mathrm{ref}}:=N_{\mathrm{ep}}-N_{\mathrm{wu}}$ subsequent refinement rounds costs
$
O\!\left(
(K+1)
\big(|\mathcal{C}|c_{\mathrm{pair}}+|\mathcal{X}|c_{\mathrm{rec}}\big)
+
J_{\mathrm{cor}}|\mathcal{C}|c_{\mathrm{cor}}
\right).
$
The total training complexity is
\[
O\!\left(
\big(N_{\mathrm{wu}}+T_{\mathrm{ref}}\big)(K+1)
\big(|\mathcal{C}|c_{\mathrm{pair}}+|\mathcal{X}|c_{\mathrm{rec}}\big)
+
\big(1+T_{\mathrm{ref}}\big)J_{\mathrm{cor}}|\mathcal{C}|c_{\mathrm{cor}}
\right).
\]
The dominant overhead comes from the $K$-fold cross-fitting of the Estimators, namely the repeated backbone updates scaled by $(K+1)\big(|\mathcal{C}|c_{\mathrm{pair}}+|\mathcal{X}|c_{\mathrm{rec}}\big)$.
This overhead is justified by the need for out-of-fold beliefs: the Estimators are introduced not as a redundant ensemble, but to provide model-side references decoupled from the specific constraints they are later used to assess, thereby reducing self-confirmation and overfitting to noisy supervision.
This use of held-out predictions is analogous in spirit to debiased/double machine learning \cite{DML}: both reduce in-sample contamination, although ECI-PP uses them to obtain cleaner belief proxies under noisy supervision rather than to pursue semiparametric efficiency.
By contrast, the expert-specific correction stage does \emph{not} scale as $O(E|\mathcal{C}|)$, because each constraint is routed only to its corresponding expert module and $\sum_{e=1}^{E}|\mathcal{I}_e|=|\mathcal{C}|$. 
Thus, for a fixed total constraint budget, increasing $E$ does not create an $O(E|\mathcal{C}|)$ bottleneck, and the Corrector workload remains linear in $|\mathcal{C}|$.
A separate statistical limitation arises when a large expert pool leaves some $|\mathcal{I}_e|$ too small for stable expert-specific correction; an amortized Corrector with shared parameters and expert embeddings or partial pooling is a natural extension for such sparse-expert regimes.

More broadly, the practical cost profile of ECI-PP should be interpreted together with its training protocol: it does not require an additional expert-aware wrapper in multi-expert settings (see Appendix~\ref{app:protocol}), and unlike methods such as VolMaxDCC, its performance is not tightly tied to an extensive hyperparameter-search stage.
Meanwhile, Appendix~\ref{app:no_pretraining_results} shows that ECI-PP is less dependent on unsupervised pretraining, which can be advantageous in resource-limited settings, and an overall training-time comparison is reported in Appendix~\ref{app:learning_efficiency}. 
A discussion of the choice of $K$ is provided in Appendix~\ref{app:sensitivity_ablation_results}.

\section{Details of experimental settings}
\label{app:experimental_settings}

We provide supplementary information for the experimental setup in \cref{sec:exp_settings}, including the benchmark datasets, compared methods, constraint-generation procedure, experimental protocol, and implementation details.

\subsection{Datasets}
\label{app:datasets}

We use eight benchmark datasets covering image and text clustering tasks, with different numbers of classes and class-balance properties. We additionally use CIFAR-10H for evaluation under human-derived probabilistic supervision. Their details are as follows.

\noindent\textbf{CIFAR100-20}\footnote{CIFAR100 webpage: \url{https://www.cs.toronto.edu/~kriz/cifar.html}} \cite{CIFAR}: 
A 20-class version of CIFAR100 obtained by using the 20 superclasses as clustering targets.
The original CIFAR100 dataset contains 60,000 real-world $32\times32$ color images from 100 fine-grained classes, grouped into 20 superclasses.
In our experiments, the 20 superclasses are treated as ground-truth clusters, with 3,000 images per superclass.

\noindent\textbf{CIFAR10}\footnote{CIFAR10 webpage: \url{https://www.cs.toronto.edu/~kriz/cifar.html}} \cite{CIFAR}: 
A natural-image dataset containing 60,000 $32\times32$ color images from 10 object categories, with 6,000 images per category.

\noindent\textbf{FMNIST}\footnote{FMNIST repository: \url{https://github.com/zalandoresearch/fashion-mnist}} \cite{FMNIST}: 
A fashion-product image dataset containing 70,000 grayscale $28\times28$ images from 10 categories.
The official split contains 60,000 training images and 10,000 test images.

\noindent\textbf{ImageNet10}\footnote{ImageNet webpage: \url{https://image-net.org/}} \cite{Imagenet10}: 
A 10-class subset of ImageNet \cite{imagenet}, containing 13,000 color images in total, with 1,300 images per class.

\noindent\textbf{MNIST}\footnote{MNIST webpage: \url{http://yann.lecun.com/exdb/mnist/}} \cite{MNIST}: 
A handwritten-digit dataset containing 70,000 grayscale $28\times28$ images from 10 digit classes.
The official split contains 60,000 training images and 10,000 test images.

\noindent\textbf{Reuters subset}\footnote{Preprocessed Reuters subset: \url{https://github.com/piiswrong/dec}} \cite{DEC}: 
A text benchmark derived from the RCV1 corpus\footnote{RCV1 webpage: \url{https://trec.nist.gov/data/reuters/reuters.html}}, using tf--idf features over the 2,000 most frequent words.
It contains four root categories, \emph{Corporate/Industrial} (CCAT), \emph{Economics} (ECAT), \emph{Government/Social} (GCAT), and \emph{Markets} (MCAT), treated as ground-truth clusters.
The four categories contain 4,840, 3,470, 2,673, and 1,017 samples, respectively, giving a mildly imbalanced benchmark.
The official split contains 10,000 training samples and 2,000 test samples.

\noindent\textbf{STL10}\footnote{STL10 webpage: \url{https://cs.stanford.edu/~acoates/stl10/}} \cite{STL10}: 
An image dataset containing 13,000 $96\times96$ color images from 10 object categories, with 1,300 images per category.

\noindent\textbf{RCV1-10}\footnote{Preprocessed RCV1-10 subset released with \cite{SpherePair}: \url{https://github.com/spherepaircc/SpherePairCC}} \cite{SpherePair}: 
An imbalanced 10-category subset of RCV1 constructed from single-label articles in categories C14, C18, C313, C42, E21, E311, GDEF, GODD, GWELF, and M13.
It contains 177,669 documents represented by tf--idf features over the 2,000 most frequent word stems.
The class distribution is highly skewed, ranging from 903 documents in GWELF to 53,127 documents in M13, making it a benchmark for evaluating clustering under severe class imbalance.

\noindent\textbf{CIFAR-10H}\footnote{CIFAR10H repository: \url{https://github.com/jcpeterson/cifar-10h}} \cite{CIFAR10H}:
A human-annotated extension of the official 10,000-image CIFAR10 test set, providing aggregate human class-distribution annotations for each image.
The images are evenly distributed across the 10 CIFAR10 classes.
We use a stratified $80\%/20\%$ train/test split of these 10,000 images, yielding 8,000 training and 2,000 test images.
CIFAR-10H is used only for evaluation under human-derived probabilistic supervision.

The eight benchmark datasets used in our main comparison follow common evaluation practice in deep clustering and constrained clustering: MNIST, FMNIST, Reuters subset, CIFAR10, ImageNet10, STL10, and CIFAR100-20 have been widely used in prior deep clustering/DCC studies \cite{DEC,IDEC,SDEC,CIDEC1,VolMaxDCC,SpherePair}, while RCV1-10 follows the imbalanced text benchmark introduced in \cite{SpherePair}.
For MNIST, FMNIST, and Reuters subset, we use the official train/test splits, yielding 60,000/10,000, 60,000/10,000, and 10,000/2,000 training/test samples, respectively.
For CIFAR100-20, CIFAR10, ImageNet10, STL10, and RCV1-10, we randomly split each dataset into $80\%$ training and $20\%$ testing, yielding 48,000/12,000, 48,000/12,000, 10,400/2,600, 10,400/2,600, and 142,135/35,534 training/test samples, respectively.
For CIFAR-10H, we use the stratified $80\%/20\%$ split of its 10,000 annotated CIFAR10 test images described above.
All probabilistic constraints for training are generated from the corresponding training split only.

\subsection{Compared methods}
\label{app:compared_methods}

We summarize the compared methods and their use under probabilistic pairwise supervision. 
Methods tied to hard binary constraints or discrete signed priors are not included, as they do not natively operate on the real-valued probabilistic supervision considered in UPCC.

\paragraph{VanillaDCC.}
VanillaDCC \cite{MCL} is an end-to-end DCC method based on the Meta Classification Likelihood (MCL) loss.
Given the soft cluster-assignment vector $\boldsymbol{q}_j=(q_{j1},\dots,q_{jC})$ for each instance $x_j$, the pairwise co-clustering probability of a constrained pair is
$P^{\mathrm{co}}_i=\boldsymbol{q}_{a_i}\boldsymbol{q}_{b_i}^{\top}$.
The MCL loss is
\[
\mathcal{L}_{\mathrm{MCL}}
=
-\frac{1}{|\mathcal{C}|}
\sum_{i=1}^{|\mathcal{C}|}
\left[
y_i\log P^{\mathrm{co}}_i
+
(1-y_i)\log(1-P^{\mathrm{co}}_i)
\right].
\]
Although originally used for binary constraints, this logistic/BCE-form objective can directly take real-valued targets $y_i\in[0,1]$, so VanillaDCC serves as a simple end-to-end DCC baseline under probabilistic supervision.

\paragraph{VolMaxDCC.}
VolMaxDCC \cite{VolMaxDCC} extends MCL by introducing an explicit noise-modeling mechanism through a confusion-adjusted co-clustering probability.
Concretely, it replaces $P^{\mathrm{co}}_i$ with
$P_i^{\prime\mathrm{co}}=\boldsymbol{q}_{a_i}B\boldsymbol{q}_{b_i}^{\top}$,
where $B$ is derived from a learnable confusion matrix, and adds a volume maximization term:
\[
\mathcal{L}_{\mathrm{VolMax}}
=
-\frac{1}{|\mathcal{C}|}
\sum_{i=1}^{|\mathcal{C}|}
\left[
y_i\log P_i^{\prime\mathrm{co}}
+
(1-y_i)\log(1-P_i^{\prime\mathrm{co}})
\right]
-
\lambda\log\det(\mathcal{Q}^{\top}\mathcal{Q}),
\]
where $\mathcal{Q}$ is the assignment matrix.
The first term allows the model to account for systematic annotation confusion, while the volume term encourages separated and distinguishable cluster assignments.
We therefore include VolMaxDCC as the compared baseline with explicit noise modeling.
At the same time, its design is still rooted in noisy binary supervision and assignment-level separation, whereas UPCC requires preserving a continuum of probabilistic pairwise relations rather than only correcting corrupted hard relations.

\paragraph{CIDEC.}
CIDEC \cite{CIDEC1} extends the DEC/IDEC deep embedding clustering framework \cite{DEC,IDEC} by incorporating pairwise constraints through the MCL loss.
It uses an autoencoder to obtain latent embeddings, initializes $C$ cluster anchors by K-means, and computes a soft assignment vector $\boldsymbol{q}_j=(q_{j1},\dots,q_{jC})$ for each instance.
Its unsupervised clustering term follows the DEC-style KL objective
\[
\mathcal{L}_{\mathrm{DEC}}
=
\sum_{j=1}^{|\mathcal{X}|}\sum_{c=1}^{C}
p_{jc}\log\frac{p_{jc}}{q_{jc}},
\]
where $\boldsymbol{p}_j$ is the sharpened target distribution derived from $\boldsymbol{q}_j$.
CIDEC combines this clustering loss with a reconstruction loss and the MCL pairwise constraint loss.
In our setting, its MCL component uses the real-valued $y_i$ directly through the same BCE-form target as VanillaDCC, making CIDEC a strong end-to-end baseline compatible with probabilistic supervision.

\paragraph{SpherePair.}
SpherePair \cite{SpherePair} is a deep constraint embedding method that learns normalized angular representations with an autoencoder.
For binary hard constraints, it defines an angular BCE-form loss
\[
\mathcal{L}_{\mathrm{ang}}
=
-\frac{1}{|\mathcal{C}|}
\sum_{i=1}^{|\mathcal{C}|}
\Big[
y_i\log s^{+}_{a_i b_i}
+
(1-y_i)\log(1-s^{-}_{\mathrm{hard},a_i b_i})
\Big],
\]
where $s^{+}_{a_i b_i}$ is the must-link similarity and $s^{-}_{\mathrm{hard},a_i b_i}$ is the cannot-link transformed similarity:
\[
s^{+}_{a_i b_i}
=
\tfrac{1}{2}\big(\cos(\theta_{\boldsymbol{z}_{a_i},\boldsymbol{z}_{b_i}})+1\big),
\qquad
s^{-}_{\mathrm{hard},a_i b_i}
=
\tfrac{1}{2}\big(\cos(\min\{\omega\theta_{\boldsymbol{z}_{a_i},\boldsymbol{z}_{b_i}},\pi\})+1\big).
\]
Here $\omega$ controls the negative zone of angular size $\pi/\omega$; the SpherePair theory fixes $\omega=2$, giving the $\pi/2$ negative-zone boundary and the corresponding conflict-free equidistant geometry for hard must-link/cannot-link constraints.
Since this original formulation is defined for binary relation types, applying SpherePair to UPCC requires a continuous counterpart for intermediate labels.
We therefore keep the same $\omega=2$ angular scale and use
\[
s^{-}_{\mathrm{soft},a_i b_i}
=
\tfrac{1}{2}\big(\cos(2\theta_{\boldsymbol{z}_{a_i},\boldsymbol{z}_{b_i}})+1\big)
=
\cos^2(\theta_{\boldsymbol{z}_{a_i},\boldsymbol{z}_{b_i}})
\]
in the BCE-form loss with each real-valued $y_i\in[0,1]$:
\[
\mathcal{L}_{\mathrm{ang}}^{\mathrm{soft}}
=
-\frac{1}{|\mathcal{C}|}
\sum_{i=1}^{|\mathcal{C}|}
\Big[
y_i\log s^{+}_{a_i b_i}
+
(1-y_i)\log(1-s^{-}_{\mathrm{soft},a_i b_i})
\Big].
\]
Under this soft extension, the attractive endpoint $y_i=1$ is minimized at $\theta_{\boldsymbol{z}_{a_i},\boldsymbol{z}_{b_i}}=0$, while the fully repulsive endpoint $y_i=0$ is minimized at $\theta_{\boldsymbol{z}_{a_i},\boldsymbol{z}_{b_i}}=\pi/2$, matching the orthogonal negative-zone boundary selected by the original $\omega=2$ SpherePair geometry.
For intermediate labels, the preferred pairwise angle varies continuously between these two endpoints; in the single-pair case, the optimum satisfies $\cos\theta_{\boldsymbol{z}_{a_i},\boldsymbol{z}_{b_i}}=y_i/(2-y_i)$.
SpherePair also uses the normalized-embedding reconstruction regularizer
$\mathcal{L}_{\mathrm{rec}}=\frac{1}{|\mathcal{X}|}\sum_j\|x_j-\hat{x}_j\|_2^2$ with
$\hat{x}_j=g(\mathrm{Norm}(f(x_j)))$, which is also adopted by our ProbPair, Weighted ProbPair, and ECI-PP.
Thus, the SpherePair baseline in our experiments is the natural real-valued angular extension of the original hard-constraint method, while its formal geometric guarantee remains tied to the binary setting.

\paragraph{ProbPair and Weighted ProbPair.}
ProbPair is the direct probabilistic pairwise baseline introduced in \cref{sec:probpair}.
It uses the probabilistic angular readout in \cref{eq:probpair_readout} and minimizes the reconstruction-regularized objective $\mathcal{L}$ in \cref{eq:probpair_total}, without our Estimator--Corrector--Integrator refinement.
Weighted ProbPair further applies the information-based constraint weight $\kappa_i$ in \cref{eq:eci_infoweight} to the pairwise term:
\[
\mathcal{L}_{\mathrm{WPP}}
=
-\frac{1}{|\mathcal{C}|}
\sum_{i=1}^{|\mathcal{C}|}
\kappa_i
\left[
y_i\log \hat{y}_{a_i b_i}
+
(1-y_i)\log(1-\hat{y}_{a_i b_i})
\right]
+
\lambda_{\mathrm{rec}}\mathcal{L}_{\mathrm{rec}}.
\]
Thus, ProbPair isolates the base probabilistic angular formulation, while Weighted ProbPair isolates the effect of decisiveness-based weighting.
Comparing Weighted ProbPair with ECI-PP further tests the contribution of the Estimator--Corrector--Integrator pipeline beyond static reweighting.

\subsection{Constraint generation}
\label{app:constraint_generation}

Here, we detail how probabilistic pairwise constraints are generated in our experiments.
For simulated-expert experiments, clean probabilistic judgments are generated from trained expert classifiers, after which the Bernoulli--Uniform corruption stage in \cref{sec:observation} is applied to obtain the recorded constraints.

\paragraph{Expert classifiers.}
For each dataset, simulated experts are trained on controlled labeled subsets of the training split and then used to generate predictive class distributions for queried pairs.
The limited labeled-subset construction makes these classifiers expert-dependent judgment mechanisms, whose soft predictive distributions reflect both systematic bias and residual uncertainty.
The subset selection rules for single- and multi-expert settings are described below, while the expert network architecture and optimization details are given in Appendix~\ref{app:implementation}.

\paragraph{Single-expert settings.}
In the single-expert setting, one expert classifier is trained per dataset.
We use three expert-quality levels, denoted by {\tt lv0.1}, {\tt lv0.01}, and {\tt lv0.001}.
For {\tt lv}$r$, the expert is trained using an $r$ fraction of the labeled training samples from each class.
The sampling is class-stratified: for every class, the number of selected samples is $\lfloor r n_c\rfloor$, where $n_c$ is the number of available training samples in class $c$, with at least one sample retained for each class.
Thus, smaller values of $r$ produce weaker experts and hence more epistemically distorted probabilistic judgments.

\paragraph{Multi-expert settings.}
In the multi-expert setting, we construct a group of complementary experts.
The three settings {\tt multi2}, {\tt multi3}, and {\tt multi10} use $2$, $3$, and $10$ experts, respectively.
Each expert has a set of familiar classes and a set of unfamiliar classes.
For all multi-expert settings, familiar classes use a high labeled-data fraction $h=0.1$, while unfamiliar classes use a low labeled-data fraction $l=0.0001$.
The class-level training fraction for expert $e$ is therefore
\[
\rho_{e,c}
=
\begin{cases}
h, & c\in \mathcal{F}_e,\\
l, & c\in \mathcal{U}_e,
\end{cases}
\]
where $\mathcal{F}_e$ and $\mathcal{U}_e$ denote the familiar and unfamiliar class sets of expert $e$.

The unfamiliar classes are assigned to make the experts complementary.
Let the dataset contain $C$ classes under the fixed label ordering used in expert generation.
When the number of experts $E$ satisfies $E\le C$, the class list is split into $E$ consecutive blocks that are as balanced as possible.
Expert $e$ takes the $e$-th block as its unfamiliar classes and treats all remaining classes as familiar.
Equivalently, if $C=qE+r$ with $0\le r<E$, then the first $r$ experts have $q+1$ unfamiliar classes and the remaining experts have $q$ unfamiliar classes.
For example, on a 10-class dataset, {\tt multi2} uses unfamiliar-class counts $(5,5)$, {\tt multi3} uses $(4,3,3)$, and {\tt multi10} gives each expert exactly one unfamiliar class.
For CIFAR100-20, the corresponding counts are $(10,10)$, $(7,7,6)$, and $(2,\ldots,2)$.
For Reuters subset with four classes, {\tt multi2} gives $(2,2)$ and {\tt multi3} gives $(2,1,1)$.
If $E>C$, the experts are divided into full groups of size $C$, where each expert in a group is unfamiliar with one distinct class; any remaining experts form a partial group whose unfamiliar classes are sampled without replacement from the class list.
This case occurs for Reuters subset under {\tt multi10}.
The resulting unfamiliar-class counts are summarized in \cref{tab:multi_expert_blind_counts}.

\begin{table}[t]
\centering
\caption{Unfamiliar-class counts per expert under the multi-expert settings. These counts describe the blind-spot distribution used to construct complementary experts.}
\label{tab:multi_expert_blind_counts}
\small
\setlength{\tabcolsep}{5pt}
\begin{tabular}{lccc}
\toprule
Dataset type & {\tt multi2} & {\tt multi3} & {\tt multi10} \\
\midrule
20-class datasets & $(10,10)$ & $(7,7,6)$ & $(2,2,2,2,2,2,2,2,2,2)$ \\
10-class datasets & $(5,5)$ & $(4,3,3)$ & $(1,1,1,1,1,1,1,1,1,1)$ \\
4-class datasets  & $(2,2)$ & $(2,1,1)$ & $(1,1,1,1,1,1,1,1,1,1)$ \\
\bottomrule
\end{tabular}
\end{table}

\paragraph{Pair sampling and expert assignment.}
All queried pairs are sampled randomly and uniformly from the training split.
In the single-expert setting, every sampled pair is assigned to the single available expert.
In the multi-expert setting, constraints are generated separately for each expert, and the total constraint budget is distributed evenly across experts unless otherwise specified.
For a queried pair $(x_a,x_b)$ assigned to expert $e$, the trained expert outputs predictive class distributions $\boldsymbol{p}^{e}_a$ and $\boldsymbol{p}^{e}_b$.
The clean expert judgment is computed as
\[
y^{\mathrm{jud}}_{e,ab}
=
\langle \boldsymbol{p}^{e}_a,\boldsymbol{p}^{e}_b\rangle
=
\sum_{c=1}^{C} p^{e}_{a,c}p^{e}_{b,c}
\in[0,1].
\]
Under independent categorical draws from the expert predictive distributions, this inner product equals the expert-implied probability that the two samples receive the same class.
Importantly, this classifier-space inner product is used only to construct experimental supervision and is not the ProbPair readout in \cref{eq:probpair_readout}, which operates on learned latent embeddings using cosine similarity followed by sigmoid calibration.
The resulting judgment is high for similar expert-assigned class distributions, low when the expert separates the samples, and intermediate under uncertainty or overlapping class mass.

\paragraph{Stochastic corruption.}
After the classifier-derived clean judgment is obtained, the Bernoulli--Uniform corruption stage in \cref{sec:observation} is applied to produce the recorded constraint value.
For each pair, an independent corruption indicator is drawn as
$c_i\sim\mathrm{Bernoulli}(\pi)$, where $\pi$ is the corruption probability used by the corresponding experiment.
If $c_i=0$, the recorded label is the clean expert judgment:
\[
y_i = y^{\mathrm{jud}}_{e_i,a_i b_i}.
\]
If $c_i=1$, the clean judgment is replaced by an unstructured random value:
\[
y_i \sim \mathrm{Uniform}(0,1).
\]
Thus, the final constraints combine three factors: data-dependent soft pairwise ambiguity expressed through expert predictive distributions, structured expert-dependent epistemic distortion from class-specific familiarity, and unstructured stochastic corruption from the uniform replacement channel.

\paragraph{Alternative relation and corruption mechanisms.}
To test whether performance depends on the specific relation construction and corruption mechanism used above, we additionally vary these two components.
For the clean relation, we replace the default dot-product judgment with normalized Jensen--Shannon similarity.
Given expert predictive distributions $\boldsymbol{p}^{e}_a$ and $\boldsymbol{p}^{e}_b$, let
\[
\bar{\boldsymbol{p}}^{e}_{ab}
=
\frac{1}{2}
\left(
\boldsymbol{p}^{e}_a+\boldsymbol{p}^{e}_b
\right),
\]
and define
\[
y^{\mathrm{JSD}}_{e,ab}
=
1-
\frac{
\frac{1}{2}D_{\mathrm{KL}}
\left(
\boldsymbol{p}^{e}_a
\|
\bar{\boldsymbol{p}}^{e}_{ab}
\right)
+
\frac{1}{2}D_{\mathrm{KL}}
\left(
\boldsymbol{p}^{e}_b
\|
\bar{\boldsymbol{p}}^{e}_{ab}
\right)
}{
\log 2
}.
\]
For the corruption mechanism, we replace uniform random replacement with complementary flipping: when a constraint is corrupted, its clean judgment $y$ is recorded as $1-y$.
Relative to the default Dot+Uniform generator, where Dot denotes the default inner-product relation and Uniform denotes the default uniform-replacement corruption mechanism, these alternatives yield three stress-test conditions: \textit{Dot+Flip} changes only the corruption mechanism, \textit{JSD+Uniform} changes only the clean relation construction, and \textit{JSD+Flip} changes both.

\paragraph{Human-derived supervision.}
For CIFAR-10H, we use the aggregate human class-distribution annotation of each image to construct probabilistic pairwise supervision.
Let $\boldsymbol{h}_j\in\Delta^{9}$ denote the normalized human class distribution for image $x_j$.
For a sampled pair $(x_a,x_b)$, we define
\[
y^{\mathrm{human}}_{ab}
=
\langle \boldsymbol{h}_a,\boldsymbol{h}_b\rangle
=
\sum_{c=1}^{10} h_{a,c}h_{b,c}.
\]
Under independent categorical draws from the two human response distributions, this quantity is the induced probability that the pair receives the same class label.
Unlike the simulated-expert constraints above, CIFAR-10H supervision is constructed only from the aggregate human judgments, without annotator identities or additional synthetic corruption.

\subsection{Experimental protocol}
\label{app:protocol}

Given the dataset details in Appendix~\ref{app:datasets} and the constraint-generation procedure in Appendix~\ref{app:constraint_generation}, we specify the experimental protocol for comparative evaluations, expert-aware ensemble references, held-out diagnostics, ablation, and sensitivity analysis.

\paragraph{Comparative evaluations.}
For clustering-performance comparisons, we report \textit{Accuracy} (ACC),\textit{ Normalized Mutual Information}
(NMI), and \textit{Adjusted Rand Index} (ARI) over five trials.
For embedding-based methods, including SpherePair, ProbPair, Weighted ProbPair, and ECI-PP, the final partition is obtained by applying K-means to the learned representations, while end-to-end DCC baselines use their native clustering outputs.

\textit{Main comparison.}
For the full main comparison, we use {\tt lv0.01} and {\tt multi3} with corruption probability $0.3$ and total constraint budgets $3$k/$6$k/$9$k; under {\tt multi3}, these correspond to $1$k/$2$k/$3$k constraints per expert.

\textit{Robustness comparisons.}
Robustness comparisons use a total budget of $9$k constraints and vary one supervision factor at a time.
For single-expert robustness, we use {\tt lv0.01} and corruption probability $0.3$ as the reference setting: expert quality is varied over {\tt lv0.1}/{\tt lv0.01}/{\tt lv0.001} while fixing corruption probability at $0.3$, and corruption probability is varied over $0.1/0.3/0.5$ while fixing the expert level at {\tt lv0.01}.
For multi-expert robustness, we compare {\tt multi2}, {\tt multi3}, and {\tt multi10} specified in Appendix~\ref{app:constraint_generation}; under the $9$k total budget, these correspond to $4{,}500$, $3{,}000$, and $900$ constraints per expert, respectively.
Moving from {\tt multi2} to {\tt multi10} increases both the number of experts and the average expert accuracy (\cref{tab:multi_expert_avg_acc}), because the complementary blind-spot rule assigns fewer unfamiliar classes to each expert (see \cref{tab:multi_expert_blind_counts}).

\textit{Human-derived evaluation.}
For CIFAR-10H, we uniformly sample $3{,}000$ constraints from the 8,000-image training split described in Appendix~\ref{app:datasets}, using the human-derived relations defined in Appendix~\ref{app:constraint_generation}, and evaluate on the held-out 2,000 images; pretrain-capable methods are trained from random initialization without autoencoder pretraining.
VolMaxDCC follows the validation-based hyperparameter search described in Appendix~\ref{app:implementation}.

\textit{Generator misspecification.}
We evaluate \textit{Dot+Flip}, \textit{JSD+Uniform}, and \textit{JSD+Flip} as defined in Appendix~\ref{app:constraint_generation}.
Experiments cover all eight main datasets with pretrained initialization for pretrain-capable methods, {\tt multi3}, a total budget of $9$k constraints ($3$k per expert), and corruption probability $0.3$, using five trials per condition.
All methods otherwise retain their default {\tt multi3} configuration; in particular, VolMaxDCC uses the dataset-specific {\tt multi} value in \cref{tab:volmax_lambda_selection}, held fixed across the three misspecification conditions.

\paragraph{Expert-aware ensemble references.}
Under multi-expert supervision, ECI-PP uses expert identities through its expert-specific Correctors, whereas expert-agnostic baselines treat the merged constraints as a single supervision pool.
As an additional reference, we also include expert-aware ensemble variants for baselines for multi-expert comparisons.
For a baseline trained under an $E$-expert setting, we instantiate $E$ members, each associated with one expert identity.
All members are trained on the full merged constraint set $\mathcal{C}=\{(a_i,b_i,e_i,y_i)\}_{i=1}^{|\mathcal{C}|}$, but member $e$ applies the multiplier
\[
\alpha_i^{(e)}
=
\begin{cases}
1, & e_i=e,\\
\alpha, & e_i\neq e,
\end{cases}
\qquad
\alpha\in[0,1],
\]
to the pairwise loss term of each constraint.
Thus, \(\alpha=0\) makes each member use only constraints from its associated expert, while \(\alpha=1\) recovers the expert-agnostic merged-constraint training for every member.
More generally, if the pairwise part of a baseline objective is decomposed into per-constraint terms \(\ell_i^{\mathrm{pair}}\), member \(e\) uses
\[
\mathcal{L}_{\mathcal{B},\mathrm{pair}}^{(e)}
=
\frac{1}{|\mathcal{C}|}
\sum_{i=1}^{|\mathcal{C}|}
\alpha_i^{(e)}\ell_i^{\mathrm{pair}},
\]
while all non-pairwise components of the baseline objective are kept unchanged.
After training, each member produces a hard partition \(\widehat{S}^{(e)}\).
We aggregate the members through the co-association matrix
\[
A^{\mathrm{ens}}_{ab}
=
\frac{1}{E}
\sum_{e=1}^{E}
\mathbb{I}\!\left[\widehat{S}^{(e)}(x_a)=\widehat{S}^{(e)}(x_b)\right],
\]
then form \(D^{\mathrm{ens}}_{ab}=1-A^{\mathrm{ens}}_{ab}\) and apply hierarchical clustering to obtain the final ensemble partition.
The non-target expert weight \(\alpha\) is selected separately for each dataset--baseline--multi-expert configuration using the reserved \(1{,}000\)-sample validation split: we search \(\alpha\in\{0,0.01,0.05,0.1,0.5,1\}\), run each candidate three times, and choose the value with the highest mean validation ACC before test evaluation.
These ensemble variants give expert-agnostic baselines access to expert identities, but they also alter the training and aggregation structure through multiple expert-weighted members; they are therefore interpreted as auxiliary expert-aware references rather than fully symmetric replacements for the native baselines.

\paragraph{Held-out diagnostics.}
The held-out diagnostics examine whether the internal quantities of ECI-PP show trends consistent with their intended roles: 
Estimator-side beliefs, corrected relations, and integrated relations are compared with an external pairwise reference, and the post-correction gap is tested as a residual-discrepancy signal.
To this end, we construct a diagnostic constraint set $\mathcal{C}^{\mathrm{diag}}$ where $|\mathcal{C}^{\mathrm{diag}}|=1000$ from the test split.
These constraints are sample-disjoint from the training constraints and are used only for diagnosis, not for training, early stopping, or hyperparameter selection.
For a held-out pair $i=(a_i,b_i)$ with ground-truth class labels $t_{a_i}$ and $t_{b_i}$, we define the external hard oracle
\[
o_i^\star=\mathbb{I}[t_{a_i}=t_{b_i}],
\]
which serves only as a diagnostic reference and is \textbf{not identified with the canonical aleatoric relation $R_{a_ib_i}^\star$ in \cref{sec:problem_formulation}}.
On $\mathcal{C}^{\mathrm{diag}}$, which is not assigned to any cross-fitting fold, the Estimator diagnostic belief is computed by averaging the predictions of all $K$ fold-specific Estimators; we still denote this held-out Estimator belief by $\hat{y}^{\mathrm{oof}}_i$ for notational consistency.
We then evaluate $\hat{y}^{\mathrm{oof}}_i$, the corrected relation $y^{\mathrm{cor}}_i$, and the Integrator belief $y^{\mathrm{int}}_i$ against this external oracle using Brier scores.
For any held-out score $v_i\in[0,1]$ and index set $\mathcal{I}$, we use
\[
\mathrm{Brier}(v;\mathcal{I})
=
\frac{1}{|\mathcal{I}|}
\sum_{i\in\mathcal{I}}(v_i-o_i^\star)^2.
\]
The Estimator and Integrator diagnostics are computed on all diagnostic constraints, whereas Corrector-oracle alignment is evaluated on the clean subset
$\mathcal{I}_{\mathrm{clean}}=\{i\in\mathcal{C}^{\mathrm{diag}}:c_i=0\}$.
We also compute the clean post-correction discrepancy
\[
\mathrm{gap}_{i,\mathrm{clean}}
=
\left|
\operatorname{softclip}_{\xi}(\hat{y}_i^{\mathrm{oof}})
-
y_i^{\mathrm{cor}}
\right|,
\qquad i\in\mathcal{I}_{\mathrm{clean}},
\]
to measure the residual disagreement between the Estimator belief and the corrected relation on uncorrupted held-out judgments.
For corruption-screening diagnostics, the corresponding gap score is computed on all diagnostic constraints and used to rank corrupted records $(c_i=1)$ against clean records $(c_i=0)$, from which we report $\mathrm{AUC}_{\mathrm{corrupt}}$ and $\mathrm{AP}_{\mathrm{corrupt}}$.
Because this gap captures unexplained discrepancy after correction, it may reflect both stochastic corruption and residual epistemic uncertainty, rather than stochastic corruption alone.
We therefore also include an oracle-supervision diagnostic, where the clean relation is given by $o_i^\star$ before applying the same corruption channel; this provides a reference case in which expert-dependent epistemic distortion is removed.

\paragraph{Ablation.}
We conduct a direct component-level ablation on CIFAR100-20 and Reuters, complementary image and text benchmarks with different class-balance profiles, under the default pretrained {\tt multi3} setting with a total budget of $9$k constraints and corruption probability $0.3$, using five trials.
We compare full ECI-PP with five variants:
\begin{itemize}
    \item \textit{w/o Cross-fitting}: replaces the out-of-fold Estimator beliefs with in-sample beliefs while leaving the subsequent correction, screening, and integration stages unchanged.
    \item \textit{w/o Corrector}: disables the Corrector by setting $\hat{\Delta}_i=0$, so that $y_i^{\mathrm{cor}}=\operatorname{softclip}_{\xi}(y_i)$, while retaining the remaining screening and integration stages.
    \item \textit{w/o Screening}: sets $w_i=1$ while retaining correction and integration.
    \item \textit{1-pass ECI}: executes one complete Estimator--Corrector--Integrator cycle initialized from $(y_i,\kappa_i)$ but disables feedback of $(y_i^{\mathrm{int}},w_i)$ to the Estimators, while matching the total optimization budget of full ECI-PP.
    \item \textit{Weighted ProbPair}: serves as the reference with the ECI pipeline removed.
\end{itemize}

\paragraph{Sensitivity analysis.}
Sensitivity analysis is conducted for ECI-PP under the default single-expert setting, namely {\tt lv0.01} with corruption probability $0.3$.
Each study is repeated over five trials and varies one factor at a time while keeping the remaining hyperparameters and the constraint-generation protocol fixed.
Here, we examine ECI-PP-specific sensitivity factors and implementation choices: the number of estimator folds $K$ and the use of the information-based warm-up score $\kappa_i$ in \cref{eq:eci_infoweight}; the Corrector regularization weight $\lambda_{\mathrm{cor}}$ and its capacity; the reliability-screening parameters $\gamma$, $n_0$, and soft-clipping sharpness $\xi$; and the reconstruction weight $\lambda_{\mathrm{rec}}$.

\subsection{Implementation}
\label{app:implementation}

We provide implementation details\footnote{\textbf{Our source code is available on GitHub: } \url{https://github.com/ecipp/ECI-PP}} for the experimental pipeline, including the computing environment, data preprocessing, expert classifiers and generated constraints, model backbones, pretraining, optimization, and hyperparameter settings.

\paragraph{Software and hardware.}
All model training code is implemented in Python~3.7 and PyTorch~1.5.1\footnote{PyTorch~1.5.1: \url{https://github.com/pytorch/pytorch/releases/tag/v1.5.1}}.
For methods whose final prediction is obtained by K-means on learned representations, we use the K-means implementation in scikit-learn\footnote{Scikit-learn webpage: \url{https://scikit-learn.org/0.19/documentation.html}}.
For the hierarchical clustering step used by the expert-aware ensemble references in Appendix~\ref{app:protocol}, we use the \texttt{fastcluster} package\footnote{fastcluster library: \url{https://pypi.org/project/fastcluster/}}.
Each training run is executed on a single NVIDIA A100 40GB or NVIDIA H200 141GB GPU.

\paragraph{Data preprocessing.}
We follow the same preprocessing pipeline as \cite{SpherePair}.
For Reuters subset and RCV1-10, we directly use the preprocessed tf--idf representations over the $2{,}000$ most frequent words or word stems, as described in Appendix~\ref{app:datasets}.
For MNIST and FMNIST, the $28\times28$ grayscale images are flattened into $784$-dimensional vectors, following the standard fully connected input format used in prior deep clustering/DCC studies \cite{SDEC,CIDEC1,DCGMM,SpherePair}.
For CIFAR10, CIFAR-10H, CIFAR100-20, STL10, and ImageNet10, we adopt the unsupervised feature extraction strategy in \cite{ContrastiveClustering}: a ResNet-34 model \cite{resnet} is trained for $1{,}000$ epochs, and the resulting $512$-dimensional representations are used as input features.
This image preprocessing follows the setting used in DCC studies \cite{VolMaxDCC,SpherePair} and keeps the inputs compatible with fully connected architectures across datasets.
The same preprocessed inputs are used for expert classifier training and for all compared clustering methods.

\begin{table}[hbtp]
\centering
\caption{Test accuracy (\%) of single-expert classifiers used for constraint generation.}
\label{tab:single_expert_acc}
\small
\setlength{\tabcolsep}{5pt}
\begin{tabular}{lccc}
\toprule
Dataset & {\tt lv0.1} & {\tt lv0.01} & {\tt lv0.001} \\
\midrule
CIFAR100-20    & 73.08 & 65.23 & 46.58 \\
CIFAR10        & 91.01 & 88.49 & 73.83 \\
FMNIST         & 84.45 & 74.77 & 59.83 \\
ImageNet10     & 96.27 & 92.62 & 87.62 \\
MNIST          & 93.67 & 88.76 & 65.20 \\
Reuters        & 92.90 & 83.60 & 49.75 \\
STL10          & 90.04 & 84.31 & 72.77 \\
RCV1-10        & 94.26 & 90.47 & 77.62 \\
\bottomrule
\end{tabular}
\end{table}

\begin{table}[hbtp]
\centering
\caption{Average test accuracy (\%) of multi-expert groups used for constraint generation. Each value is the arithmetic mean over experts in the corresponding group.}
\label{tab:multi_expert_avg_acc}
\small
\setlength{\tabcolsep}{5pt}
\begin{tabular}{lccc}
\toprule
Dataset & {\tt multi2} & {\tt multi3} & {\tt multi10} \\
\midrule
CIFAR100-20    & 41.23 & 51.90 & 67.36 \\
CIFAR10        & 52.72 & 65.94 & 82.74 \\
FMNIST         & 54.57 & 60.68 & 77.64 \\
ImageNet10     & 73.87 & 76.73 & 90.80 \\
MNIST          & 50.45 & 65.89 & 86.13 \\
Reuters        & 47.58 & 63.02 & 71.42 \\
STL10          & 53.98 & 65.67 & 83.22 \\
RCV1-10        & 49.06 & 64.57 & 85.96 \\
\bottomrule
\end{tabular}
\end{table}

\begin{figure}[hbtp]
    \centering
    \includegraphics[width=\linewidth]{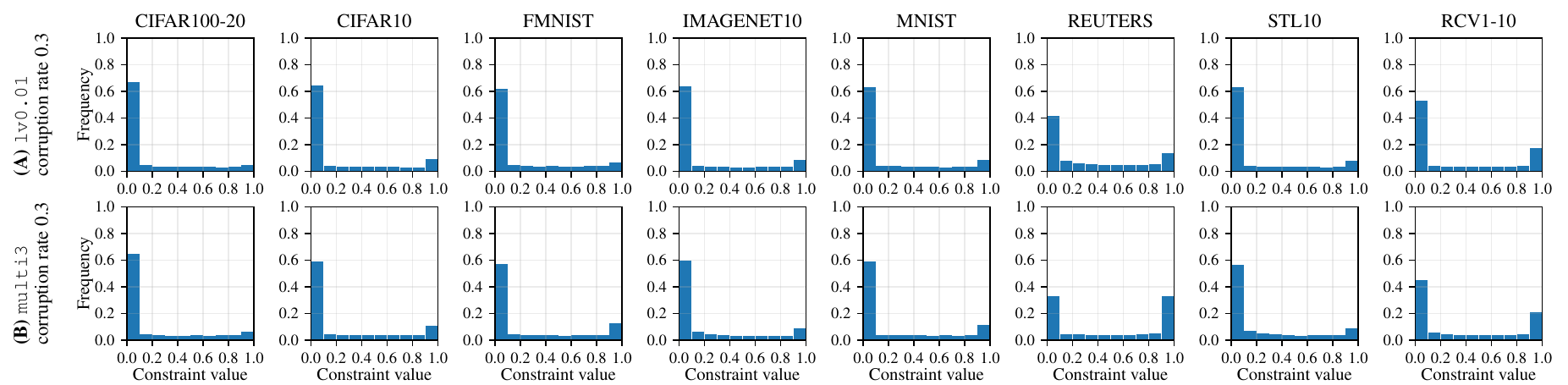}
    \caption{
    Distributions of recorded constraint values for representative generated-supervision settings.
    The top row shows {\tt lv0.01} single-expert supervision and the bottom row shows {\tt multi3} multi-expert supervision, both with corruption probability $0.3$.
    }
    \label{fig:app_constraint_value_dist}
\end{figure}

\paragraph{Expert classifiers and generated constraints.}
The expert classifiers used in Appendix~\ref{app:constraint_generation} are implemented as MLP classifiers with two hidden layers of size $512$--$512$, ReLU activations, dropout rate $0.2$, and a final softmax classification layer.
They are trained with cross-entropy loss using Adam, with learning rate $10^{-3}$, batch size $256$, a maximum of $100$ epochs, and an early-stopping patience of $10$ based on training accuracy.
The resulting test accuracies of single experts are reported in \cref{tab:single_expert_acc}, and the average test accuracies of multi-expert groups are reported in \cref{tab:multi_expert_avg_acc}.
These accuracies characterize the effective quality of the simulated experts produced by the subset-selection and blind-spot protocols in Appendix~\ref{app:constraint_generation}.
After the corruption channel in \cref{sec:observation}, these expert judgments yield the recorded constraint values used for training.
\cref{fig:app_constraint_value_dist} visualizes their per-dataset distributions for representative {\tt lv0.01} single-expert and {\tt multi3} multi-expert settings, both with corruption probability $0.3$.

\paragraph{Model backbones and pretraining.}
For encoder-based methods, we follow the fully connected autoencoder backbone used in \cite{SDEC,CIDEC1,SpherePair}: the encoder has hidden layers $500$--$500$--$2000$, paired with a symmetric decoder when reconstruction is used.
The embedding dimension is set to $D=10$ for all datasets except CIFAR100-20, where $D=20$, following \cite{SpherePair}.
VanillaDCC and VolMaxDCC use their native MLP classifier architectures with two hidden layers of size $512$--$512$, and do not use autoencoder pretraining.
For pretrain-capable methods, we use the two-stage stacked denoising autoencoder (SDAE) pretraining approach \cite{hinton2006reducing,vincent2008denoising}, following the protocol used in \cite{SDEC,CIDEC1,SpherePair}: hidden layers are first pretrained layer-wise as denoising autoencoders, followed by end-to-end autoencoder fine-tuning on the training split.
Unless otherwise specified, all pretrain-capable methods use this pretrained initialization.
For methods using angular reconstruction (SpherePair, ProbPair, Weighted ProbPair, ECI-PP), decoding is performed from normalized embeddings as in \cite{SpherePair}.

In ECI-PP, both the $K$ Estimator backbones and the Integrator use the same autoencoder architecture described above.
Each Estimator fold has its own encoder--decoder backbone and its own probabilistic readout parameters $(m_k,T_k)$, while the Integrator has its own readout parameters $(m,T)$ in \cref{eq:probpair_readout}.
These scalar readout parameters are parameterized as $m=\tanh(\tilde m)$ and $T=\mathrm{softplus}(\tilde T)$ for numerical stability within the cosine range, with the same parameterization used for each Estimator fold.
When pretraining is used, the same pretrained autoencoder weights initialize the Integrator backbone and all Estimator backbones, while the readout parameters and Correctors are initialized separately.
Unless otherwise specified, each expert-specific Corrector is an MLP with hidden layers $64$--$16$ and a $\tanh$ output, which keeps the predicted probability-space residual within $(-1,1)$.

\paragraph{Optimization and hyperparameters.}
We use the reported or recommended hyperparameters for the baseline methods whenever available.
Below, we summarize the optimization settings and implementation configurations for ECI-PP and all compared baselines.

\begin{itemize}
    \item \textbf{VanillaDCC.}
    VanillaDCC optimizes the MCL loss $\mathcal{L}_{\mathrm{MCL}}$ using Adam with learning rate $10^{-3}$ and batch size $256$.
    Training runs for at most $500$ epochs, with early stopping triggered when the relative change in the soft cluster assignments on the training samples falls below $10^{-3}$ for two consecutive checks.
    This assignment-change stopping rule follows common practice in deep clustering and deep constrained clustering \cite{DEC,IDEC,SDEC,CIDEC1}.

    \item \textbf{VolMaxDCC.}
    For VolMaxDCC, we follow its original optimization and validation-based hyperparameter-search procedure \cite{VolMaxDCC}.
    The learnable matrix $B$ is parameterized elementwise as
    $B_{cc'}=\left(1+\exp(-B'_{cc'})\right)^{-1}$, where $B'_{cc'}$ is initialized to $1$ for $c=c'$ and to $-1$ otherwise.
    Optimization uses SGD with learning rate $0.5$ for the network parameters and $0.1$ for $B'$, with batch size $128$.
    The VolMaxDCC geometric regularization weight is selected from
    $\{0,10^{-1},10^{-2},10^{-3},10^{-4},10^{-5}\}$
    using the reserved validation split; for each dataset and supervision setting, we choose the value with the highest mean validation ACC over three validation runs.
    For multi-expert supervision, a shared selected value is used for the multi-expert settings defined in Appendix~\ref{app:constraint_generation}, where the familiar and unfamiliar labeled-data fractions are fixed as $h=0.1$ and $l=0.0001$.
    The selected values for the eight main benchmark datasets under simulated-expert supervision are reported in \cref{tab:volmax_lambda_selection}; for the CIFAR-10H human-derived evaluation, the same search selects $\lambda=10^{-2}$.

\begin{table}[t]
\centering
\caption{Selected VolMaxDCC geometric regularization weights for the eight main benchmark datasets under simulated-expert supervision. Values are selected by validation ACC over three runs for each dataset and supervision setting.}
\label{tab:volmax_lambda_selection}
\small
\setlength{\tabcolsep}{5pt}
\begin{tabular}{lcccc}
\toprule
Dataset & {\tt lv0.1} & {\tt lv0.01} & {\tt lv0.001} & {\tt multi} \\
\midrule
CIFAR100-20    & $10^{-3}$ & $10^{-3}$ & $10^{-1}$ & $10^{-2}$ \\
CIFAR10        & $10^{-3}$ & $10^{-3}$ & $10^{-1}$ & $10^{-2}$ \\
FMNIST         & $10^{-3}$ & $10^{-3}$ & $10^{-2}$ & $10^{-2}$ \\
ImageNet10     & $0$       & $10^{-4}$ & $10^{-1}$ & $10^{-2}$ \\
MNIST          & $10^{-2}$ & $10^{-3}$ & $10^{-2}$ & $10^{-2}$ \\
Reuters        & $10^{-2}$ & $0$       & $10^{-1}$ & $10^{-1}$ \\
STL10          & $10^{-3}$ & $10^{-3}$ & $10^{-2}$ & $10^{-1}$ \\
RCV1-10        & $10^{-4}$ & $10^{-4}$ & $10^{-5}$ & $10^{-3}$ \\
\bottomrule
\end{tabular}
\end{table}

    \item \textbf{CIDEC.}
    For CIDEC, we follow the authors' recommended settings: the reconstruction/clustering trade-off is set to $\lambda_1=1$, and the MCL constraint-balancing parameter is set to $\lambda_2=0.1$.
    The $C$ cluster anchors are initialized by K-means.
    Optimization uses Adam with learning rate $10^{-3}$ and batch size $256$.
    Training proceeds for at most $500$ epochs, with early stopping applied when the relative change in the soft assignments falls below $10^{-3}$ over consecutive checks.


    \item \textbf{SpherePair.}
    For SpherePair, the reconstruction weight is set to $0.02$, matching the setting used in \cite{SpherePair}.
    In the original binary formulation, the negative-zone factor is fixed as $\omega=2$, which yields the theoretically guaranteed $\pi/2$ negative-zone boundary.
    As described in Appendix~\ref{app:compared_methods}, we use the real-valued extension of SpherePair by keeping the same $\omega=2$ angular scale and replacing the original clamped cannot-link score with the continuous counterpart
    $s^{-}_{\mathrm{soft},a_i b_i}=\frac{1}{2}\big(\cos(2\theta_{\boldsymbol{z}_{a_i},\boldsymbol{z}_{b_i}})+1\big)$.
    The observed $y_i\in[0,1]$ is then used directly in the BCE-form angular loss.
    Under this implementation, $y_i=1$ favors $\theta_{\boldsymbol{z}_{a_i},\boldsymbol{z}_{b_i}}=0$ and $y_i=0$ favors $\theta_{\boldsymbol{z}_{a_i},\boldsymbol{z}_{b_i}}=\pi/2$, while intermediate labels continuously interpolate between these two angular endpoints through the two BCE terms.
    Optimization uses Adam with learning rate $10^{-3}$ and constraint mini-batch size $256$.
    The instance mini-batch size for reconstruction is set according to the number of constraint mini-batches, following \cite{SpherePair}.
    Training runs for at most $500$ epochs, with early stopping after the first $100$ epochs if the relative change in the total loss remains below $0.1$ for five consecutive checks.
    For the expert-aware ensemble reference in Appendix~\ref{app:protocol}, the expert-aware hyperparameter $\alpha$ is selected on the reserved validation split; the selected values are reported in \cref{tab:expert_aware_alpha_selection}.

\begin{table}[t]
\centering
\caption{Selected expert-aware ensemble hyperparameter $\alpha$. Values are chosen by mean validation ACC over three runs and shared across the multi-expert configurations defined in Appendix~\ref{app:constraint_generation}.}
\label{tab:expert_aware_alpha_selection}
\small
\setlength{\tabcolsep}{7pt}
\begin{tabular}{lcc}
\toprule
Dataset & SpherePair & Weighted ProbPair \\
\midrule
CIFAR100-20    & $0.05$ & $0.05$ \\
CIFAR10        & $1.00$ & $0.05$ \\
FMNIST         & $1.00$ & $0.50$ \\
ImageNet10     & $0.00$ & $0.10$ \\
MNIST          & $0.05$ & $0.50$ \\
Reuters        & $0.05$ & $0.05$ \\
STL10          & $0.01$ & $0.01$ \\
RCV1-10        & $0.01$ & $0.05$ \\
\bottomrule
\end{tabular}
\end{table}

    \item \textbf{ProbPair variants and ECI-PP (Full).}
    ProbPair, Weighted ProbPair, and ECI-PP use the ProbPair-style readout in \cref{eq:probpair_readout} and the reconstruction weight $\lambda_{\mathrm{rec}}=0.02$.
    ProbPair optimizes \cref{eq:probpair_total}, while Weighted ProbPair additionally applies the information-based weight $\kappa_i$ in \cref{eq:eci_infoweight}.
    For the expert-aware ensemble reference based on Weighted ProbPair in Appendix~\ref{app:protocol}, the expert-aware hyperparameter $\alpha$ is selected on the reserved validation split; the selected values are reported in \cref{tab:expert_aware_alpha_selection}.
    The shared base optimizer is Adam with learning rate $10^{-3}$ and batch size $256$.
    The probabilistic readout parameters, including $(m,T)$ and the fold-wise $(m_k,T_k)$ when present, are initialized as $(0,0.1)$ and optimized with learning rate $10^{-2}$.
    We use a fixed outer training budget of $500$ epochs; principled stopping criteria for iterative refinement remain an open question for future study.
    For any weighted ProbPair-style objective, including Weighted ProbPair and the weighted Estimator/Integrator updates in ECI-PP, the weight is applied inside the constraint average:
    \[
    \mathcal{L}^{w}_{\mathrm{PP}}
    =
    -\frac{1}{|\mathcal{C}|}
    \sum_{i=1}^{|\mathcal{C}|}
    w_i
    \Big[
    \tilde{y}_i\log \hat{y}_{a_i b_i}
    +
    (1-\tilde{y}_i)\log(1-\hat{y}_{a_i b_i})
    \Big],
    \]
    where $\tilde{y}_i$ denotes the corresponding training target, such as $y_i$, $y^{\mathrm{int}}_i$, or $y_i^{\mathrm{BC}}$.

    For ECI-PP, we use one global hyperparameter setting unless explicitly varied in sensitivity analysis: $K=5$ estimator folds, $\lambda_{\mathrm{cor}}=0.5$, $\gamma=10$, $n_0=10$, soft-clipping sharpness $\xi=20$, and information-based warm-up enabled.
    The default warm-up trains the Estimators and Integrator for $50$ epochs before iterative refinement.
    The initial $\hat{y}^{\mathrm{oof}}_i$ is obtained by strict fold-heldout prediction, and later refinement rounds use refreshed Estimator beliefs under the Integrator-updated targets while retaining the same fold structure.
    Each Corrector update
    uses the shared base optimizer setting, with the Huber quadratic-to-linear transition in \cref{eq:eci_corrector_loss} fixed at $0.1$, holds out 10\% of its assigned constraints for internal
    validation, runs for at most 50 epochs, and stops early if the validation loss does not improve
    for 10 consecutive epochs.
    Between refinement iterations, Gaussian noise with standard deviation $0.01$ is added to the final linear layers of the Estimators, Correctors, and Integrator to mildly perturb the next refinement step.
\end{itemize}

Notably, the validation-based selections of the VolMaxDCC geometric regularization weight and the expert-aware ensemble hyperparameter $\alpha$ rely on ground-truth class labels on the validation split, which are \textbf{typically unavailable} in constrained clustering and are not part of the observable supervision under UPCC.
In contrast, ECI-PP uses \textbf{one global hyperparameter setting} across datasets and handles expert identities directly through expert-specific Correctors, without dataset-specific tuning or an additional hyperparameterized ensemble wrapper.

\section{Additional experimental results}
\label{app:additional_results}

\subsection{Comparison under different constraint budgets}
\label{app:full_budget_results}

\cref{table:exp1_1_pretrain,table:exp1_2_no_ensemble} report the full main-comparison results under the default {\tt lv0.01} single-expert and {\tt multi3} multi-expert regimes, with corruption probability fixed at $0.3$.
The total constraint budget is varied over $3$k/$6$k/$9$k; under {\tt multi3}, these correspond to $1$k/$2$k/$3$k constraints per expert.
All results in this subsection use pretrained backbones for methods with a pretraining stage, while VanillaDCC and VolMaxDCC have no pretraining variant.

\begin{table}[htbp]
    \caption{Comparative performance (\%) (ACC, NMI, ARI) across datasets for methods under the {\tt lv0.01} setting, with corruption probability $0.3$ and $3$k/$6$k/$9$k constraints. \textcolor{blue}{Blue} and black represent \textcolor{blue}{training} and test results, respectively. Best results are in \textbf{bold}, and second-best are \underline{underlined}.}
    \label{table:exp1_1_pretrain}
    \begin{center}
    \begin{scriptsize}
    \renewcommand{\arraystretch}{0.0}
    \setlength{\tabcolsep}{3.1pt}
    \setlength{\aboverulesep}{0.6pt}
    \setlength{\belowrulesep}{0.6pt}
    \setlength{\extrarowheight}{0pt}
\begin{tabular}{lllccccccc}
        \toprule
        & & & \multicolumn{4}{c}{\textbf{External baselines}} & \multicolumn{3}{c}{\textbf{Our progressive models}} \\
        \cmidrule(lr){4-7}\cmidrule(lr){8-10}
         &  &  & \makecell{Vanilla-\\DCC} & \makecell{VolMax-\\DCC} & CIDEC & SpherePair & ProbPair & \makecell{Weighted\\ProbPair} & \makecell{\textbf{ECI-PP}\\(Full)} \\
        \midrule
        \multirow{9}{*}{\raisebox{-3ex}{CIFAR100-20}}
        & \multirow{3}{*}{3k} & ACC & \textcolor{blue}{14.7}, 14.7 & \textcolor{blue}{42.8}, 42.5 & \textcolor{blue}{14.9}, 14.6 & \textcolor{blue}{\underline{45.7}}, \underline{45.7} & \textcolor{blue}{44.5}, 44.4 & \textcolor{blue}{45.2}, 45.2 & \textcolor{blue}{\textbf{48.1}}, \textbf{47.8} \\
        &  & NMI & \textcolor{blue}{11.2}, 11.8 & \textcolor{blue}{42.5}, 42.9 & \textcolor{blue}{13.1}, 13.6 & \textcolor{blue}{42.8}, 43.3 & \textcolor{blue}{42.4}, 42.9 & \textcolor{blue}{\underline{43.4}}, \underline{43.7} & \textcolor{blue}{\textbf{46.0}}, \textbf{46.5} \\
        &  & ARI & \textcolor{blue}{4.0}, 4.1 & \textcolor{blue}{26.2}, 26.0 & \textcolor{blue}{2.6}, 2.1 & \textcolor{blue}{28.5}, \underline{28.5} & \textcolor{blue}{27.6}, 27.6 & \textcolor{blue}{\underline{28.6}}, \underline{28.5} & \textcolor{blue}{\textbf{31.3}}, \textbf{31.4} \\
        \cmidrule(lr){2-10}
        & \multirow{3}{*}{6k} & ACC & \textcolor{blue}{16.5}, 16.6 & \textcolor{blue}{45.4}, 45.3 & \textcolor{blue}{16.6}, 15.6 & \textcolor{blue}{\underline{49.0}}, \underline{49.2} & \textcolor{blue}{46.1}, 46.4 & \textcolor{blue}{48.3}, 48.5 & \textcolor{blue}{\textbf{51.2}}, \textbf{51.1} \\
        &  & NMI & \textcolor{blue}{12.6}, 13.5 & \textcolor{blue}{\underline{43.9}}, 44.5 & \textcolor{blue}{15.0}, 15.1 & \textcolor{blue}{43.8}, \underline{44.9} & \textcolor{blue}{42.4}, 43.4 & \textcolor{blue}{43.7}, 44.5 & \textcolor{blue}{\textbf{48.3}}, \textbf{48.8} \\
\bottomrule
    \end{tabular}
    \end{scriptsize}
    \end{center}
\end{table}
\begin{table}[htbp]
    \caption{Comparative performance (\%) (ACC, NMI, ARI) across datasets for methods under the {\tt multi3} setting, with corruption probability $0.3$ and $3$k/$6$k/$9$k total constraints. \textcolor{blue}{Blue} and black represent \textcolor{blue}{training} and test results, respectively. Best results are in \textbf{bold}, and second-best are \underline{underlined}.}
    \label{table:exp1_2_no_ensemble}
    \begin{center}
    \begin{scriptsize}
    \renewcommand{\arraystretch}{0.0}
    \setlength{\tabcolsep}{3.1pt}
    \setlength{\aboverulesep}{0.6pt}
    \setlength{\belowrulesep}{0.6pt}
    \setlength{\extrarowheight}{0pt}
\begin{tabular}{lllccccccc}
        \toprule
        & & & \multicolumn{4}{c}{\textbf{External baselines}} & \multicolumn{3}{c}{\textbf{Our progressive models}} \\
        \cmidrule(lr){4-7}\cmidrule(lr){8-10}
         &  &  & \makecell{Vanilla-\\DCC} & \makecell{VolMax-\\DCC} & \makecell{CIDEC} & \makecell{SpherePair} & \makecell{ProbPair} & \makecell{Weighted\\ProbPair} & \makecell{\textbf{ECI-PP}\\(Full)} \\
        \midrule
        \multirow{9}{*}{\raisebox{-3ex}{CIFAR100-20}} & \multirow{3}{*}{3k} & ACC & \textcolor{blue}{12.2}, 12.4 & \textcolor{blue}{\underline{43.7}}, \underline{43.3} & \textcolor{blue}{19.5}, 18.4 & \textcolor{blue}{42.0}, 42.1 & \textcolor{blue}{41.5}, 41.6 & \textcolor{blue}{43.1}, \underline{43.3} & \textcolor{blue}{\textbf{45.6}}, \textbf{45.3} \\
         &  & NMI & \textcolor{blue}{9.5}, 9.9 & \textcolor{blue}{\underline{43.4}}, \underline{43.6} & \textcolor{blue}{20.5}, 20.5 & \textcolor{blue}{41.5}, 42.1 & \textcolor{blue}{40.2}, 41.1 & \textcolor{blue}{41.9}, 42.4 & \textcolor{blue}{\textbf{45.0}}, \textbf{45.5} \\
         &  & ARI & \textcolor{blue}{2.6}, 2.5 & \textcolor{blue}{\underline{26.8}}, 26.6 & \textcolor{blue}{3.9}, 3.2 & \textcolor{blue}{26.1}, 26.3 & \textcolor{blue}{24.8}, 25.3 & \textcolor{blue}{26.7}, \underline{26.8} & \textcolor{blue}{\textbf{29.7}}, \textbf{29.8} \\
        \cmidrule(lr){2-10}
         & \multirow{3}{*}{6k} & ACC & \textcolor{blue}{14.2}, 14.4 & \textcolor{blue}{43.1}, 42.5 & \textcolor{blue}{17.4}, 16.6 & \textcolor{blue}{42.4}, 42.4 & \textcolor{blue}{42.0}, 42.1 & \textcolor{blue}{\underline{44.1}}, \underline{44.4} & \textcolor{blue}{\textbf{49.1}}, \textbf{49.3} \\
         &  & NMI & \textcolor{blue}{11.3}, 12.2 & \textcolor{blue}{\underline{44.3}}, \underline{44.5} & \textcolor{blue}{16.2}, 16.3 & \textcolor{blue}{40.0}, 40.9 & \textcolor{blue}{39.8}, 40.6 & \textcolor{blue}{41.1}, 41.9 & \textcolor{blue}{\textbf{47.1}}, \textbf{47.7} \\
\bottomrule
    \end{tabular}
    \end{scriptsize}
    \end{center}
\end{table}

The main-text conclusion remains stable across budgets.
Across the $144$ test entries in the two tables, ECI-PP ranks first in $122/144$ entries and within the top two in $134/144$, while one ProbPair-family method is best in $137/144$ entries.
The budget-wise pattern is not completely uniform, as expected: with only $3$k constraints, all methods are less stable and the separation between methods is smaller, whereas at $9$k several strong baselines also improve and narrow some ACC gaps.
Even under these effects, ECI-PP remains the leading method at every budget, ranking first in $38/48$, $43/48$, and $41/48$ entries under $3$k, $6$k, and $9$k constraints, respectively.

The advantage is clearest on structure-sensitive metrics.
ECI-PP gives the best NMI in $46/48$ entries and the best ARI in $43/48$, compared with $33/48$ for ACC.
Since NMI and ARI better reflect grouping consistency than matched-label accuracy alone, this suggests that ECI-PP most reliably improves the learned clustering structure.
For example, on single-expert CIFAR100-20, ECI-PP raises NMI from roughly $44$--$46\%$ for the strongest non-ECI competitors to $47$--$50\%$ across budgets; on multi-expert Reuters, it raises NMI from about $53$--$58\%$ for Weighted ProbPair to $65$--$69\%$.
Dataset-wise, ECI-PP is best in all entries on CIFAR100-20, Reuters, and STL10, and in nearly all entries on CIFAR10 and ImageNet10.
FMNIST and MNIST show more competition from Weighted ProbPair or SpherePair, especially on ACC and low-budget MNIST, but ECI-PP remains strongest or near-strongest on NMI/ARI in most cases.

The ProbPair variants show a consistent ablation pattern.
Weighted ProbPair improves over ProbPair in most entries, and this effect is stronger under multi-expert supervision ($67/72$ entries) than under single-expert supervision ($62/72$ entries).
This supports the interpretation that information-based weighting is especially useful when low-decisiveness constraints may reflect expert unfamiliarity.
ECI-PP further improves over Weighted ProbPair in $129/144$ entries, showing that weighting alone does not replace explicit estimation, correction, and integration.

Notably, VolMaxDCC also provides a useful noise-aware reference among the non-ProbPair baselines.
It improves over VanillaDCC and is competitive on some datasets, but reaches the top two in only $15/144$ entries.
This suggests that explicit noise modeling is beneficial, yet its hard-noise-oriented formulation is not sufficient for the expert-conditioned probabilistic supervision in UPCC.

The main caveat remains the severely imbalanced RCV1-10 benchmark.
SpherePair remains the main ACC competitor there, consistent with the advantage of its angular geometry for imbalanced pair structures.
Nevertheless, ECI-PP gives the best NMI across all RCV1-10 budgets and regimes, and is often strongest in ARI.
Thus, ECI-PP corrects much of the weakness of plain ProbPair/Weighted ProbPair on this benchmark, while the ACC results indicate that severe class imbalance remains a limitation for the probabilistic formulation.

\subsection{Comparison without pretraining}
\label{app:no_pretraining_results}

\begin{table}[!h]
    \caption{Comparative performance (\%) (ACC, NMI, ARI) across datasets for methods under the {\tt lv0.01} setting with corruption probability $0.3$ and $3$k/$6$k/$9$k constraints. $^{\dagger}$ indicates models without pretraining. \textcolor{blue}{Blue} and black represent \textcolor{blue}{training} and test results, respectively. Best results are in \textbf{bold}, and second-best are \underline{underlined}.}
    \label{table:exp1_1_nopretrain}
    \begin{center}
    \begin{scriptsize}
    \renewcommand{\arraystretch}{0.0}
    \setlength{\tabcolsep}{3.1pt}
    \setlength{\aboverulesep}{0.6pt}
    \setlength{\belowrulesep}{0.6pt}
    \setlength{\extrarowheight}{0pt}
\begin{tabular}{lllccccccc}
        \toprule
        & & & \multicolumn{4}{c}{\textbf{External baselines}} & \multicolumn{3}{c}{\textbf{Our progressive models}} \\
        \cmidrule(lr){4-7}\cmidrule(lr){8-10}
         &  &  & \makecell{Vanilla-\\DCC} & \makecell{VolMax-\\DCC} & CIDEC$^{\dagger}$ & SpherePair$^{\dagger}$ & ProbPair$^{\dagger}$ & \makecell{Weighted\\ProbPair$^{\dagger}$} & \makecell{\textbf{ECI-PP}$^{\dagger}$\\(Full)} \\
        \midrule
        \multirow{9}{*}{\raisebox{-3ex}{CIFAR100-20}}
        & \multirow{3}{*}{3k} & ACC & \textcolor{blue}{14.7}, 14.7 & \textcolor{blue}{42.8}, 42.5 & \textcolor{blue}{14.4}, 14.2 & \textcolor{blue}{\underline{44.0}}, \underline{43.8} & \textcolor{blue}{43.2}, 43.0 & \textcolor{blue}{42.0}, 41.9 & \textcolor{blue}{\textbf{46.6}}, \textbf{46.5} \\
        &  & NMI & \textcolor{blue}{11.2}, 11.8 & \textcolor{blue}{\underline{42.5}}, \underline{42.9} & \textcolor{blue}{11.2}, 11.5 & \textcolor{blue}{41.9}, 42.4 & \textcolor{blue}{39.7}, 40.4 & \textcolor{blue}{38.9}, 39.5 & \textcolor{blue}{\textbf{45.5}}, \textbf{46.2} \\
        &  & ARI & \textcolor{blue}{4.0}, 4.1 & \textcolor{blue}{26.2}, 26.0 & \textcolor{blue}{3.7}, 3.5 & \textcolor{blue}{\underline{27.5}}, \underline{27.5} & \textcolor{blue}{25.7}, 25.8 & \textcolor{blue}{24.9}, 25.0 & \textcolor{blue}{\textbf{30.5}}, \textbf{30.8} \\
        \cmidrule(lr){2-10}
        & \multirow{3}{*}{6k} & ACC & \textcolor{blue}{16.5}, 16.6 & \textcolor{blue}{45.4}, 45.3 & \textcolor{blue}{12.3}, 12.2 & \textcolor{blue}{\underline{47.6}}, \underline{47.8} & \textcolor{blue}{44.7}, 44.9 & \textcolor{blue}{45.2}, 45.5 & \textcolor{blue}{\textbf{49.3}}, \textbf{49.1} \\
\bottomrule
    \end{tabular}
    \end{scriptsize}
    \end{center}
\end{table}

\cref{table:exp1_1_nopretrain} repeats the single-expert comparison without the autoencoder pretraining stage for pretrain-capable methods, denoted by $^{\dagger}$ in the table: CIDEC$^{\dagger}$, SpherePair$^{\dagger}$, ProbPair$^{\dagger}$, Weighted ProbPair$^{\dagger}$, and ECI-PP$^{\dagger}$.
These methods are trained from random initialization, while VanillaDCC and VolMaxDCC are unchanged because they have no pretraining variant.
All other settings match \cref{table:exp1_1_pretrain}: {\tt lv0.01} supervision, corruption probability $0.3$, and $3$k/$6$k/$9$k constraints.

The main effect of removing pretraining is not a collapse of ECI-PP$^{\dagger}$, but a sharper separation from the other pretrain-capable methods.
ECI-PP$^{\dagger}$ ranks first in $68/72$ test entries and within the top two in $71/72$, compared with $59/72$ and $66/72$ in the pretrained comparison.
It is also best in every entry on seven of the eight datasets; the only non-best entries occur on RCV1-10 ACC/ARI at $6$k/$9$k.
This indicates that ECI-PP$^{\dagger}$ is less dependent on a favorable autoencoder initialization than the compared deep baselines.

The contrast with SpherePair$^{\dagger}$ is also clearer without pretraining.
ECI-PP$^{\dagger}$ outperforms SpherePair$^{\dagger}$ in $68/72$ entries, compared with $62/72$ in the pretrained setting, and achieves the best NMI in all cases.
Its absolute NMI margins over SpherePair$^{\dagger}$ range from about $3$--$12\%$, with visible gains on FMNIST, MNIST, and STL10.
This suggests that the Estimator--Corrector--Integrator structure helps form a more reliable grouping structure when the representation is learned from scratch.

The ProbPair variants show a more mixed response to random initialization.
Weighted ProbPair$^{\dagger}$ still improves over ProbPair$^{\dagger}$ in most entries, so information-based weighting remains useful.
However, Weighted ProbPair$^{\dagger}$ no longer compares as consistently with SpherePair$^{\dagger}$: it beats SpherePair$^{\dagger}$ in only $33/72$ entries, down from $53/72$ with pretraining.
By contrast, ECI-PP$^{\dagger}$ improves over Weighted ProbPair$^{\dagger}$ in all entries, showing that weighting alone is not sufficient when the representation starts from a weaker initialization.

The remaining caveat is again the severely imbalanced RCV1-10 benchmark.
SpherePair$^{\dagger}$ keeps stronger ACC at $9$k and stronger ARI at $6$k/$9$k, consistent with the robustness of its angular geometry under severe imbalance.
Nevertheless, ECI-PP$^{\dagger}$ gives the best NMI at every budget and substantially improves over ProbPair$^{\dagger}$/Weighted ProbPair$^{\dagger}$, indicating stronger grouping consistency despite the remaining imbalance-related ACC/ARI limitation.

\subsection{Comparison with expert-aware baseline extensions}
\label{app:expert_aware_ensemble_results}

\begin{table}[!h]
    \caption{Comparative performance (\%) (ACC, NMI, ARI) across datasets for ECI-PP, SP$^{\mathrm{EA}}$, and WPP$^{\mathrm{EA}}$ under the {\tt multi3} setting, with corruption probability $0.3$ and $3$k/$6$k/$9$k total constraints. SP$^{\mathrm{EA}}$ and WPP$^{\mathrm{EA}}$ denote the expert-aware ensemble extensions of SpherePair and Weighted ProbPair, respectively. \textcolor{blue}{Blue} and black represent \textcolor{blue}{training} and test results, respectively. Best results are in \textbf{bold}, and second-best are \underline{underlined}.}
    \label{table:exp1_2_ensemble_vs_ecu}
    \begin{center}
    \begin{scriptsize}
    \renewcommand{\arraystretch}{1.0}
    \setlength{\tabcolsep}{3.1pt}
    \setlength{\aboverulesep}{1pt}
    \setlength{\belowrulesep}{1pt}
    \setlength{\extrarowheight}{0pt}
    \begin{tabular}{llccc@{\hskip 5pt}>{\hskip 5pt}ccc<{\hskip 7pt}@{\hskip 5pt}ccc}
    \toprule
      \multicolumn{1}{l}{} & & \multicolumn{3}{c}{3k} & \multicolumn{3}{c}{6k} & \multicolumn{3}{c}{9k} \\
    \cline{3-11}\noalign{\vskip 2pt}
      \multicolumn{1}{l}{} & & ACC & NMI & ARI & ACC & NMI & ARI & ACC & NMI & ARI \\
    \midrule
    \multirow{3}{*}{CIFAR100-20} & SP$^{\mathrm{EA}}$ & \textcolor{blue}{42.2}, 43.2 & \textcolor{blue}{\underline{41.7}}, \underline{42.5} & \textcolor{blue}{\underline{26.6}}, \underline{27.0} & \textcolor{blue}{41.0}, 40.9 & \textcolor{blue}{39.7}, 40.9 & \textcolor{blue}{25.5}, 25.8 & \textcolor{blue}{41.9}, 42.3 & \textcolor{blue}{39.2}, 40.5 & \textcolor{blue}{26.0}, 27.1 \\
     & WPP$^{\mathrm{EA}}$ & \textcolor{blue}{\underline{43.1}}, \underline{43.4} & \textcolor{blue}{41.4}, 42.2 & \textcolor{blue}{26.4}, 26.5 & \textcolor{blue}{\underline{41.5}}, \underline{41.4} & \textcolor{blue}{\underline{40.4}}, \underline{41.3} & \textcolor{blue}{\underline{25.8}}, \underline{26.1} & \textcolor{blue}{\underline{46.4}}, \underline{46.4} & \textcolor{blue}{\underline{42.9}}, \underline{43.8} & \textcolor{blue}{\underline{29.8}}, \underline{29.9} \\
     & \makecell{\textbf{ECI-PP}} & \textcolor{blue}{\textbf{45.6}}, \textbf{45.3} & \textcolor{blue}{\textbf{45.0}}, \textbf{45.5} & \textcolor{blue}{\textbf{29.7}}, \textbf{29.8} & \textcolor{blue}{\textbf{49.1}}, \textbf{49.3} & \textcolor{blue}{\textbf{47.1}}, \textbf{47.7} & \textcolor{blue}{\textbf{32.1}}, \textbf{32.4} & \textcolor{blue}{\textbf{48.6}}, \textbf{48.8} & \textcolor{blue}{\textbf{48.1}}, \textbf{48.5} & \textcolor{blue}{\textbf{32.9}}, \textbf{32.9} \\
    \midrule
    \multirow{3}{*}{CIFAR10} & SP$^{\mathrm{EA}}$ & \textcolor{blue}{76.6}, 77.1 & \textcolor{blue}{70.7}, 71.5 & \textcolor{blue}{64.0}, 64.8 & \textcolor{blue}{75.5}, 75.9 & \textcolor{blue}{67.9}, 69.4 & \textcolor{blue}{62.4}, 63.5 & \textcolor{blue}{74.1}, 75.2 & \textcolor{blue}{67.0}, 69.0 & \textcolor{blue}{61.3}, 62.9 \\
     & WPP$^{\mathrm{EA}}$ & \textcolor{blue}{\underline{82.6}}, \textbf{82.5} & \textcolor{blue}{\underline{74.1}}, \underline{74.2} & \textcolor{blue}{\underline{69.4}}, \underline{69.2} & \textcolor{blue}{\underline{81.2}}, \underline{81.2} & \textcolor{blue}{\underline{74.6}}, \underline{74.9} & \textcolor{blue}{\underline{69.6}}, \underline{69.6} & \textcolor{blue}{\textbf{82.5}}, \textbf{82.8} & \textcolor{blue}{\underline{74.2}}, \underline{74.5} & \textcolor{blue}{\underline{70.0}}, \underline{70.3} \\
     & \makecell{\textbf{ECI-PP}} & \textcolor{blue}{\textbf{82.8}}, \textbf{82.5} & \textcolor{blue}{\textbf{76.0}}, \textbf{75.9} & \textcolor{blue}{\textbf{70.8}}, \textbf{70.7} & \textcolor{blue}{\textbf{84.0}}, \textbf{84.0} & \textcolor{blue}{\textbf{77.0}}, \textbf{77.1} & \textcolor{blue}{\textbf{72.6}}, \textbf{72.6} & \textcolor{blue}{\underline{81.6}}, \underline{82.1} & \textcolor{blue}{\textbf{77.8}}, \textbf{77.8} & \textcolor{blue}{\textbf{73.1}}, \textbf{73.0} \\
    \midrule
    \multirow{3}{*}{FMNIST} & SP$^{\mathrm{EA}}$ & \textcolor{blue}{\underline{63.2}}, \underline{61.6} & \textcolor{blue}{\underline{60.3}}, 59.7 & \textcolor{blue}{\underline{48.4}}, 46.9 & \textcolor{blue}{\underline{63.5}}, \underline{62.5} & \textcolor{blue}{59.4}, 59.9 & \textcolor{blue}{48.5}, 48.3 & \textcolor{blue}{62.5}, 62.5 & \textcolor{blue}{58.6}, 60.2 & \textcolor{blue}{49.0}, 49.4 \\
     & WPP$^{\mathrm{EA}}$ & \textcolor{blue}{\textbf{64.5}}, \textbf{64.3} & \textcolor{blue}{59.7}, \underline{59.8} & \textcolor{blue}{47.9}, \underline{47.7} & \textcolor{blue}{\textbf{68.0}}, \textbf{68.3} & \textcolor{blue}{\underline{62.1}}, \underline{62.4} & \textcolor{blue}{\textbf{52.0}}, \textbf{52.0} & \textcolor{blue}{\textbf{71.8}}, \textbf{71.7} & \textcolor{blue}{\underline{64.6}}, \underline{64.8} & \textcolor{blue}{\textbf{55.2}}, \textbf{55.1} \\
     & \makecell{\textbf{ECI-PP}} & \textcolor{blue}{60.1}, 59.2 & \textcolor{blue}{\textbf{62.7}}, \textbf{61.9} & \textcolor{blue}{\textbf{49.5}}, \textbf{48.4} & \textcolor{blue}{62.0}, 61.1 & \textcolor{blue}{\textbf{64.3}}, \textbf{63.8} & \textcolor{blue}{\underline{51.0}}, \underline{51.0} & \textcolor{blue}{\underline{65.3}}, \underline{64.8} & \textcolor{blue}{\textbf{66.0}}, \textbf{65.6} & \textcolor{blue}{\underline{53.4}}, \underline{52.8} \\
    \midrule
    \multirow{3}{*}{ImageNet10} & SP$^{\mathrm{EA}}$ & \textcolor{blue}{81.7}, 85.3 & \textcolor{blue}{78.0}, 79.8 & \textcolor{blue}{71.8}, 75.2 & \textcolor{blue}{82.1}, 86.6 & \textcolor{blue}{77.8}, 79.6 & \textcolor{blue}{72.8}, 75.8 & \textcolor{blue}{81.2}, 81.4 & \textcolor{blue}{75.3}, 76.7 & \textcolor{blue}{70.6}, 71.3 \\
     & WPP$^{\mathrm{EA}}$ & \textcolor{blue}{\underline{90.5}}, \underline{90.8} & \textcolor{blue}{\underline{83.8}}, \underline{84.2} & \textcolor{blue}{\underline{81.8}}, \underline{82.0} & \textcolor{blue}{\underline{90.9}}, \underline{91.6} & \textcolor{blue}{\underline{83.4}}, \underline{84.8} & \textcolor{blue}{\underline{82.2}}, \underline{83.4} & \textcolor{blue}{\underline{91.2}}, \underline{92.1} & \textcolor{blue}{\underline{82.4}}, \underline{84.7} & \textcolor{blue}{\underline{82.2}}, \underline{83.8} \\
     & \makecell{\textbf{ECI-PP}} & \textcolor{blue}{\textbf{91.5}}, \textbf{91.8} & \textcolor{blue}{\textbf{86.9}}, \textbf{87.3} & \textcolor{blue}{\textbf{84.1}}, \textbf{84.4} & \textcolor{blue}{\textbf{91.9}}, \textbf{92.1} & \textcolor{blue}{\textbf{87.2}}, \textbf{87.8} & \textcolor{blue}{\textbf{84.6}}, \textbf{85.0} & \textcolor{blue}{\textbf{92.4}}, \textbf{92.6} & \textcolor{blue}{\textbf{87.7}}, \textbf{88.1} & \textcolor{blue}{\textbf{85.4}}, \textbf{85.7} \\
    \midrule
    \multirow{3}{*}{MNIST} & SP$^{\mathrm{EA}}$ & \textcolor{blue}{80.5}, 81.1 & \textcolor{blue}{75.2}, 76.7 & \textcolor{blue}{70.7}, 71.7 & \textcolor{blue}{78.6}, 79.6 & \textcolor{blue}{72.7}, 75.5 & \textcolor{blue}{67.9}, 70.2 & \textcolor{blue}{77.5}, 78.1 & \textcolor{blue}{70.2}, 73.7 & \textcolor{blue}{65.7}, 67.8 \\
     & WPP$^{\mathrm{EA}}$ & \textcolor{blue}{\textbf{87.9}}, \textbf{88.7} & \textcolor{blue}{\textbf{77.2}}, \textbf{78.5} & \textcolor{blue}{\textbf{76.2}}, \textbf{77.4} & \textcolor{blue}{\underline{85.8}}, \textbf{86.4} & \textcolor{blue}{\underline{77.4}}, \underline{79.0} & \textcolor{blue}{\underline{75.5}}, \underline{76.8} & \textcolor{blue}{\textbf{89.2}}, \textbf{90.4} & \textcolor{blue}{\underline{78.1}}, \underline{80.4} & \textcolor{blue}{\textbf{78.2}}, \textbf{80.3} \\
     & \makecell{\textbf{ECI-PP}} & \textcolor{blue}{\underline{80.9}}, \underline{81.5} & \textcolor{blue}{\underline{75.7}}, \underline{77.1} & \textcolor{blue}{\underline{71.9}}, \underline{73.2} & \textcolor{blue}{\textbf{86.1}}, \textbf{86.4} & \textcolor{blue}{\textbf{79.8}}, \textbf{80.5} & \textcolor{blue}{\textbf{77.8}}, \textbf{78.3} & \textcolor{blue}{\underline{84.7}}, \underline{85.1} & \textcolor{blue}{\textbf{79.9}}, \textbf{80.7} & \textcolor{blue}{\underline{77.1}}, \underline{77.5} \\
    \midrule
    \multirow{3}{*}{REUTERS} & SP$^{\mathrm{EA}}$ & \textcolor{blue}{75.8}, 76.1 & \textcolor{blue}{41.8}, 45.7 & \textcolor{blue}{48.0}, 50.2 & \textcolor{blue}{72.0}, 77.0 & \textcolor{blue}{35.3}, 43.8 & \textcolor{blue}{40.9}, 49.8 & \textcolor{blue}{68.4}, 77.0 & \textcolor{blue}{30.6}, 43.1 & \textcolor{blue}{34.7}, 49.6 \\
     & WPP$^{\mathrm{EA}}$ & \textcolor{blue}{\underline{82.9}}, \underline{83.3} & \textcolor{blue}{\underline{55.4}}, \underline{58.3} & \textcolor{blue}{\underline{62.4}}, \underline{64.5} & \textcolor{blue}{\underline{81.4}}, \underline{83.7} & \textcolor{blue}{\underline{53.6}}, \underline{58.4} & \textcolor{blue}{\underline{59.5}}, \underline{64.4} & \textcolor{blue}{\underline{77.7}}, \underline{80.9} & \textcolor{blue}{\underline{50.3}}, \underline{56.0} & \textcolor{blue}{\underline{54.8}}, \underline{61.5} \\
     & \makecell{\textbf{ECI-PP}} & \textcolor{blue}{\textbf{85.0}}, \textbf{86.3} & \textcolor{blue}{\textbf{62.0}}, \textbf{65.0} & \textcolor{blue}{\textbf{68.3}}, \textbf{71.1} & \textcolor{blue}{\textbf{87.3}}, \textbf{88.5} & \textcolor{blue}{\textbf{63.8}}, \textbf{67.2} & \textcolor{blue}{\textbf{71.3}}, \textbf{74.3} & \textcolor{blue}{\textbf{88.3}}, \textbf{89.6} & \textcolor{blue}{\textbf{65.4}}, \textbf{68.7} & \textcolor{blue}{\textbf{72.9}}, \textbf{76.0} \\
    \midrule
    \multirow{3}{*}{STL10} & SP$^{\mathrm{EA}}$ & \textcolor{blue}{74.2}, 76.7 & \textcolor{blue}{66.5}, 68.1 & \textcolor{blue}{59.7}, 61.1 & \textcolor{blue}{73.0}, 75.4 & \textcolor{blue}{64.0}, 67.3 & \textcolor{blue}{57.4}, 60.1 & \textcolor{blue}{72.7}, 72.9 & \textcolor{blue}{62.5}, 66.2 & \textcolor{blue}{56.7}, 58.7 \\
     & WPP$^{\mathrm{EA}}$ & \textcolor{blue}{\underline{84.0}}, \underline{84.5} & \textcolor{blue}{\underline{72.4}}, \underline{73.8} & \textcolor{blue}{\underline{69.0}}, \underline{70.0} & \textcolor{blue}{\underline{85.2}}, \underline{85.6} & \textcolor{blue}{\underline{73.7}}, \underline{75.2} & \textcolor{blue}{\underline{71.3}}, \underline{72.0} & \textcolor{blue}{\underline{85.1}}, \underline{86.1} & \textcolor{blue}{\underline{73.3}}, \underline{75.4} & \textcolor{blue}{\underline{71.1}}, \underline{72.7} \\
     & \makecell{\textbf{ECI-PP}} & \textcolor{blue}{\textbf{85.7}}, \textbf{85.9} & \textcolor{blue}{\textbf{75.7}}, \textbf{76.3} & \textcolor{blue}{\textbf{72.2}}, \textbf{72.5} & \textcolor{blue}{\textbf{87.4}}, \textbf{87.5} & \textcolor{blue}{\textbf{77.5}}, \textbf{78.1} & \textcolor{blue}{\textbf{75.1}}, \textbf{75.4} & \textcolor{blue}{\textbf{87.9}}, \textbf{87.7} & \textcolor{blue}{\textbf{78.3}}, \textbf{78.5} & \textcolor{blue}{\textbf{75.9}}, \textbf{75.8} \\
    \midrule
    \multirow{3}{*}{RCV1-10} & SP$^{\mathrm{EA}}$ & \textcolor{blue}{\textbf{60.2}}, \textbf{60.9} & \textcolor{blue}{\underline{52.6}}, \underline{53.1} & \textcolor{blue}{\textbf{48.9}}, \textbf{49.5} & \textcolor{blue}{\textbf{68.2}}, \textbf{63.5} & \textcolor{blue}{\underline{54.8}}, \underline{53.8} & \textcolor{blue}{\textbf{55.7}}, \textbf{51.4} & \textcolor{blue}{\textbf{71.2}}, \textbf{67.7} & \textcolor{blue}{\underline{54.8}}, \underline{54.4} & \textcolor{blue}{\textbf{58.8}}, \textbf{55.9} \\
     & WPP$^{\mathrm{EA}}$ & \textcolor{blue}{52.5}, 47.6 & \textcolor{blue}{\underline{52.6}}, 52.5 & \textcolor{blue}{41.9}, 38.3 & \textcolor{blue}{54.6}, 52.6 & \textcolor{blue}{53.6}, 53.7 & \textcolor{blue}{45.0}, 42.8 & \textcolor{blue}{\underline{56.5}}, \underline{58.4} & \textcolor{blue}{53.5}, 54.2 & \textcolor{blue}{\underline{47.0}}, \underline{49.2} \\
     & \makecell{\textbf{ECI-PP}} & \textcolor{blue}{\underline{55.3}}, \underline{55.2} & \textcolor{blue}{\textbf{56.8}}, \textbf{56.9} & \textcolor{blue}{\underline{45.6}}, \underline{45.4} & \textcolor{blue}{\underline{55.2}}, \underline{55.0} & \textcolor{blue}{\textbf{58.4}}, \textbf{58.6} & \textcolor{blue}{\underline{46.8}}, \underline{46.8} & \textcolor{blue}{53.4}, 53.3 & \textcolor{blue}{\textbf{58.9}}, \textbf{59.0} & \textcolor{blue}{46.2}, 46.1 \\
    \bottomrule
    \end{tabular}
    \end{scriptsize}
    \end{center}
\end{table}

\cref{table:exp1_2_ensemble_vs_ecu} compares ECI-PP with SP$^{\mathrm{EA}}$ and WPP$^{\mathrm{EA}}$ under the {\tt multi3} setting, where SP$^{\mathrm{EA}}$ and WPP$^{\mathrm{EA}}$ denote our expert-aware ensemble extensions of SpherePair and Weighted ProbPair, respectively.
As described in Appendix~\ref{app:protocol}, these extensions instantiate multiple expert-weighted members for a baseline and aggregate their partitions through ensemble clustering.
The expert-weighting parameter $\alpha$ is selected on the reserved validation split, with the selected values reported in \cref{tab:expert_aware_alpha_selection}.
Thus, the comparison should be read as an auxiliary reference rather than a native baseline comparison: SP$^{\mathrm{EA}}$ and WPP$^{\mathrm{EA}}$ use expert identities, but they also introduce validation tuning and ensemble aggregation, making them not strictly comparable with the native ECI-PP protocol.

The extensions substantially strengthen the corresponding baselines, confirming that expert identities carry useful information in the multi-expert setting.
WPP$^{\mathrm{EA}}$ benefits clearly on datasets such as MNIST, FMNIST, and STL10, while SP$^{\mathrm{EA}}$ gains competitiveness on Reuters and RCV1-10.
However, these gains conflate two effects: using expert identities and ensembling multiple expert-weighted members.
For this reason, the results are best interpreted as stronger expert-aware references, not as fully symmetric replacements for the native baselines in \cref{table:exp1_9k_main}.

Even against these stronger references, ECI-PP remains the most reliable method overall.
It obtains the best result in $55/72$ test entries and ranks within the top two in $68/72$.
The advantage is clearest in NMI, where ECI-PP is best in $23/24$ cases, indicating that its expert-conditioned design most consistently improves grouping structure.
Dataset-wise, ECI-PP is best in all entries on CIFAR100-20, ImageNet10, Reuters, and STL10, and remains strongest in most CIFAR10 entries.
The main losses occur on FMNIST and MNIST, mostly on ACC/ARI, where WPP$^{\mathrm{EA}}$ is particularly strong, and on RCV1-10 ACC/ARI, where the angular geometry of SP$^{\mathrm{EA}}$ remains advantageous under severe imbalance.
Overall, expert-aware ensembling makes strong baselines substantially more competitive, but ECI-PP still provides a more direct and generally stronger way to use expert identities without relying on a validation-tuned ensemble wrapper.

\begin{figure}[p]
    \begin{center}
    \centerline{\includegraphics[width=\textwidth]{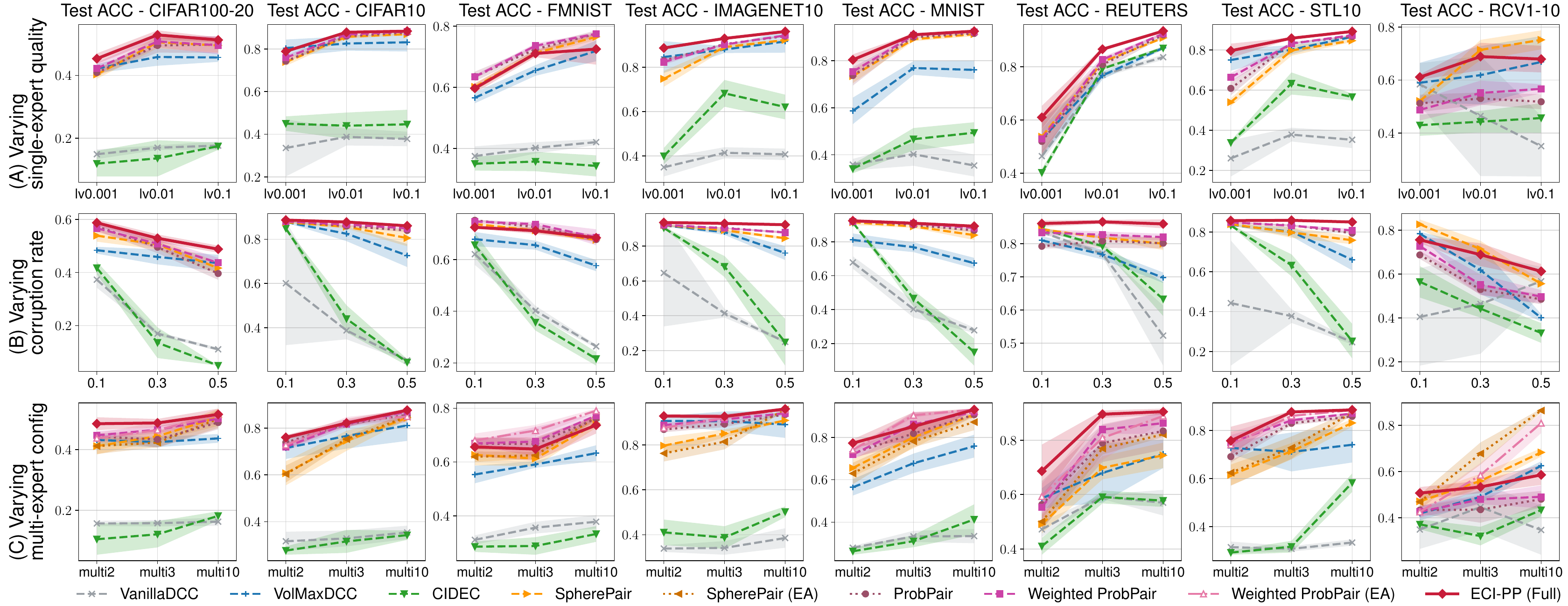}}
    \caption{
    Full robustness results in ACC (mean$\pm$std over 5 runs) across datasets under varying (A) expert quality, (B) corruption probability, and (C) multi-expert configuration, with $9$k constraints.
    The multi-expert row additionally includes expert-aware SpherePair and Weighted ProbPair extensions.
    }
    \label{fig:app_exp2_acc}
    \end{center}
\end{figure}

\begin{figure}[p]
    \begin{center}
    \centerline{\includegraphics[width=\textwidth]{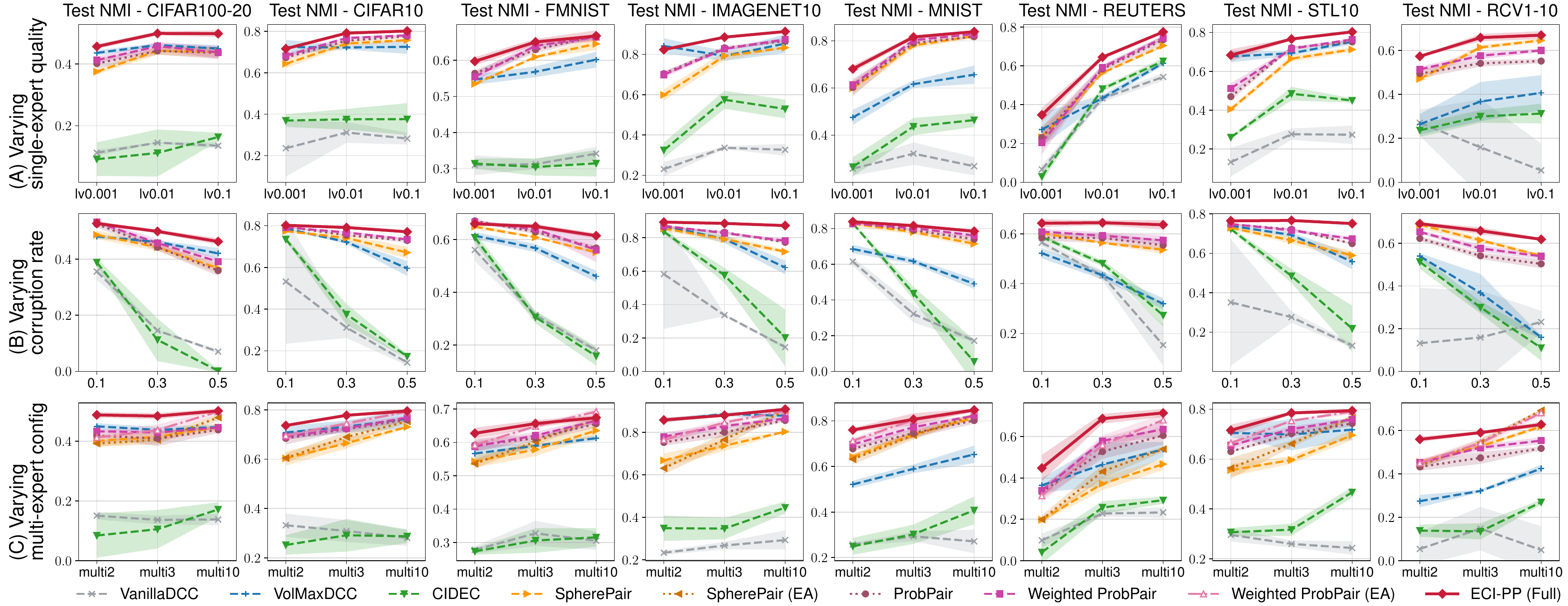}}
    \caption{
    Full robustness results in NMI (mean$\pm$std over 5 runs) across datasets under varying (A) expert quality, (B) corruption probability, and (C) multi-expert configuration, with $9$k constraints.
    The multi-expert row additionally includes expert-aware SpherePair and Weighted ProbPair extensions.
    }
    \label{fig:app_exp2_nmi}
    \end{center}
\end{figure}

\begin{figure}[p]
    \begin{center}
    \centerline{\includegraphics[width=\textwidth]{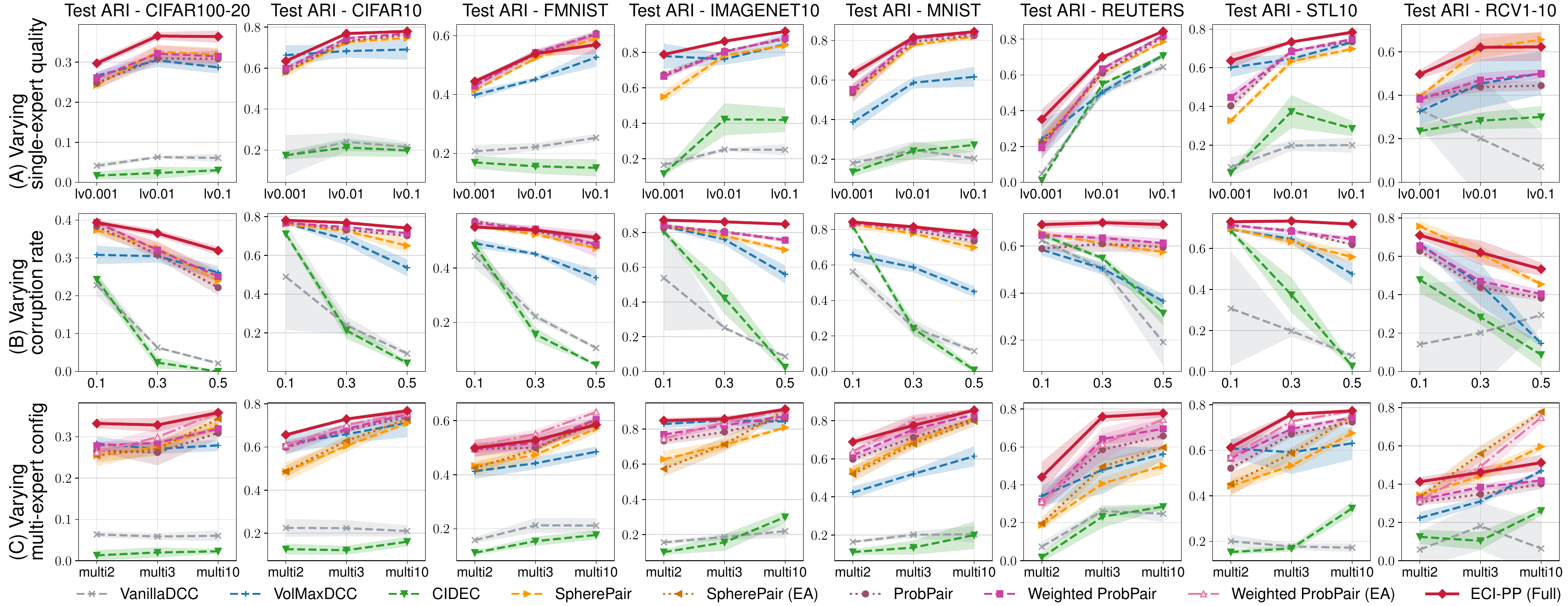}}
    \caption{
    Full robustness results in ARI (mean$\pm$std over 5 runs) across datasets under varying (A) expert quality, (B) corruption probability, and (C) multi-expert configuration, with $9$k constraints.
    The multi-expert row additionally includes expert-aware SpherePair and Weighted ProbPair extensions.
    }
    \label{fig:app_exp2_ari}
    \end{center}
\end{figure}

\subsection{Robustness across supervision conditions}
\label{app:robustness_full_results}

\cref{fig:app_exp2_acc,fig:app_exp2_nmi,fig:app_exp2_ari} report the full ACC, NMI, and ARI results corresponding to the robustness study in the main text.
All three sweeps use a fixed budget of $9$k constraints, and the central setting in each sweep, namely {\tt lv0.01} expert quality, corruption probability $0.3$, and {\tt multi3}, matches the default setting used in \cref{table:exp1_9k_main}.
Thus, these figures also cover the main-comparison results in \cref{table:exp1_9k_main}, while additionally reporting the standard deviations omitted from the table for space.
For the multi-expert sweep, we further include SP$^{\mathrm{EA}}$ and WPP$^{\mathrm{EA}}$, the expert-aware ensemble extensions introduced in Appendix~\ref{app:expert_aware_ensemble_results}, as stronger references that also exploit expert identities.

\paragraph{Effects of supervision factors.}
Across the three sweeps, the results largely follow the expected supervision-quality ordering.
In the single-expert sweep, stronger experts generally yield better clustering quality; in the corruption sweep, increasing corruption leads to clear degradation.
In the multi-expert sweep, moving from {\tt multi2} to {\tt multi10} also tends to help the stronger methods, because the protocol progressively reduces expert blind spots and increases average expert accuracy, as quantified in \cref{tab:multi_expert_avg_acc}.
Since the total constraint budget is fixed throughout, the performance trends mainly reflect changes in supervision quality and heterogeneity rather than in the amount of supervision.

\paragraph{Baseline behavior.}
The non-ECI baselines show a clear hierarchy as supervision becomes less reliable.
Plain end-to-end methods such as VanillaDCC and CIDEC often nearly collapse under low-quality or heterogeneous supervision, especially in NMI and ARI, suggesting that cluster-assignment training is vulnerable when pairwise supervision is inconsistent.
VolMaxDCC is usually more resilient, consistent with the benefit of explicit noise modeling.
SpherePair and the ProbPair-family baselines are substantially stronger than the plain end-to-end methods, reflecting the robustness of geometric representation learning under noisy supervision.
Nevertheless, corruption still causes visible degradation for the stronger baselines, particularly on FMNIST, ImageNet10, and STL10.

\paragraph{ECI-PP robustness.}
Across the supervision sweeps, ECI-PP better preserves clustering quality than the native baselines, especially as expert quality decreases or corruption increases.
This advantage is most consistent on NMI and ARI, suggesting that the correction-and-integration pipeline mainly improves the recovered grouping structure.
ACC is more competitive and contains more exceptions, especially on FMNIST, MNIST, and RCV1-10, where Weighted ProbPair, SP$^{\mathrm{EA}}$, or WPP$^{\mathrm{EA}}$ can sometimes match or exceed ECI-PP.
This is consistent with the main-comparison results (Appendix~\ref{app:full_budget_results}): matched-label accuracy can favor angular or ensemble-based baselines on some datasets, even when ECI-PP remains stronger or among the strongest on grouping-oriented metrics (NMI/ARI).
The expert-aware extensions further clarify the role of expert identities: SP$^{\mathrm{EA}}$ and WPP$^{\mathrm{EA}}$ often improve over their native counterparts in the multi-expert sweep, confirming that expert identity information is useful under heterogeneous supervision.
However, these extensions rely on validation-selected expert weighting and ensemble aggregation, whereas ECI-PP uses expert-conditioned correction within the learning pipeline.
Overall, the full results support the main-text conclusion that ECI-PP better preserves clustering quality when supervision is fallible, heterogeneous, or corrupted.





\subsection{Robustness to relation and corruption misspecification}
\label{app:misspecification_results}

We further examine whether ECI-PP's advantage depends on the default experimental constraint generator.
The three conditions defined in Appendix~\ref{app:constraint_generation} isolate different departures from that generator: \textit{Dot+Flip} changes only the corruption mechanism, \textit{JSD+Uniform} changes only the clean relation construction using normalized Jensen--Shannon similarity, and \textit{JSD+Flip} changes both.

As shown in \cref{tab:misspec,tab:misspec_part2}, ECI-PP provides the strongest overall performance across the full evaluation, ranking first in $60/72$ dataset--condition--metric entries and within the top two in $70/72$.
The advantage holds under each type of generator shift: ECI-PP ranks first in $17/24$ entries under \textit{Dot+Flip}, $22/24$ under \textit{JSD+Uniform}, and $21/24$ under \textit{JSD+Flip}.
Notably, even in the Flip-based settings, which are naturally aligned with VolMaxDCC's binary-noise modeling, ECI-PP still ranks first in $38/48$ entries and within the top two in all $48$.
Together, these results show that ECI-PP's advantage is not specific to the default relation construction or corruption mechanism used in the main experiments.

\begin{table}[htbp]
    \caption{Clustering performance (\%) under alternative relation and corruption constructions on CIFAR100-20, CIFAR10, FMNIST, and ImageNet10 in the pretrained {\tt multi3} setting with corruption probability $0.3$ and 9k constraints. Dot denotes the default inner-product relation, JSD denotes the normalized Jensen--Shannon relation, Uniform denotes the default uniform-replacement corruption mechanism, and Flip denotes complementary flipping. Results are mean$\pm$std over 5 runs. Best results are in \textbf{bold}, and second-best are \underline{underlined}.}
    \label{tab:misspec}
    \begin{center}
    \begin{scriptsize}
    \renewcommand{\arraystretch}{1.0}
    \setlength{\tabcolsep}{3.0pt}
    \setlength{\aboverulesep}{0.8pt}
    \setlength{\belowrulesep}{0.8pt}
    \setlength{\extrarowheight}{0pt}
\begin{tabular}{lllccccccc}
        \toprule
        & & & \multicolumn{4}{c}{\textbf{External baselines}} & \multicolumn{3}{c}{\textbf{Our progressive models}} \\
        \cmidrule(lr){4-7}\cmidrule(lr){8-10}
        Dataset & Condition & Metric & \makecell{Vanilla-\\DCC} & \makecell{VolMax-\\DCC} & CIDEC & SpherePair & ProbPair & \makecell{Weighted\\ProbPair} & \makecell{\textbf{ECI-PP}\\(Full)} \\
        \midrule
        \multirow{9}{*}{\raisebox{-4ex}{CIFAR100-20}} & \multirow{3}{*}{Dot+Flip} & ACC & 5.0$\pm$0.0 & \textbf{43.8$\pm$2.6} & 5.0$\pm$0.0 & 22.9$\pm$2.2 & 26.3$\pm$2.5 & 28.9$\pm$2.2 & \underline{41.8$\pm$2.3} \\
          &  & NMI & 0.0$\pm$0.0 & \textbf{44.6$\pm$1.6} & 0.2$\pm$0.3 & 17.5$\pm$1.5 & 22.0$\pm$1.4 & 24.7$\pm$1.4 & \underline{41.0$\pm$1.0} \\
          &  & ARI & 0.0$\pm$0.0 & \textbf{27.9$\pm$1.8} & 0.0$\pm$0.0 & 8.1$\pm$1.4 & 11.0$\pm$1.5 & 12.8$\pm$1.6 & \underline{26.2$\pm$1.2} \\
        \cmidrule(lr){2-10}
          & \multirow{3}{*}{JSD+Uniform} & ACC & 16.0$\pm$0.3 & 43.7$\pm$2.2 & 13.8$\pm$1.5 & 45.7$\pm$1.6 & 47.0$\pm$1.0 & \underline{47.5$\pm$0.7} & \textbf{50.3$\pm$1.2} \\
          &  & NMI & 15.3$\pm$0.7 & 44.9$\pm$1.3 & 12.8$\pm$1.9 & 43.9$\pm$0.9 & 44.5$\pm$0.8 & \underline{45.5$\pm$0.6} & \textbf{50.2$\pm$0.5} \\
          &  & ARI & 6.6$\pm$0.2 & 28.1$\pm$1.5 & 2.1$\pm$0.9 & 29.3$\pm$1.1 & 29.2$\pm$0.7 & \underline{30.1$\pm$0.6} & \textbf{35.3$\pm$0.6} \\
        \cmidrule(lr){2-10}
          & \multirow{3}{*}{JSD+Flip} & ACC & 5.0$\pm$0.0 & \textbf{43.8$\pm$1.3} & 7.1$\pm$1.9 & 23.3$\pm$0.7 & 27.2$\pm$1.4 & 29.1$\pm$1.0 & \underline{41.8$\pm$1.7} \\
          &  & NMI & 0.0$\pm$0.0 & \textbf{44.7$\pm$1.0} & 3.2$\pm$2.6 & 17.5$\pm$0.6 & 22.6$\pm$1.1 & 24.9$\pm$0.9 & \underline{41.5$\pm$0.9} \\
          &  & ARI & 0.0$\pm$0.0 & \textbf{27.8$\pm$1.0} & 0.2$\pm$0.2 & 8.1$\pm$0.6 & 11.1$\pm$1.0 & 13.1$\pm$0.9 & \underline{26.2$\pm$1.2} \\
        \midrule
        \multirow{9}{*}{\raisebox{-4ex}{CIFAR10}} & \multirow{3}{*}{Dot+Flip} & ACC & 10.0$\pm$0.0 & \underline{76.2$\pm$5.9} & 14.4$\pm$4.1 & 48.6$\pm$3.6 & 65.9$\pm$4.0 & 68.7$\pm$2.7 & \textbf{79.4$\pm$3.5} \\
          &  & NMI & 0.0$\pm$0.0 & \underline{67.7$\pm$5.2} & 4.8$\pm$4.4 & 35.9$\pm$3.5 & 48.0$\pm$3.0 & 50.4$\pm$3.1 & \textbf{70.8$\pm$2.1} \\
          &  & ARI & 0.0$\pm$0.0 & \underline{60.8$\pm$6.8} & 2.2$\pm$2.0 & 27.7$\pm$3.8 & 43.1$\pm$4.2 & 46.0$\pm$3.6 & \textbf{65.3$\pm$3.3} \\
        \cmidrule(lr){2-10}
          & \multirow{3}{*}{JSD+Uniform} & ACC & 31.7$\pm$3.4 & 75.0$\pm$2.8 & 27.1$\pm$3.3 & 79.2$\pm$1.4 & 82.7$\pm$1.9 & \underline{83.7$\pm$0.7} & \textbf{85.7$\pm$1.0} \\
          &  & NMI & 30.7$\pm$5.1 & 72.3$\pm$2.5 & 24.7$\pm$4.5 & 69.0$\pm$0.5 & 73.4$\pm$1.2 & \underline{74.2$\pm$0.6} & \textbf{77.4$\pm$0.8} \\
          &  & ARI & 21.4$\pm$4.1 & 64.4$\pm$4.0 & 10.7$\pm$1.5 & 64.1$\pm$0.9 & 69.2$\pm$1.8 & \underline{70.4$\pm$0.9} & \textbf{73.7$\pm$1.3} \\
        \cmidrule(lr){2-10}
          & \multirow{3}{*}{JSD+Flip} & ACC & 11.7$\pm$3.7 & \underline{73.4$\pm$6.4} & 14.4$\pm$6.2 & 51.2$\pm$4.0 & 66.8$\pm$3.5 & 64.4$\pm$4.4 & \textbf{78.9$\pm$2.8} \\
          &  & NMI & 1.7$\pm$3.7 & \underline{67.7$\pm$4.1} & 5.8$\pm$8.5 & 37.8$\pm$3.9 & 49.5$\pm$2.8 & 48.8$\pm$4.1 & \textbf{70.4$\pm$2.0} \\
          &  & ARI & 0.8$\pm$1.8 & \underline{59.8$\pm$6.7} & 1.3$\pm$1.9 & 29.9$\pm$5.1 & 45.0$\pm$3.4 & 43.1$\pm$4.2 & \textbf{63.8$\pm$3.0} \\
        \midrule
        \multirow{9}{*}{\raisebox{-4ex}{FMNIST}} & \multirow{3}{*}{Dot+Flip} & ACC & 10.0$\pm$0.0 & \underline{53.6$\pm$4.0} & 16.7$\pm$4.1 & 33.9$\pm$3.4 & 52.7$\pm$3.2 & 50.7$\pm$3.5 & \textbf{63.3$\pm$5.2} \\
          &  & NMI & 0.0$\pm$0.0 & \underline{52.6$\pm$2.4} & 8.4$\pm$5.3 & 21.6$\pm$2.3 & 40.6$\pm$1.5 & 39.8$\pm$2.8 & \textbf{59.0$\pm$2.1} \\
          &  & ARI & 0.0$\pm$0.0 & \underline{37.2$\pm$4.0} & 1.8$\pm$1.5 & 13.8$\pm$2.6 & 31.3$\pm$1.5 & 30.8$\pm$3.0 & \textbf{47.2$\pm$2.8} \\
        \cmidrule(lr){2-10}
          & \multirow{3}{*}{JSD+Uniform} & ACC & 33.4$\pm$0.9 & 59.2$\pm$2.8 & 28.3$\pm$2.0 & 61.6$\pm$2.0 & \underline{70.2$\pm$1.9} & \textbf{71.6$\pm$4.4} & 64.4$\pm$3.9 \\
          &  & NMI & 32.2$\pm$1.4 & 59.5$\pm$1.4 & 28.8$\pm$2.2 & 59.2$\pm$0.9 & 63.4$\pm$1.1 & \underline{64.0$\pm$2.3} & \textbf{65.8$\pm$1.1} \\
          &  & ARI & 20.5$\pm$0.9 & 45.8$\pm$2.2 & 14.6$\pm$1.0 & 48.8$\pm$1.3 & \underline{54.0$\pm$2.1} & \textbf{54.8$\pm$3.9} & 53.1$\pm$2.2 \\
        \cmidrule(lr){2-10}
          & \multirow{3}{*}{JSD+Flip} & ACC & 14.0$\pm$5.4 & \underline{55.8$\pm$1.3} & 18.9$\pm$0.9 & 33.2$\pm$2.9 & 53.7$\pm$4.1 & 53.6$\pm$2.5 & \textbf{63.7$\pm$1.1} \\
          &  & NMI & 5.1$\pm$6.1 & \underline{55.1$\pm$1.3} & 12.0$\pm$1.8 & 22.5$\pm$0.9 & 41.2$\pm$2.5 & 41.9$\pm$2.2 & \textbf{59.7$\pm$1.4} \\
          &  & ARI & 2.1$\pm$2.8 & \underline{40.1$\pm$1.2} & 3.9$\pm$2.1 & 14.5$\pm$0.6 & 31.9$\pm$2.5 & 32.7$\pm$1.8 & \textbf{47.0$\pm$1.7} \\
        \midrule
        \multirow{9}{*}{\raisebox{-4ex}{ImageNet10}} & \multirow{3}{*}{Dot+Flip} & ACC & 21.1$\pm$0.6 & \textbf{89.0$\pm$3.9} & 10.0$\pm$0.0 & 41.1$\pm$2.0 & 67.7$\pm$5.8 & 69.1$\pm$7.3 & \underline{88.9$\pm$1.7} \\
          &  & NMI & 10.2$\pm$0.7 & \textbf{84.5$\pm$3.0} & 0.0$\pm$0.0 & 26.1$\pm$1.8 & 49.7$\pm$4.3 & 52.1$\pm$3.6 & \underline{81.4$\pm$1.0} \\
          &  & ARI & 4.9$\pm$0.6 & \textbf{80.4$\pm$5.0} & 0.0$\pm$0.0 & 17.8$\pm$1.4 & 45.6$\pm$5.6 & 47.3$\pm$5.8 & \underline{78.7$\pm$1.6} \\
        \cmidrule(lr){2-10}
          & \multirow{3}{*}{JSD+Uniform} & ACC & 32.6$\pm$1.7 & 91.1$\pm$4.9 & 28.8$\pm$2.9 & 87.9$\pm$1.7 & 90.5$\pm$1.3 & \underline{91.8$\pm$0.7} & \textbf{93.4$\pm$1.4} \\
          &  & NMI & 27.4$\pm$1.6 & \underline{88.5$\pm$2.1} & 27.3$\pm$3.1 & 78.9$\pm$2.0 & 83.0$\pm$1.7 & 84.8$\pm$0.7 & \textbf{88.7$\pm$0.8} \\
          &  & ARI & 18.2$\pm$1.4 & \underline{85.4$\pm$3.8} & 9.6$\pm$2.1 & 76.6$\pm$2.3 & 81.2$\pm$2.0 & 83.4$\pm$1.0 & \textbf{86.9$\pm$2.1} \\
        \cmidrule(lr){2-10}
          & \multirow{3}{*}{JSD+Flip} & ACC & 17.0$\pm$4.0 & \underline{83.3$\pm$6.9} & 10.0$\pm$0.0 & 48.0$\pm$3.7 & 71.9$\pm$4.4 & 70.9$\pm$5.8 & \textbf{90.2$\pm$0.7} \\
          &  & NMI & 7.2$\pm$4.1 & \underline{81.5$\pm$5.7} & 0.0$\pm$0.0 & 32.2$\pm$3.2 & 55.8$\pm$2.5 & 53.8$\pm$4.1 & \textbf{82.3$\pm$1.0} \\
          &  & ARI & 3.2$\pm$1.8 & \underline{75.7$\pm$7.9} & 0.0$\pm$0.0 & 24.3$\pm$3.6 & 52.1$\pm$3.4 & 49.7$\pm$5.5 & \textbf{80.3$\pm$1.0} \\
        \bottomrule
    \end{tabular}
    \end{scriptsize}
    \end{center}
\end{table}

\begin{table}[htbp]
    \caption{Clustering performance (\%) under alternative relation and corruption constructions on MNIST, Reuters, STL10, and RCV1-10 in the pretrained {\tt multi3} setting with corruption probability $0.3$ and 9k constraints. Dot denotes the default inner-product relation, JSD denotes the normalized Jensen--Shannon relation, Uniform denotes the default uniform-replacement corruption mechanism, and Flip denotes complementary flipping. Results are mean$\pm$std over 5 runs. Best results are in \textbf{bold}, and second-best are \underline{underlined}.}
    \label{tab:misspec_part2}
    \begin{center}
    \begin{scriptsize}
    \renewcommand{\arraystretch}{1.0}
    \setlength{\tabcolsep}{3.0pt}
    \setlength{\aboverulesep}{0.8pt}
    \setlength{\belowrulesep}{0.8pt}
    \setlength{\extrarowheight}{0pt}
\begin{tabular}{lllccccccc}
        \toprule
        & & & \multicolumn{4}{c}{\textbf{External baselines}} & \multicolumn{3}{c}{\textbf{Our progressive models}} \\
        \cmidrule(lr){4-7}\cmidrule(lr){8-10}
        Dataset & Condition & Metric & \makecell{Vanilla-\\DCC} & \makecell{VolMax-\\DCC} & CIDEC & SpherePair & ProbPair & \makecell{Weighted\\ProbPair} & \makecell{\textbf{ECI-PP}\\(Full)} \\
        \midrule
        \multirow{9}{*}{\raisebox{-4ex}{MNIST}} & \multirow{3}{*}{Dot+Flip} & ACC & 11.3$\pm$0.0 & 53.3$\pm$5.3 & 14.4$\pm$6.8 & 53.0$\pm$6.8 & \underline{65.1$\pm$4.1} & 65.0$\pm$7.9 & \textbf{77.3$\pm$2.8} \\
          &  & NMI & 0.0$\pm$0.0 & 42.9$\pm$2.6 & 5.1$\pm$11.4 & 42.6$\pm$4.6 & 51.2$\pm$4.7 & \underline{51.4$\pm$6.9} & \textbf{70.7$\pm$1.5} \\
          &  & ARI & 0.0$\pm$0.0 & 34.3$\pm$4.3 & 0.5$\pm$1.2 & 32.3$\pm$4.8 & 45.5$\pm$4.9 & \underline{45.5$\pm$7.8} & \textbf{65.1$\pm$1.9} \\
        \cmidrule(lr){2-10}
          & \multirow{3}{*}{JSD+Uniform} & ACC & 30.2$\pm$3.0 & 68.3$\pm$3.0 & 31.6$\pm$1.1 & 77.2$\pm$1.8 & 86.4$\pm$3.0 & \underline{89.4$\pm$1.6} & \textbf{90.8$\pm$2.3} \\
          &  & NMI & 25.8$\pm$3.2 & 59.9$\pm$1.8 & 30.0$\pm$2.7 & 75.4$\pm$1.3 & 76.8$\pm$1.4 & \underline{79.4$\pm$1.5} & \textbf{82.2$\pm$1.7} \\
          &  & ARI & 17.1$\pm$2.3 & 53.1$\pm$2.9 & 12.6$\pm$1.2 & 69.6$\pm$2.0 & 75.0$\pm$2.8 & \underline{78.7$\pm$2.5} & \textbf{81.7$\pm$3.1} \\
        \cmidrule(lr){2-10}
          & \multirow{3}{*}{JSD+Flip} & ACC & 13.1$\pm$4.0 & 56.0$\pm$2.2 & 11.3$\pm$0.0 & 52.0$\pm$4.5 & 67.1$\pm$5.9 & \underline{68.4$\pm$4.5} & \textbf{79.3$\pm$4.4} \\
          &  & NMI & 1.9$\pm$4.2 & 46.1$\pm$1.3 & 0.0$\pm$0.0 & 43.9$\pm$3.0 & 54.1$\pm$2.4 & \underline{54.4$\pm$2.3} & \textbf{72.4$\pm$1.8} \\
          &  & ARI & 0.9$\pm$2.0 & 37.7$\pm$2.0 & 0.0$\pm$0.0 & 33.1$\pm$3.3 & 48.6$\pm$3.6 & \underline{48.9$\pm$3.7} & \textbf{67.5$\pm$3.1} \\
        \midrule
        \multirow{9}{*}{\raisebox{-4ex}{Reuters}} & \multirow{3}{*}{Dot+Flip} & ACC & 42.4$\pm$4.6 & \underline{59.7$\pm$7.9} & 38.7$\pm$0.0 & 44.6$\pm$3.7 & 53.3$\pm$5.5 & 57.0$\pm$4.8 & \textbf{83.8$\pm$2.0} \\
          &  & NMI & 2.9$\pm$3.6 & \underline{35.7$\pm$8.8} & 0.5$\pm$0.7 & 11.4$\pm$1.8 & 21.4$\pm$3.5 & 24.9$\pm$3.8 & \textbf{58.7$\pm$3.1} \\
          &  & ARI & 3.3$\pm$4.1 & \underline{33.8$\pm$11.9} & 0.0$\pm$0.0 & 11.4$\pm$2.1 & 22.0$\pm$4.2 & 25.3$\pm$4.6 & \textbf{65.3$\pm$3.6} \\
        \cmidrule(lr){2-10}
          & \multirow{3}{*}{JSD+Uniform} & ACC & 58.4$\pm$0.9 & 68.6$\pm$6.1 & 47.7$\pm$5.5 & 65.5$\pm$4.9 & 78.2$\pm$8.6 & \underline{80.9$\pm$5.1} & \textbf{90.1$\pm$1.5} \\
          &  & NMI & 21.1$\pm$1.4 & 48.5$\pm$3.2 & 11.5$\pm$6.1 & 35.1$\pm$4.5 & 51.3$\pm$5.8 & \underline{56.5$\pm$3.5} & \textbf{70.0$\pm$3.2} \\
          &  & ARI & 24.1$\pm$1.3 & 49.3$\pm$6.2 & 7.8$\pm$4.9 & 37.2$\pm$5.4 & 57.1$\pm$8.0 & \underline{61.3$\pm$5.9} & \textbf{77.2$\pm$3.1} \\
        \cmidrule(lr){2-10}
          & \multirow{3}{*}{JSD+Flip} & ACC & 38.4$\pm$0.5 & \underline{66.8$\pm$8.7} & 38.7$\pm$0.0 & 41.4$\pm$3.6 & 54.7$\pm$7.6 & 58.5$\pm$4.8 & \textbf{83.0$\pm$4.1} \\
          &  & NMI & 0.3$\pm$0.5 & \underline{42.0$\pm$7.8} & 0.0$\pm$0.0 & 9.8$\pm$3.1 & 22.5$\pm$4.9 & 27.6$\pm$3.6 & \textbf{58.0$\pm$4.2} \\
          &  & ARI & 0.4$\pm$0.7 & \underline{41.4$\pm$9.9} & 0.0$\pm$0.0 & 9.0$\pm$2.5 & 22.3$\pm$6.4 & 28.7$\pm$4.3 & \textbf{64.6$\pm$4.8} \\
        \midrule
        \multirow{9}{*}{\raisebox{-4ex}{STL10}} & \multirow{3}{*}{Dot+Flip} & ACC & 19.7$\pm$2.4 & \underline{72.8$\pm$7.9} & 12.4$\pm$5.3 & 32.2$\pm$1.3 & 49.6$\pm$5.8 & 55.5$\pm$2.0 & \textbf{79.6$\pm$3.4} \\
          &  & NMI & 10.1$\pm$2.3 & \textbf{69.4$\pm$3.1} & 3.8$\pm$7.7 & 18.7$\pm$1.5 & 34.3$\pm$4.0 & 37.3$\pm$2.6 & \underline{68.1$\pm$2.1} \\
          &  & ARI & 4.8$\pm$1.9 & \underline{59.0$\pm$6.5} & 0.2$\pm$0.4 & 11.4$\pm$0.8 & 26.7$\pm$4.4 & 31.1$\pm$3.1 & \textbf{63.0$\pm$4.1} \\
        \cmidrule(lr){2-10}
          & \multirow{3}{*}{JSD+Uniform} & ACC & 26.1$\pm$0.9 & 70.4$\pm$4.1 & 27.4$\pm$1.4 & 76.0$\pm$3.7 & 82.0$\pm$1.6 & \underline{83.4$\pm$1.9} & \textbf{86.9$\pm$0.9} \\
          &  & NMI & 22.5$\pm$1.9 & 70.2$\pm$1.8 & 28.0$\pm$1.5 & 64.9$\pm$1.7 & 70.8$\pm$1.5 & \underline{72.2$\pm$1.7} & \textbf{77.8$\pm$1.2} \\
          &  & ARI & 14.1$\pm$1.2 & 58.6$\pm$2.5 & 13.8$\pm$0.7 & 59.2$\pm$2.5 & 66.2$\pm$2.2 & \underline{68.5$\pm$2.2} & \textbf{74.5$\pm$1.8} \\
        \cmidrule(lr){2-10}
          & \multirow{3}{*}{JSD+Flip} & ACC & 17.8$\pm$1.5 & \underline{69.1$\pm$7.1} & 10.0$\pm$0.0 & 37.2$\pm$2.9 & 60.8$\pm$4.7 & 58.1$\pm$4.7 & \textbf{80.6$\pm$2.0} \\
          &  & NMI & 9.7$\pm$1.8 & \underline{69.3$\pm$3.5} & 0.0$\pm$0.0 & 23.7$\pm$2.5 & 43.2$\pm$3.0 & 40.6$\pm$3.7 & \textbf{69.6$\pm$1.7} \\
          &  & ARI & 3.9$\pm$0.9 & \underline{57.6$\pm$6.2} & 0.0$\pm$0.0 & 15.9$\pm$2.8 & 37.0$\pm$3.9 & 34.0$\pm$4.5 & \textbf{64.6$\pm$2.6} \\
        \midrule
        \multirow{9}{*}{\raisebox{-4ex}{RCV1-10}} & \multirow{3}{*}{Dot+Flip} & ACC & 29.9$\pm$0.1 & 30.4$\pm$2.5 & 31.6$\pm$2.6 & \underline{46.7$\pm$4.6} & 38.6$\pm$2.3 & 39.0$\pm$1.1 & \textbf{49.1$\pm$1.7} \\
          &  & NMI & 0.3$\pm$0.4 & 16.7$\pm$1.1 & 3.2$\pm$3.9 & \underline{36.6$\pm$2.4} & 30.8$\pm$1.3 & 34.1$\pm$1.5 & \textbf{53.3$\pm$1.4} \\
          &  & ARI & 0.0$\pm$0.0 & 8.8$\pm$1.2 & 2.5$\pm$3.4 & \underline{31.2$\pm$5.6} & 22.9$\pm$1.0 & 25.4$\pm$1.2 & \textbf{39.5$\pm$2.4} \\
        \cmidrule(lr){2-10}
          & \multirow{3}{*}{JSD+Uniform} & ACC & 29.9$\pm$0.1 & 45.3$\pm$3.8 & 37.6$\pm$4.1 & \underline{55.6$\pm$2.7} & 46.0$\pm$3.7 & 48.6$\pm$2.9 & \textbf{56.4$\pm$3.8} \\
          &  & NMI & 0.8$\pm$1.4 & 29.0$\pm$3.5 & 13.3$\pm$1.2 & \underline{53.8$\pm$0.8} & 47.6$\pm$0.7 & 51.4$\pm$1.3 & \textbf{61.3$\pm$1.5} \\
          &  & ARI & -0.1$\pm$0.1 & 25.0$\pm$6.3 & 10.2$\pm$1.9 & \underline{44.8$\pm$1.4} & 35.7$\pm$2.0 & 38.4$\pm$2.0 & \textbf{48.3$\pm$3.8} \\
        \cmidrule(lr){2-10}
          & \multirow{3}{*}{JSD+Flip} & ACC & 29.9$\pm$0.0 & 30.9$\pm$2.3 & 30.9$\pm$2.3 & \underline{44.9$\pm$4.8} & 37.9$\pm$2.9 & 39.5$\pm$2.3 & \textbf{49.8$\pm$1.9} \\
          &  & NMI & 0.0$\pm$0.0 & 17.4$\pm$1.4 & 1.7$\pm$3.4 & \underline{37.2$\pm$1.7} & 31.5$\pm$1.8 & 35.3$\pm$1.8 & \textbf{53.8$\pm$1.7} \\
          &  & ARI & 0.0$\pm$0.0 & 9.2$\pm$0.5 & 1.2$\pm$2.6 & \underline{29.9$\pm$4.4} & 23.3$\pm$1.6 & 25.3$\pm$1.6 & \textbf{39.9$\pm$2.6} \\
        \bottomrule
    \end{tabular}
    \end{scriptsize}
    \end{center}
\end{table}

\subsection{Held-out diagnostics}
\label{app:heldout_diagnostics_full}

\paragraph{Diagnostic protocol.}
The held-out diagnostics complement the clustering results by tracking the training dynamics of ECI-PP's internal quantities.
Following Appendix~\ref{app:protocol}, we construct a sample-disjoint diagnostic constraint set $\mathcal{C}^{\mathrm{diag}}$ from the test split, used only for diagnosis and never for optimization, early stopping, or hyperparameter selection.
Benchmark labels are used only to define an external hard co-membership reference, rather than the canonical aleatoric relation $R^\star$.
All diagnostics use {\tt multi3} supervision with limited $3$k training constraints in total ($1$k per expert) and corruption rate $0.3$.
To reduce the influence of a strong initial representation or a long warm-up phase, we use random initialization without unsupervised pretraining, give the Estimators and Integrator only one warm-up epoch, and then track the next $500$ training iterations, each corresponding to one epoch.
We report mean$\pm$std over $5$ runs, recording at each iteration the Estimator out-of-fold relation $\hat{y}^{\mathrm{oof}}_i$, the corrected relation $y^{\mathrm{cor}}_i$, and the Integrator relation $y_i^{\mathrm{int}}$ on $\mathcal{C}^{\mathrm{diag}}$.

\begin{figure}[p]
    \centering
    \includegraphics[width=0.8\textwidth]{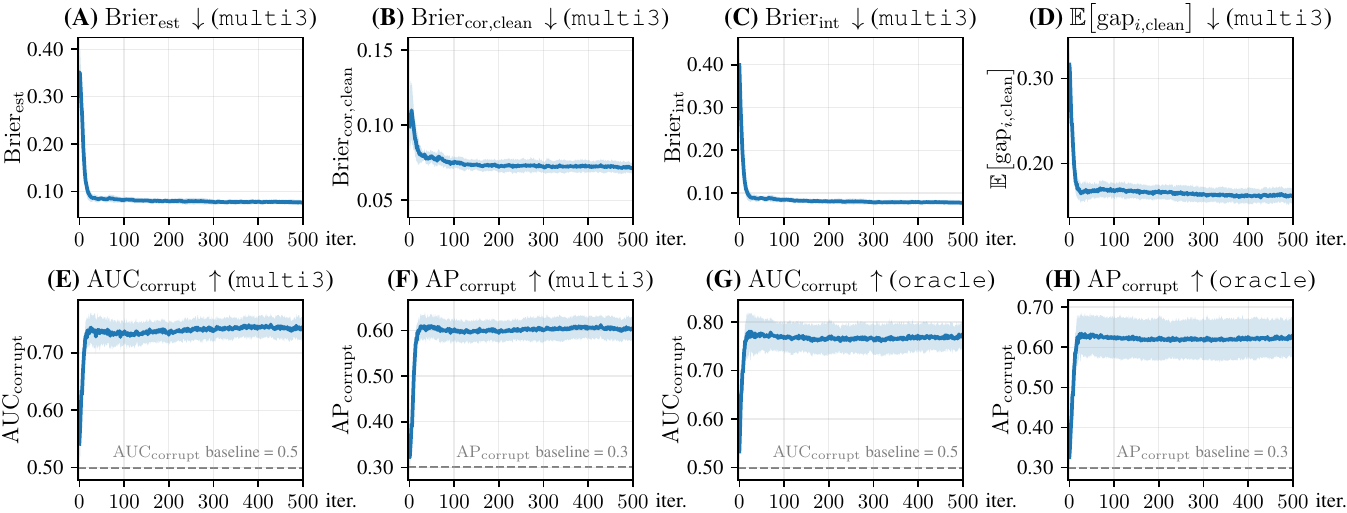}
    \caption{
    CIFAR100-20 held-out diagnostic evolution under {\tt multi3} with corruption rate $0.3$ (mean$\pm$std over 5 runs).
    Same panel definitions as in \cref{fig:diag_mnist}.
    }
    \label{fig:app_diag_cifar100}
\end{figure}

\begin{figure}[p]
    \centering
    \includegraphics[width=0.8\textwidth]{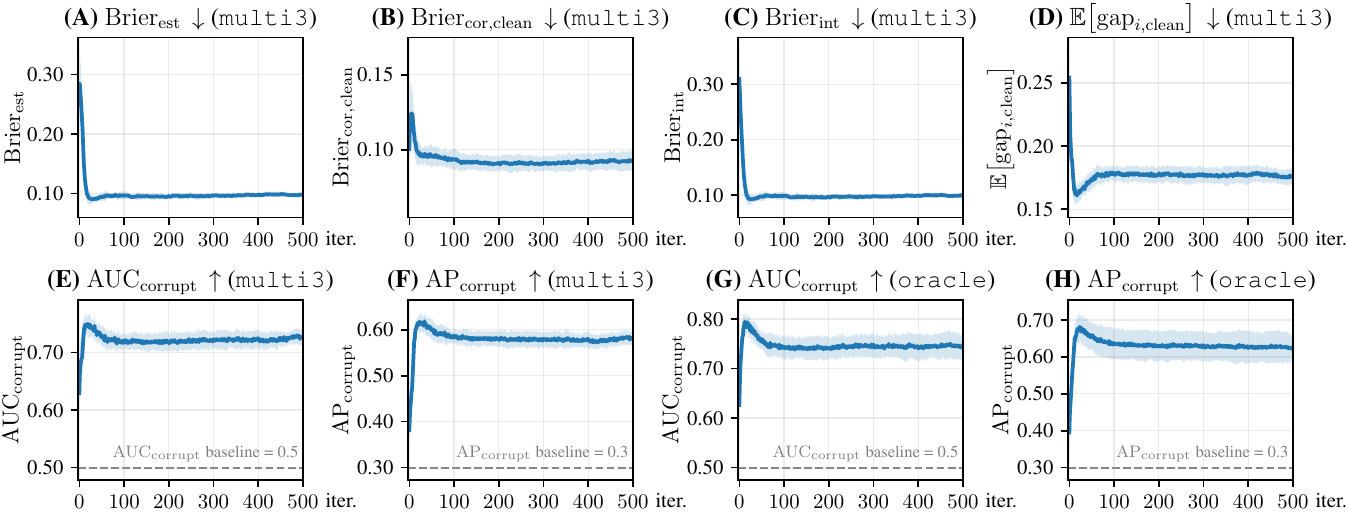}
    \caption{
    CIFAR10 held-out diagnostic evolution under {\tt multi3} with corruption rate $0.3$ (mean$\pm$std over 5 runs).
    Same panel definitions as in \cref{fig:diag_mnist}.
    }
    \label{fig:app_diag_cifar10}
\end{figure}

\begin{figure}[p]
    \centering
    \includegraphics[width=0.8\textwidth]{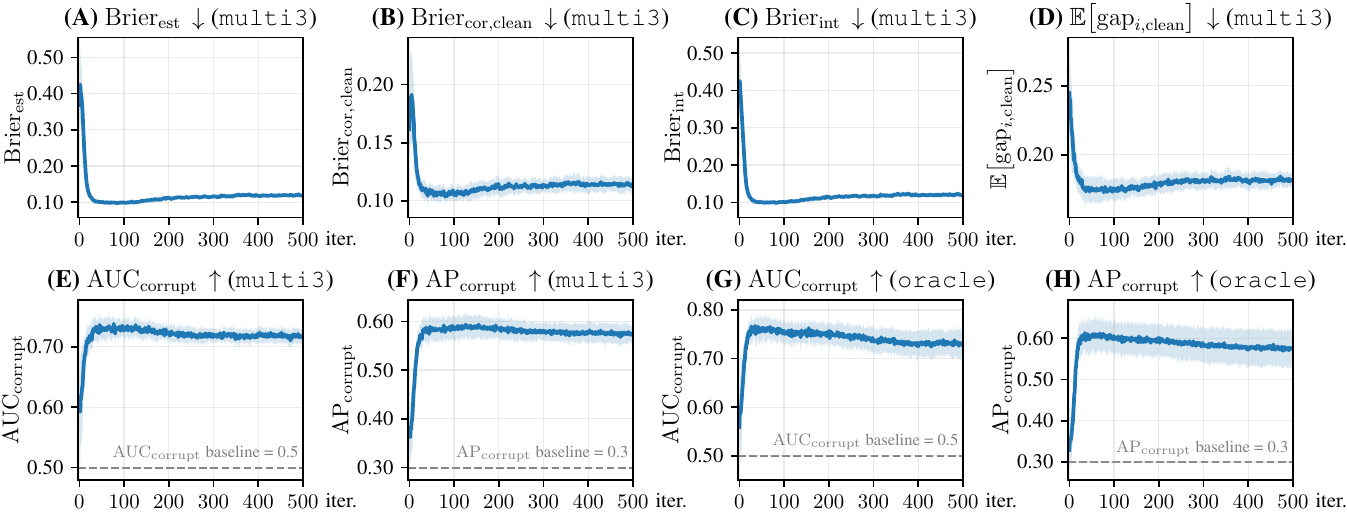}
    \caption{
    FMNIST held-out diagnostic evolution under {\tt multi3} with corruption rate $0.3$ (mean$\pm$std over 5 runs).
    Same panel definitions as in \cref{fig:diag_mnist}.
    }
    \label{fig:app_diag_fmnist}
\end{figure}

\begin{figure}[p]
    \centering
    \includegraphics[width=0.8\textwidth]{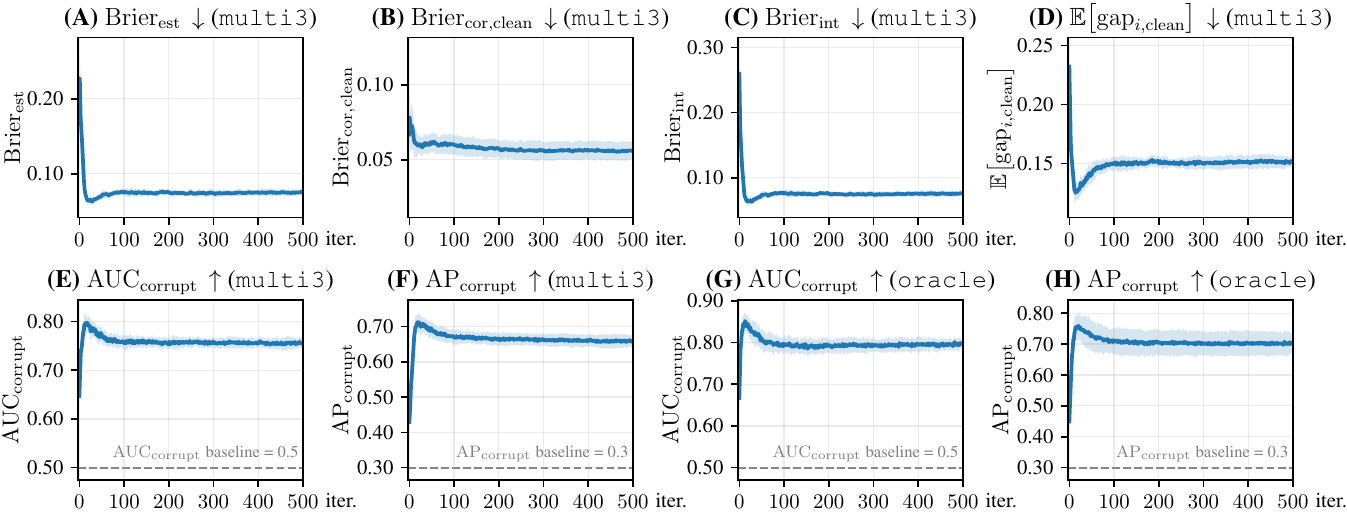}
    \caption{
    ImageNet10 held-out diagnostic evolution under {\tt multi3} with corruption rate $0.3$ (mean$\pm$std over 5 runs).
    Same panel definitions as in \cref{fig:diag_mnist}.
    }
    \label{fig:app_diag_imagenet10}
\end{figure}

\begin{figure}[p]
    \centering
    \includegraphics[width=0.8\textwidth]{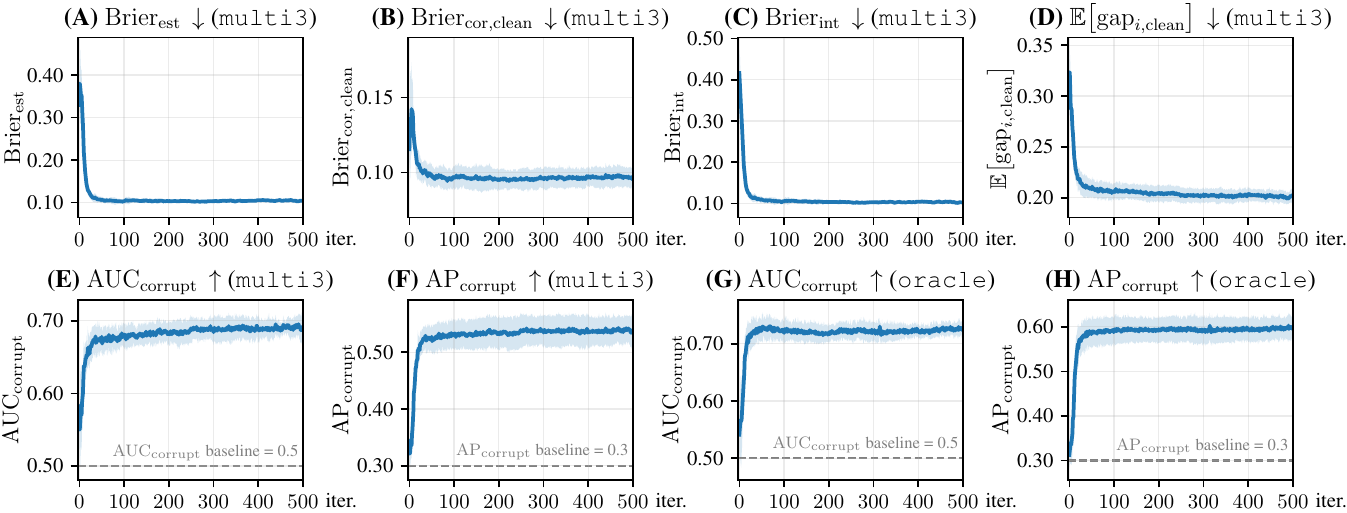}
    \caption{
    MNIST held-out diagnostic evolution under {\tt multi3} with corruption rate $0.3$ (mean$\pm$std over 5 runs).
    Same panel definitions as in \cref{fig:diag_mnist}.
    }
    \label{fig:app_diag_mnist}
\end{figure}

\begin{figure}[p]
    \centering
    \includegraphics[width=0.8\textwidth]{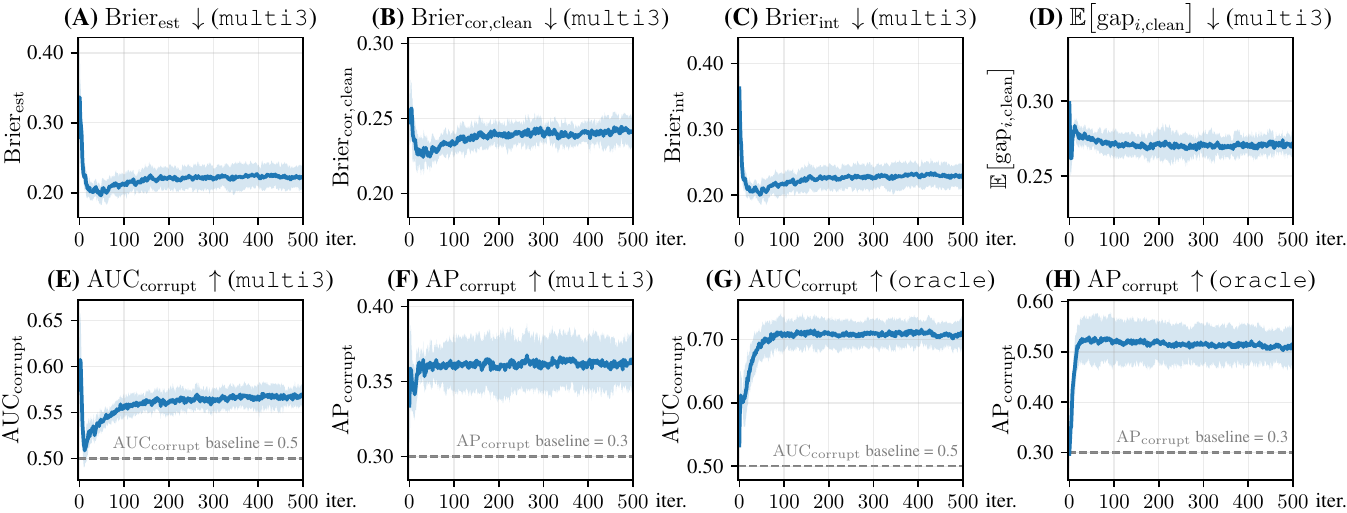}
    \caption{
    Reuters held-out diagnostic evolution under {\tt multi3} with corruption rate $0.3$ (mean$\pm$std over 5 runs).
    Same panel definitions as in \cref{fig:diag_mnist}.
    }
    \label{fig:app_diag_reuters}
\end{figure}

\begin{figure}[p]
    \centering
    \includegraphics[width=0.8\textwidth]{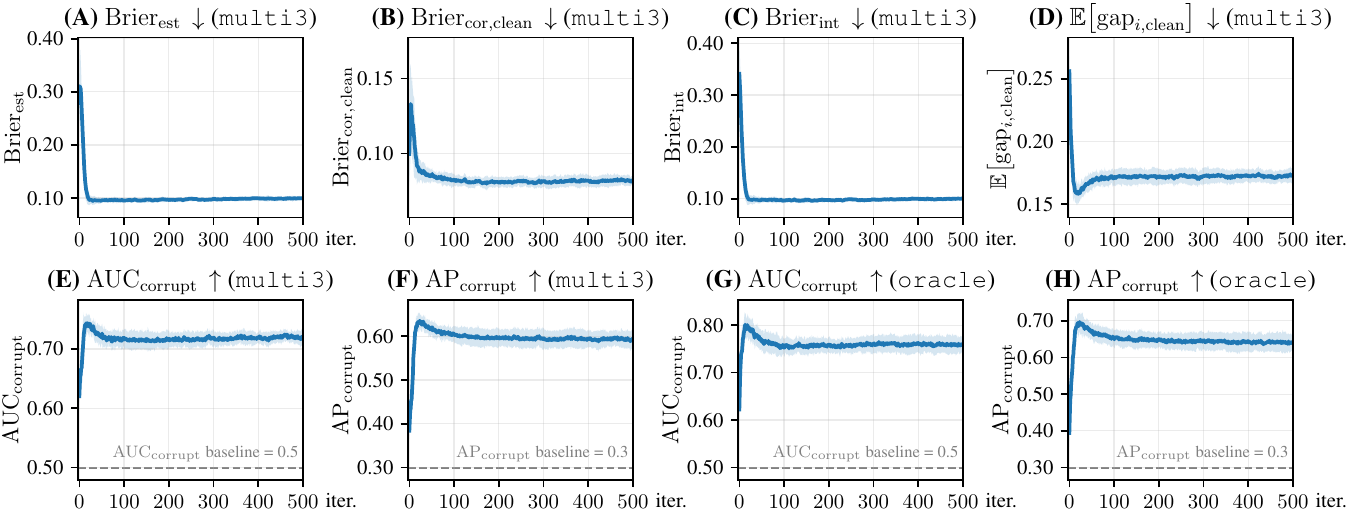}
    \caption{
    STL10 held-out diagnostic evolution under {\tt multi3} with corruption rate $0.3$ (mean$\pm$std over 5 runs).
    Same panel definitions as in \cref{fig:diag_mnist}.
    }
    \label{fig:app_diag_stl10}
\end{figure}

\begin{figure}[p]
    \centering
    \includegraphics[width=0.8\textwidth]{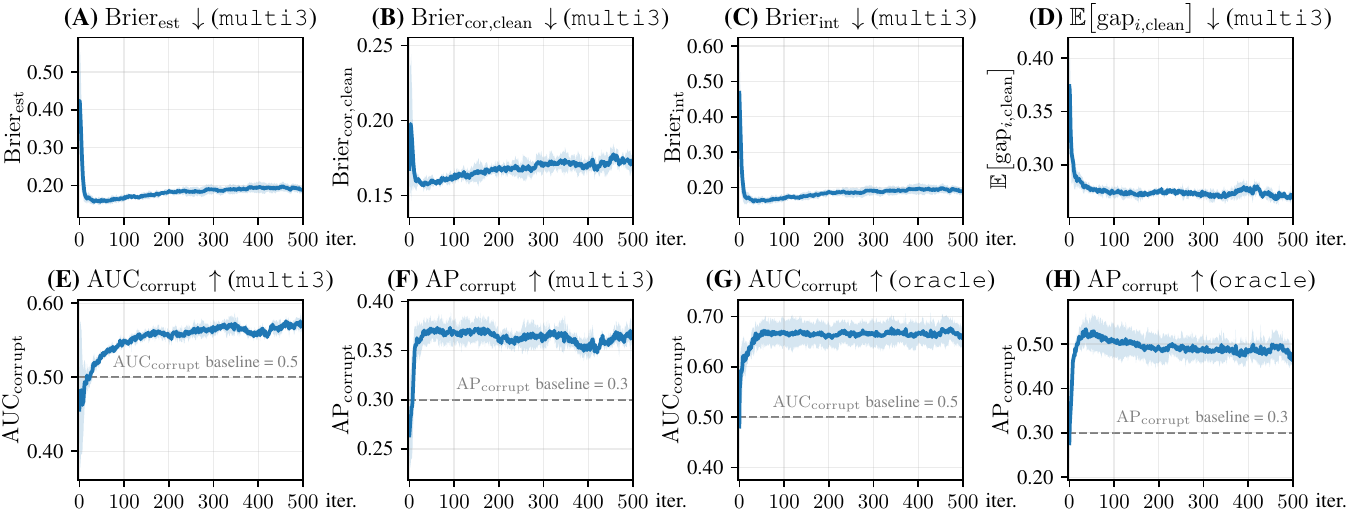}
    \caption{
    RCV1-10 held-out diagnostic evolution under {\tt multi3} with corruption rate $0.3$ (mean$\pm$std over 5 runs).
    Same panel definitions as in \cref{fig:diag_mnist}.
    }
    \label{fig:app_diag_rcv1}
\end{figure}

\paragraph{Diagnostic metrics.}
For a held-out pair $i=(a_i,b_i)$ with benchmark labels $t_{a_i}$ and $t_{b_i}$, define the external hard co-membership reference
\[
o_i^\star:=\mathbb{I}[t_{a_i}=t_{b_i}].
\]
Let $\mathcal{I}_{\mathrm{clean}}=\{i\in\mathcal{C}^{\mathrm{diag}}:c_i=0\}$ be the uncorrupted held-out subset, and define the residual discrepancy after correction as
\[
\mathrm{gap}_i
=
\left|
\operatorname{softclip}_{\xi}(\hat{y}^{\mathrm{oof}}_i)
-
y^{\mathrm{cor}}_i
\right|.
\]
For any relation score $v_i\in[0,1]$ and index set $\mathcal{I}$, we use
\[
\mathrm{Brier}(v;\mathcal{I})
=
\frac{1}{|\mathcal{I}|}
\sum_{i\in\mathcal{I}}(v_i-o_i^\star)^2.
\]
We report
$\mathrm{Brier}_{\mathrm{est}}=\mathrm{Brier}(\hat{y}^{\mathrm{oof}};\mathcal{C}^{\mathrm{diag}})$,
$\mathrm{Brier}_{\mathrm{cor,clean}}=\mathrm{Brier}(y^{\mathrm{cor}};\mathcal{I}_{\mathrm{clean}})$,
$\mathrm{Brier}_{\mathrm{int}}=\mathrm{Brier}(y^{\mathrm{int}};\mathcal{C}^{\mathrm{diag}})$, and
$\mathbb{E}[\mathrm{gap}_{i,\mathrm{clean}}]=|\mathcal{I}_{\mathrm{clean}}|^{-1}\sum_{i\in\mathcal{I}_{\mathrm{clean}}}\mathrm{gap}_i$.
For reliability-aware screening, $\mathrm{gap}_i$ is used to rank held-out constraints, and we report $\mathrm{AUC}_{\mathrm{corrupt}}$ and $\mathrm{AP}_{\mathrm{corrupt}}$ against the known corruption indicator $c_i$.
These scores measure alignment with corruption rather than pure corruption estimation, since the reliability signal may also reflect difficult clean pairs, residual epistemic variation, and imperfect estimation/correction; nevertheless, corrupted records are expected to rank higher on average.
As a reference, we also replace the {\tt multi3} expert-generated judgments by $o_i^\star$ before applying the same $3$k budget and corruption rate $0.3$, thereby removing expert epistemic uncertainty from the clean supervision.

\paragraph{Evolution of internal estimates.}
The full diagnostic curves are shown in
\cref{fig:app_diag_cifar100,fig:app_diag_cifar10,fig:app_diag_fmnist,fig:app_diag_imagenet10,fig:app_diag_mnist,fig:app_diag_reuters,fig:app_diag_stl10,fig:app_diag_rcv1}.
Panels (A--C) track the Estimator, Corrector, and Integrator relation scores through their Brier errors against the external hard co-membership reference $o_i^\star$.
Across datasets, these errors typically decrease quickly in the early stage and then become relatively stable, with some late-stage fluctuation.
This pattern is clearer on CIFAR100-20, CIFAR10, ImageNet10, MNIST, and STL10, while FMNIST and Reuters show more rebound after the initial improvement, and RCV1-10 is the most unstable case.
Because these Brier scores use $o_i^\star$ rather than the latent aleatoric relation $R_{a_ib_i}^\star$, small non-monotone changes should be read as diagnostic behavior rather than as direct evidence about latent-target recovery.
Panel (D) gives a complementary correction-side view: the clean residual discrepancy usually drops early, indicating that corrected relations become more aligned with estimator beliefs on uncorrupted held-out constraints.
Since the Estimator and Corrector co-evolve, this quantity should be interpreted together with (A--C), as evidence that the correction stage follows the improving estimator signal.

\paragraph{Screening against injected corruption.}
Panels (E,F) evaluate whether the reliability signal aligns with the injected corruption indicator.
Under {\tt multi3}, $\mathrm{AUC}_{\mathrm{corrupt}}$ and $\mathrm{AP}_{\mathrm{corrupt}}$ are generally above their chance baselines, so corrupted records tend to be ranked higher by the screening signal.
The alignment is clearer on most image datasets, especially CIFAR100-20, CIFAR10, ImageNet10, MNIST, and STL10, but weaker or noisier on Reuters and particularly RCV1-10.
This limitation reflects the heuristic nature of the screening signal: unreliability may arise not only from stochastic corruption, but also from expert-side epistemic variation, imperfect estimation/correction, and partial self-confirmation during training.
The oracle-reference curves in panels (G,H) support this interpretation, as they usually track corruption more closely than the {\tt multi3} curves when the clean supervision is generated from the hard oracle relation before corruption.

\paragraph{Summary.}
Overall, the diagnostics provide consistent qualitative evidence for the intended behavior of ECI-PP, although the curves are not uniformly monotone across datasets.
The internal estimates improve mainly in the early stage, the correction discrepancy is reduced, and the screening signal remains meaningfully associated with injected corruption.
RCV1-10 is the most challenging case, consistent with the broader experimental picture under severe class imbalance, where improvements in clustering structure are less uniformly reflected by the held-out diagnostic curves.
Early stabilization further suggests that the ECI components often reach useful internal agreement before the fixed training horizon; however, we keep this schedule here and leave principled early stopping for future work.
These diagnostics therefore provide practical held-out evidence for ECI-PP's intended behavior, while remaining proxy measurements rather than exact tests of latent aleatoric recovery or pure corruption detection.





\subsection{Sensitivity analysis}
\label{app:sensitivity_ablation_results}

We further study the sensitivity of the main ECI-PP settings around the default configuration specified in \cref{sec:exp_settings}.
All sensitivity runs use the default single-expert supervision condition, i.e., {\tt lv0.01} expert quality, corruption rate $0.3$, and $9$k constraints.
To assess sensitivity to each factor, we use random initialization without unsupervised pretraining and vary one factor while keeping the others fixed.
\cref{fig:app_sens_kfold,fig:app_sens_kappa,fig:app_sens_corr_size,fig:app_sens_lambdacor,fig:app_sens_gamma,fig:app_sens_n0,fig:app_sens_xi,fig:app_sens_lambdarec} report the corresponding test ACC, NMI, and ARI.

\paragraph{Estimator choices.}
\cref{fig:app_sens_kfold,fig:app_sens_kappa} examine the Estimator-side choices.
For the number of folds, increasing $K$ from $2$ to $20$ gives mild gains on several datasets, such as MNIST, FMNIST, CIFAR10, and ImageNet10, but the curves are mostly flat and the improvement is not proportional to the extra cross-fitting cost.
This supports using $K=5$ as a practical default: it avoids the weakest cross-fitting setting while keeping the main Estimator cost moderate.
The information-based warm-up weight $\kappa_i$ provides a small but consistent benefit.
Enabling it is usually better or comparable across datasets, with visible gains on MNIST, FMNIST, Reuters, and CIFAR10, indicating that informative constraints are useful for stabilizing the early Estimator signal.
This agrees with the main comparison in \cref{table:exp1_9k_main}, where Weighted ProbPair generally improves over ProbPair by using the same information-based weighting principle.
At the same time, the gains remain moderate, suggesting that ECI-PP benefits from this design without being overly dependent on it.

\begin{figure}[htbp]
    \centering
    \includegraphics[width=0.9\textwidth]{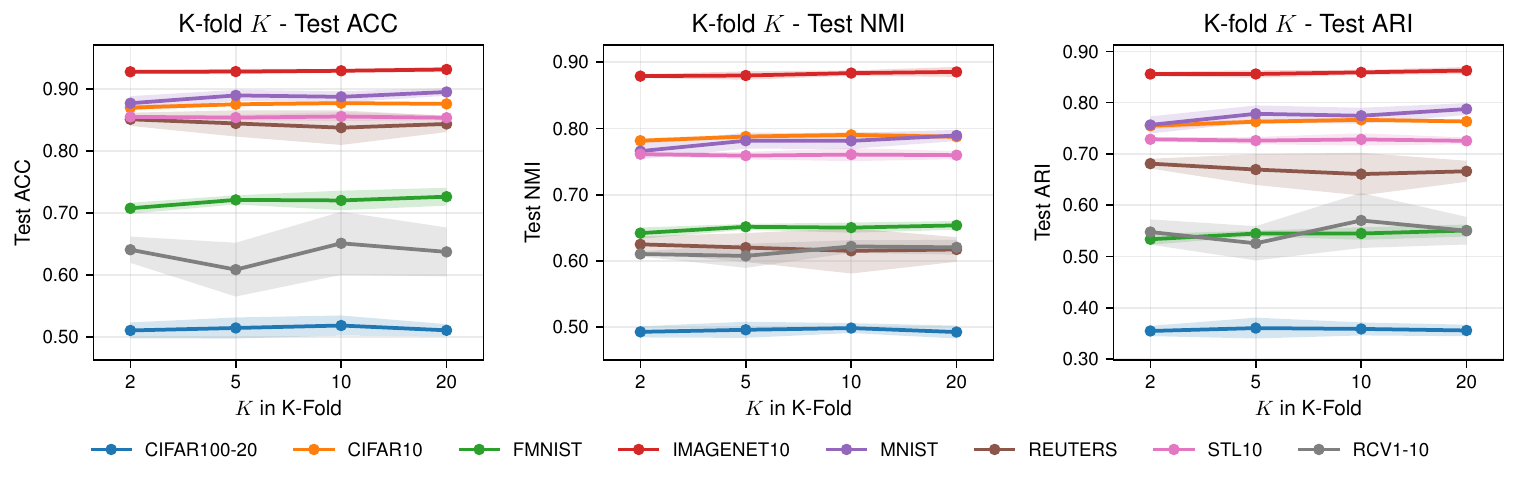}
    \caption{Sensitivity to the number of estimator folds $K$. Each panel reports test ACC, NMI, or ARI across datasets (mean$\pm$std over 5 runs).}
    \label{fig:app_sens_kfold}
\end{figure}

\begin{figure}[htbp]
    \centering
    \includegraphics[width=0.9\textwidth]{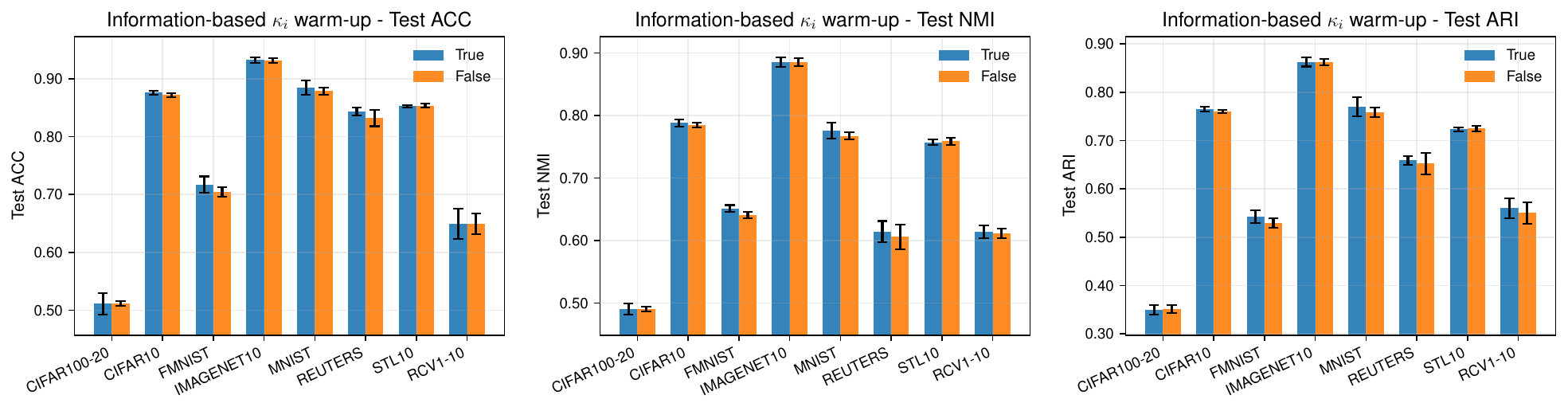}
    \caption{Effect of the information-based warm-up weight $\kappa_i$. Each panel reports test ACC, NMI, or ARI across datasets (mean$\pm$std over 5 runs).}
    \label{fig:app_sens_kappa}
\end{figure}

\begin{figure}[htbp]
    \centering
    \includegraphics[width=0.9\textwidth]{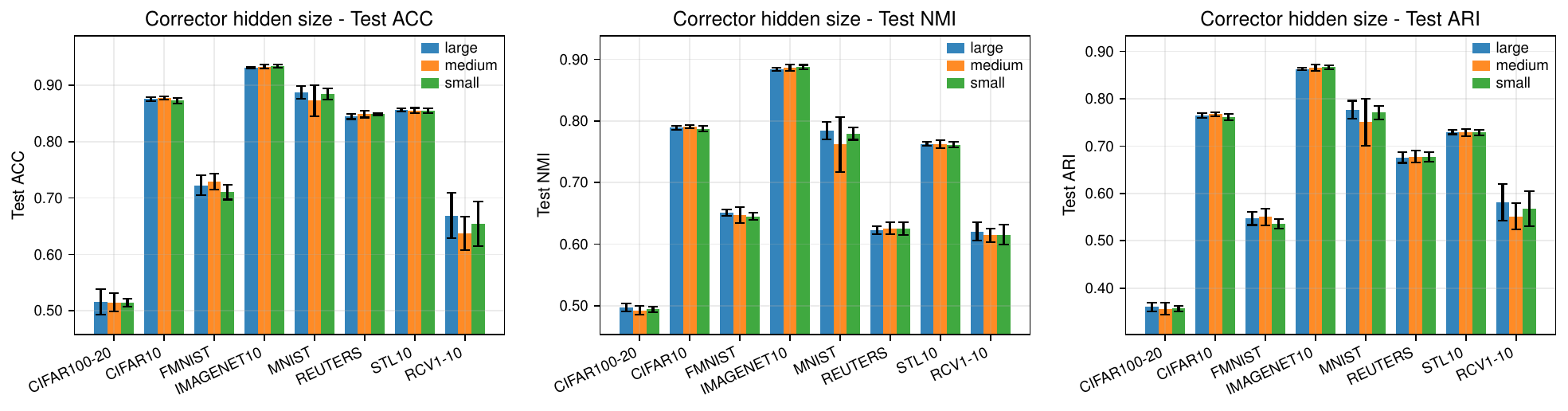}
    \caption{Sensitivity to Corrector capacity. Each panel reports test ACC, NMI, or ARI across datasets (mean$\pm$std over 5 runs).}
    \label{fig:app_sens_corr_size}
\end{figure}

\begin{figure}[htbp]
    \centering
    \includegraphics[width=0.9\textwidth]{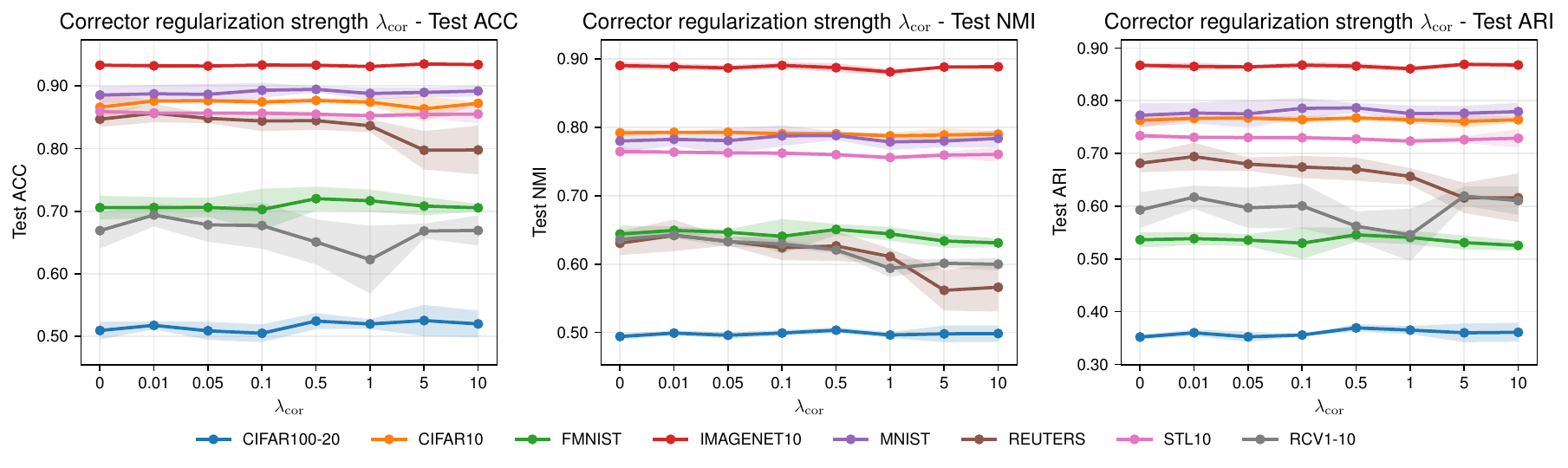}
    \caption{Sensitivity to the Corrector regularization strength $\lambda_{\mathrm{cor}}$. Each panel reports test ACC, NMI, or ARI across datasets (mean$\pm$std over 5 runs).}
    \label{fig:app_sens_lambdacor}
\end{figure}

\paragraph{Corrector choices.}
\cref{fig:app_sens_corr_size,fig:app_sens_lambdacor} study Corrector capacity and regularization.
For capacity, we compare a \textit{large} encoder-sized Corrector ($500$--$500$--$2000$), a \textit{medium} classifier-sized Corrector ($512$--$512$), and the default \textit{small} Corrector ($64$--$16$).
The three sizes perform similarly on most datasets, and this supports the default use of a small Corrector: it is cheaper and also limits the risk of fitting arbitrary pair-specific deviations rather than structured expert-conditioned correction.
The regularization strength $\lambda_{\mathrm{cor}}$ shows the same trade-off more directly.
Compared with $\lambda_{\mathrm{cor}}=0$, mild to moderate regularization can help on datasets such as FMNIST, CIFAR10, CIFAR100-20, and RCV1-10, but overly strong regularization is harmful on some datasets, most clearly Reuters when $\lambda_{\mathrm{cor}}$ reaches $5$ or $10$.
Thus, the Corrector should not be left completely unconstrained, but it also should not be forced too strongly toward zero correction.
The default $\lambda_{\mathrm{cor}}=0.5$ lies in the stable middle range.

\paragraph{Reliability screening and boundary handling.}
\cref{fig:app_sens_gamma,fig:app_sens_n0,fig:app_sens_xi} evaluate the parameters used after correction.
The screening sharpness $\gamma$ is robust over the tested range, especially for $\gamma\in[5,20]$.
Larger values are useful on some difficult datasets, notably RCV1-10 and also Reuters or MNIST in some metrics, indicating that stronger suppression of unreliable constraints can help when supervision is more fragile.
The Bayesian-confidence scale $n_0$ is also stable overall.
Larger $n_0$ improves RCV1-10 and slightly helps Reuters, suggesting that these datasets benefit from assigning more evidence to the refined supervision, whereas FMNIST does not benefit from the same increase.
Thus, $n_0$ controls a real reliability trade-off, but the default value remains within a safe region.
Finally, performance changes little as the soft-clipping sharpness $\xi$ varies: 
most datasets change little from $\xi=5$ to $50$, with only mild fluctuations on RCV1-10 and a few small metric-specific changes.
This supports soft clipping as a robust boundary-handling design for keeping corrected relations in a valid probability range without relying on a finely tuned clipping sharpness.

\paragraph{Reconstruction strength.}
\cref{fig:app_sens_lambdarec} shows a noticeable dataset-dependent trade-off for $\lambda_{\mathrm{rec}}$.
The default value $\lambda_{\mathrm{rec}}=0.02$ follows SpherePair \cite{SpherePair} and remains a reliable setting, but ECI-PP often benefits from moderately stronger reconstruction.
For example, MNIST improves steadily as $\lambda_{\mathrm{rec}}$ increases, while Reuters and RCV1-10 prefer a moderate range around $0.05$--$0.2$ before performance drops at larger values.
In contrast, FMNIST and CIFAR10 do not benefit from overly large reconstruction weights, and ImageNet10 is almost unchanged.
This suggests that reconstruction provides useful supervision-independent structure for out-of-fold estimation and final integration, but too much reconstruction can compete with clustering-oriented pairwise learning on some datasets.
When validation-based tuning is available, $\lambda_{\mathrm{rec}}$ is therefore a meaningful parameter to adjust; without such tuning, the shared default remains a fair and stable compromise.

\begin{figure}[htbp]
    \centering
    \includegraphics[width=0.9\textwidth]{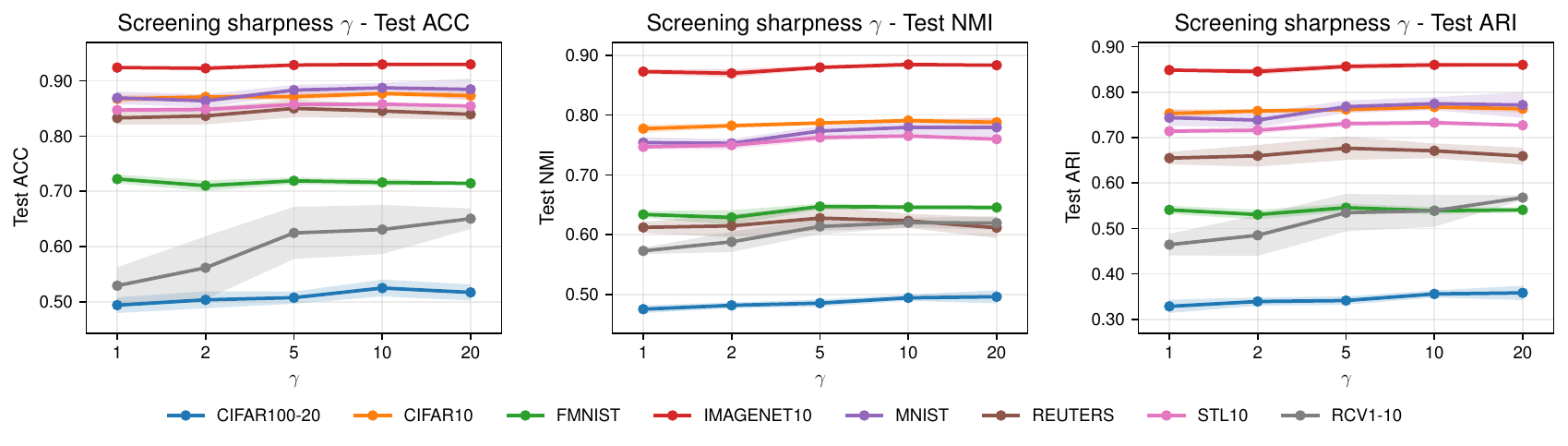}
    \caption{Sensitivity to the reliability-screening sharpness $\gamma$. Each panel reports test ACC, NMI, or ARI across datasets (mean$\pm$std over 5 runs).}
    \label{fig:app_sens_gamma}
\end{figure}

\begin{figure}[htbp]
    \centering
    \includegraphics[width=0.9\textwidth]{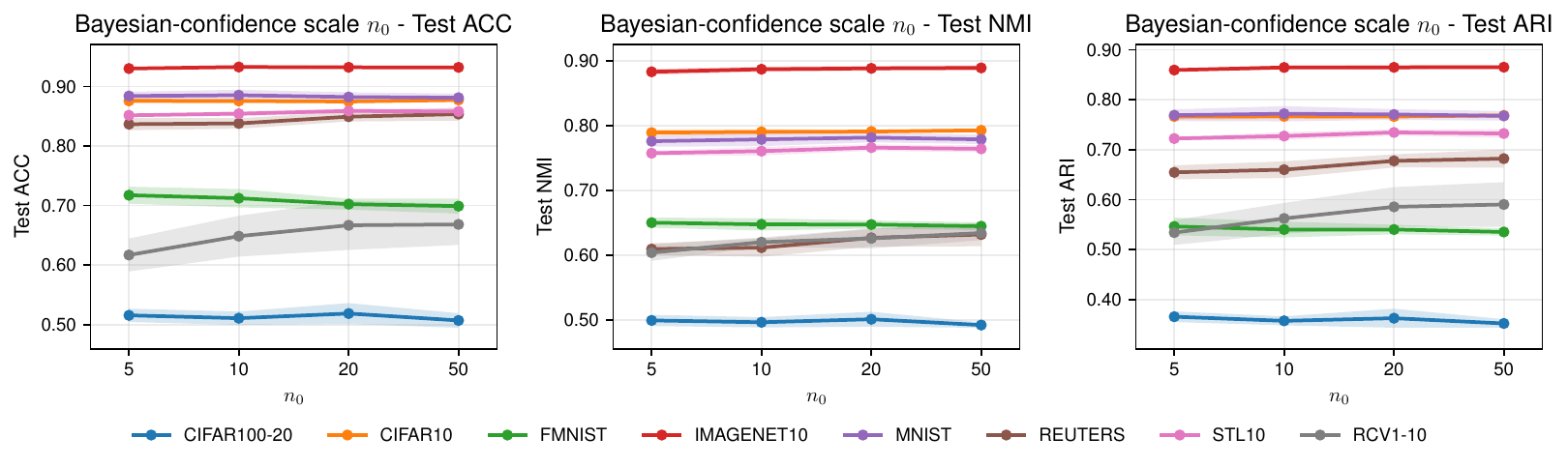}
    \caption{Sensitivity to the Bayesian-confidence scale $n_0$. Each panel reports test ACC, NMI, or ARI across datasets (mean$\pm$std over 5 runs).}
    \label{fig:app_sens_n0}
\end{figure}

\begin{figure}[htbp]
    \centering
    \includegraphics[width=0.9\textwidth]{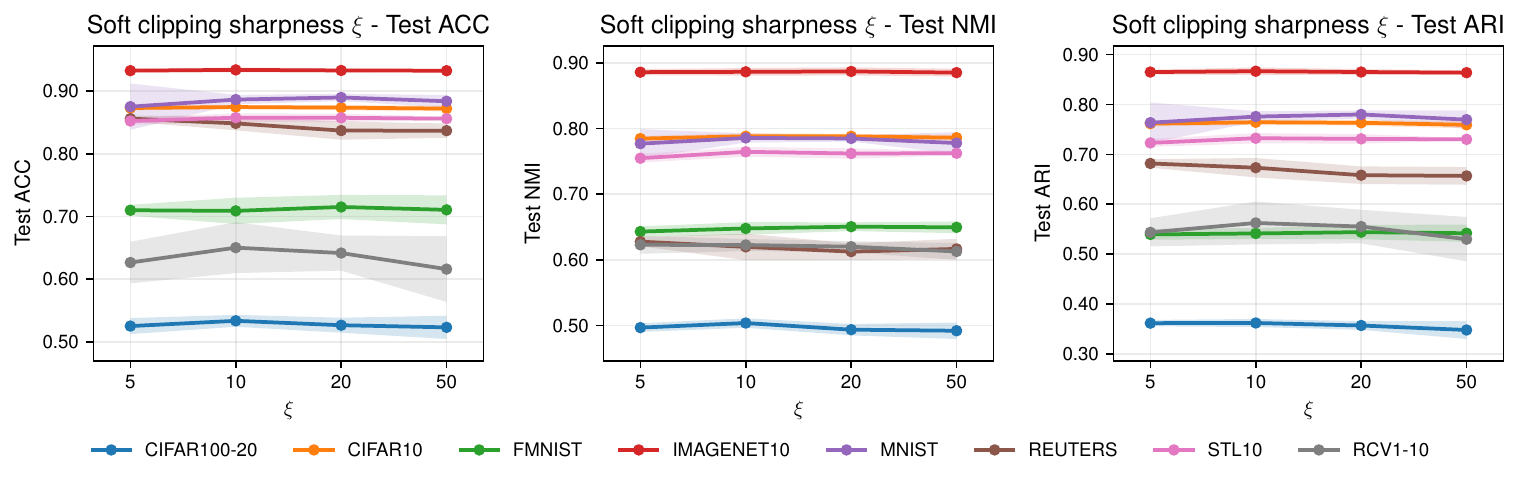}
    \caption{Sensitivity to the soft-clipping sharpness $\xi$. Each panel reports test ACC, NMI, or ARI across datasets (mean$\pm$std over 5 runs).}
    \label{fig:app_sens_xi}
\end{figure}

\begin{figure}[htbp]
    \centering
    \includegraphics[width=0.9\textwidth]{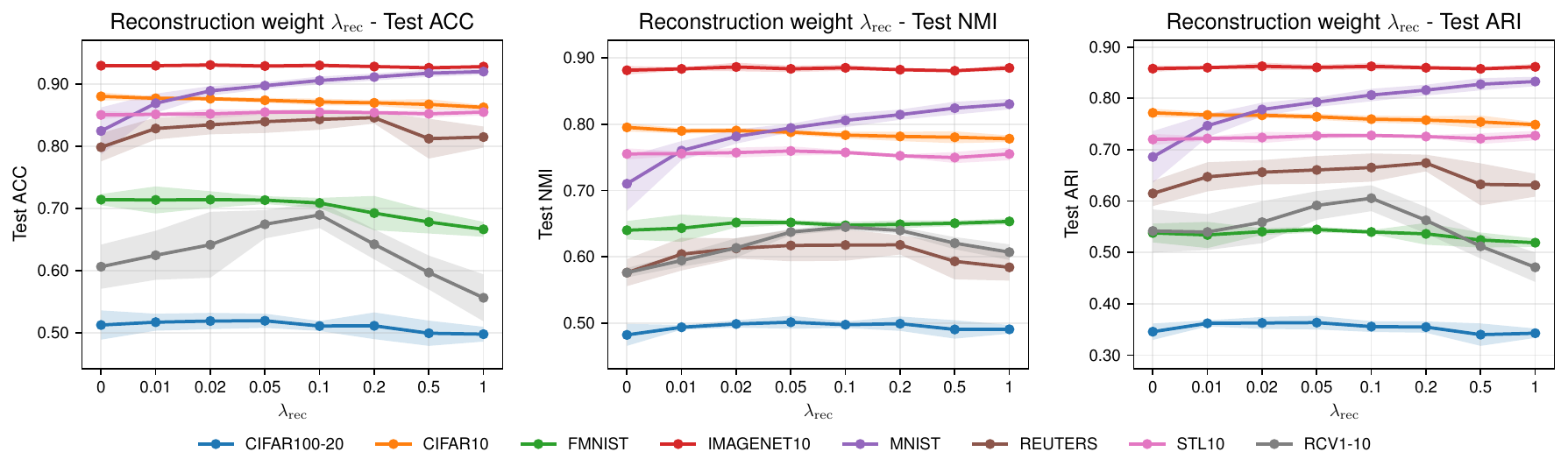}
    \caption{Sensitivity to the reconstruction weight $\lambda_{\mathrm{rec}}$. Each panel reports test ACC, NMI, or ARI across datasets (mean$\pm$std over 5 runs).}
    \label{fig:app_sens_lambdarec}
\end{figure}

Overall, the sensitivity results support the use of the default ECI-PP configuration in the main experiments.
Estimator, Corrector, screening, and soft-clipping choices are stable over broad ranges, while the more dataset-dependent reconstruction weight changes performance gradually rather than causing abrupt failure.
The results therefore indicate that ECI-PP is not hyperparameter-fragile, even though dataset-specific validation could still improve individual settings when such tuning is allowed.

\section{Learning efficiency}
\label{app:learning_efficiency}

\begin{table}[h]
\centering
\scriptsize
\setlength{\tabcolsep}{3pt}
\caption{Overall wall-clock training time for different methods on eight datasets under the {\tt multi3} multi-expert supervision setting with $9$k constraints and corruption probability $0.3$, measured on a single NVIDIA A100 40GB GPU. Methods marked with $^\ast$ require hyperparameter tuning, and their corresponding times are underlined.}
\label{tab:runtime_summary}
\begin{tabular}{lcccccccc}
\toprule
 & CIFAR100-20 & CIFAR10 & FMNIST & ImageNet10 & MNIST & Reuters & STL10 & RCV1-10 \\
\midrule
VanillaDCC & 1m12s & 1m10s & 1m20s & 0m45s & 1m27s & 0m37s & 0m41s & 1m51s \\
VolMaxDCC$^\ast$ & \underline{10m23s} & \underline{45m21s} & \underline{1h43m18s} & \underline{6m20s} & \underline{1h35m30s} & \underline{5m00s} & \underline{3m06s} & \underline{2h39m25s} \\
CIDEC & 18m27s & 19m26s & 24m49s & 5m03s & 25m06s & 5m53s & 6m03s & 1h07m26s \\
SpherePair & 18m27s & 17m58s & 22m25s & 5m09s & 22m54s & 6m05s & 5m18s & 1h01m54s \\
SpherePair (EA)$^\ast$ & \underline{3h53m10s} & \underline{3h40m47s} & \underline{4h30m35s} & \underline{1h49m25s} & \underline{4h33m48s} & \underline{2h12m12s} & \underline{1h51m31s} & \underline{10h40m49s} \\
ProbPair & 18m19s & 17m42s & 22m16s & 4m57s & 22m41s & 5m49s & 5m02s & 1h01m36s \\
Weighted ProbPair & 18m24s & 17m44s & 22m23s & 5m07s & 22m56s & 5m52s & 5m06s & 1h01m37s \\
Weighted ProbPair (EA)$^\ast$ & \underline{3h54m42s} & \underline{3h34m20s} & \underline{4h24m15s} & \underline{1h44m36s} & \underline{4h22m09s} & \underline{2h01m15s} & \underline{1h46m31s} & \underline{10h32m39s} \\
\textbf{ECI-PP} (Full) & 29m48s & 29m12s & 34m01s & 15m49s & 34m39s & 17m34s & 15m59s & 1h15m41s \\
\bottomrule
\end{tabular}
\end{table}

\paragraph{Overall training time.}
Appendix~\ref{app:complexity} analyzes the computational complexity of ECI-PP; here, based on the implementation in Appendix~\ref{app:implementation}, we report empirical wall-clock measurements for the complete training pipelines in \cref{tab:runtime_summary}.
Each entry is measured under the {\tt multi3} setting with $9$k constraints and corruption probability $0.3$, and covers the full practical training procedure:
(i) for methods requiring unsupervised pretraining, the reported time includes the pretraining stage;
(ii) for methods marked with $^\ast$, it also includes three scans over the corresponding hyperparameter list and a final training run with the selected hyperparameter.
When both pretraining and hyperparameter tuning are required, the timing includes pretraining on the training split excluding validation samples before tuning, followed by pretraining on the full training split before the final run.

\paragraph{Runtime comparison.}
The runtimes in \cref{tab:runtime_summary} should be interpreted alongside the clustering results in Appendix~\ref{app:full_budget_results}.
VanillaDCC is fastest, but mainly because collapsed assignment learning triggers early stopping; its performance in \cref{table:exp1_2_no_ensemble} is far below the competitive methods.
VolMaxDCC has variable cost, becoming expensive on FMNIST, MNIST, and RCV1-10 due to validation-based hyperparameter search.
For methods with autoencoder pretraining, including CIDEC, SpherePair, ProbPair, Weighted ProbPair, and ECI-PP, overall costs remain on the same order, with pretraining as a shared pipeline component.
ECI-PP adds moderate overhead from estimator cross-fitting and correction/integration, but this overhead is modest relative to its empirical gains; in more resource-limited settings, for example when unsupervised pretraining is unavailable, Appendix~\ref{app:no_pretraining_results} shows an even larger performance advantage.
The expert-aware SpherePair and Weighted ProbPair extensions are most expensive because ensemble construction and hyperparameter tuning scale with the number of experts.
In contrast, Appendix~\ref{app:complexity} shows that ECI-PP incorporates expert identities through expert-conditioned correction without replicating the full training pipeline, so its dominant cost scales mainly with the constraint budget rather than the expert count.



\end{document}